\documentclass[twoside]{article}

\usepackage[accepted]{aistats2020}

\usepackage[round]{natbib}

\usepackage[breaklinks=true]{hyperref}       %
\usepackage{url}            %
\usepackage{booktabs}       %
\usepackage{amsfonts}       %
\usepackage{nicefrac}       %
\usepackage{microtype}      %
\usepackage{graphicx}
\usepackage{booktabs, multirow}
\usepackage{xspace}
\usepackage{wrapfig}
\usepackage{makecell}
\usepackage{enumitem}
\usepackage[font=small,labelfont=bf,tableposition=top]{caption}

\usepackage[dvipsnames,svgnames,x11names,hyperref, enumerate]{xcolor}
\hypersetup{
	colorlinks,
	citecolor=Sepia,
	filecolor=red,
	linkcolor=[RGB]{0,0,215},
	urlcolor=Blue
}

\usepackage{algorithm, algpseudocode}
\usepackage{amsmath, amsthm, amssymb}

\newtheorem{theorem}{Theorem}

\newtheorem{lemma}{Lemma}

\newtheorem{corollary}{Corollary}
\theoremstyle{definition}
\newtheorem{definition}{Definition}
\theoremstyle{remark}
\newtheorem{remark}[theorem]{Remark}

\DeclareGraphicsExtensions{.png,.pdf}

\newcommand{\wt}{\widetilde}
\newcommand{\x}{\mathbf{x}}
\newcommand{\z}{\mathbf{z}}

\newcommand{\calX}{\mathcal{X}}

\newcommand{\calI}{\mathcal{I}}

\newcommand{\calS}{\mathcal{S}}

\renewcommand{\b}{\mathsf{b}}
\newcommand{\kth}{k^{\mathrm{th}}}
\usepackage[caption=false]{subfig}
\newcommand{\astu}{\mathsf{ast}}
\newcommand{\emprob}{\mathsf{ER}}
\newcommand{\defenderscore}{{defscore}\xspace}

\newcommand{\ropt}{f_{\mathsf{ropt}}}
	
\usepackage{amsmath}
\DeclareMathOperator*{\argmax}{argmax}
\DeclareMathOperator*{\argmin}{argmin}

\newcommand{\calT}{\mathcal{T}}
\renewcommand{\b}{\mathbf{b}}

\newcommand{\R}{\mathbb{R}}
\newcommand{\labels}{[C]}

\newcommand{\poly}{\mathrm{poly}}

\newcommand{\spruned}{\calS^{\mathsf{AP}}}

\newcommand{\bfu}{\mathbf{u}}

\begin{document}

\runningtitle{Robustness for Non-Parametric Classification: A Generic Attack and Defense}

\twocolumn[

\aistatstitle{Robustness for Non-Parametric Classification: \\ A Generic Attack and Defense}

\aistatsauthor{Yao-Yuan Yang* \And Cyrus Rashtchian* \And Yizhen Wang \And Kamalika Chaudhuri}

\aistatsaddress{
  \{yay005, crashtchian, yiw248, kamalika\}@eng.ucsd.edu \\
  University of California, San Diego, Computer Science \& Engineering } ]

\begin{abstract}
Adversarially robust machine learning has received much recent attention.
However, prior attacks and defenses for non-parametric classifiers have been developed in an ad-hoc or classifier-specific basis. 
In this work, we take a holistic look at adversarial examples for non-parametric classifiers, including nearest neighbors, decision trees, and random forests.
We provide a general defense method, adversarial pruning, that works by preprocessing the dataset to become well-separated. To test our defense, we provide a novel attack that applies to a wide  range of non-parametric classifiers. 
Theoretically, we derive an optimally robust classifier, which is analogous to the Bayes Optimal. We show that adversarial pruning can be viewed as a finite sample approximation to this optimal classifier. We empirically show that our defense and attack are either better than or competitive with prior work on non-parametric classifiers.
Overall, our results provide a strong and broadly-applicable baseline for future work on robust non-parametrics\footnote{Code available at \url{https://github.com/yangarbiter/adversarial-nonparametrics/}}.
\end{abstract}

\section{Introduction}
    
    State-of-the-art classifiers have been shown to suffer from substantial drops in accuracy when faced with adversarially modified inputs even if the modifications are imperceptibly slight. Due to the security concerns that this raises, a body of recent research has investigated the construction and prevention of adversarial examples -- small perturbations of valid inputs that cause misclassification~\citep{Carlini:EECS-2018-118, szegedy2013intriguing}. Most previous work has looked at parametric methods, i.e., neural networks and linear classifiers~\citep{biggio2013evasion,MeekLowd2005,madry2017, papernot2015limitations}, and there is a mature understanding of what properties can be exploited to design adversarial attacks and defenses for {any} parametric model. For example, parametric classifiers are based on continuous functions with gradients, which has been used to design gradient-based attacks~\citep{athalye2018obfuscated, carliniwagner}. Likewise, parametric models are mostly trained by minimizing a training loss, which has been exploited to build an effective and generic defense -- adversarial training, retraining after data augmentation with adversarial examples~\citep{carlini2019evaluating, madry2017, song2019improving}. 
    
    An alternative statistical paradigm is that of non-parametric methods, such as nearest neighbor, decision tree, and random forest classifiers, which typically apply to dense data in lower dimensional spaces. These are {{local predictors}},  whose output depends on labeled points close to an input. Surprisingly, these methods behave very differently from parametrics when it comes to adversarial examples. In many cases, they have no gradients, and adversarial examples for parametric models fail to transfer~\citep{papernot2016transferability}.  Generic defenses, such as adversarial training, appear to be ineffective as well~\citep{dubey2019defense, papernot2018deep, WJC18}. 
    
    While prior work has constructed attacks and defenses for some specific classifiers~\citep{chen2019robust,dubey2019defense,kantchelian,sitawarin2019robustness,WJC18}, there appear to be no generic approaches, and no generic principles that can be used to guide the design of attacks and defenses for variety of non-parametric methods.

\begin{figure*}[th]
	\centering
	\subfloat{
		\includegraphics[width=.31\textwidth]{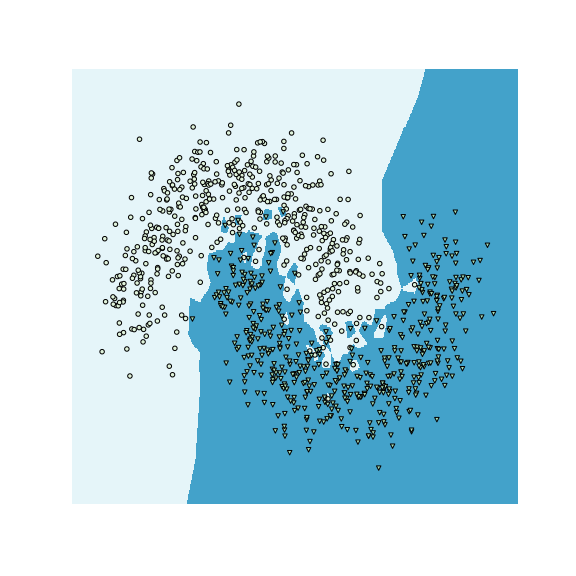}}
	\subfloat{
		\includegraphics[width=.31\textwidth]{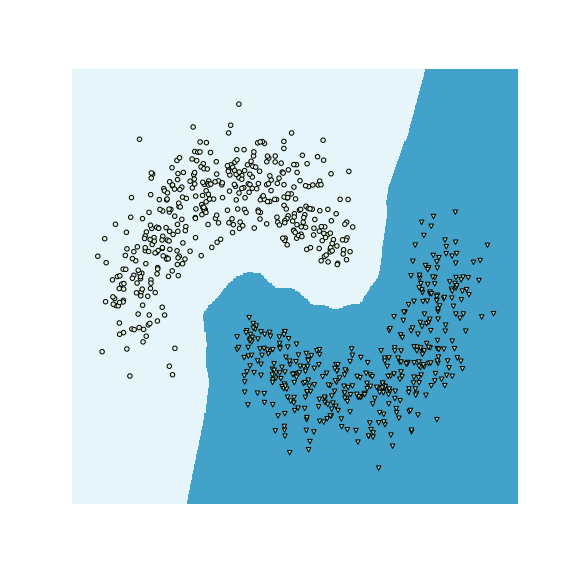}}
	\caption{Normal vs. Defended $1$-Nearest Neighbor.}
	\label{fig:ap-pics}
\end{figure*}

In this work, we identify two key general principles, and use them to design a generic defense and an attack that apply to a variety of non-parametric methods. 

To design defenses, we ask: when do non-parametric methods work well?
Figure~\ref{fig:ap-pics} depicts two variants of random forests.
In the left figure, we observe that datasets with nearby oppositely-labeled points may lead to classifiers with convoluted decision boundaries. In the right figure, we see that well-separated data lead to classification regions that are more robust to small perturbations. We will use this low-dimensional intuition as a starting point for generic defense methods.

Figure~\ref{fig:ap-pics} suggests that since these methods make local
predictions, they might work well when data from different classes are
well-separated in space. We clearly cannot hope for such separation in most real datasets. Therefore, we propose to
preprocess the training data by removing a subset so that different classes are
well-separated. To ensure classification accuracy, we propose
removing the minimal subset of points that ensure this property. We call our
method {\em{Adversarial Pruning}}, which can be used as a pre-processing step
before training any generic non-parametric classifier.

To evaluate our defense, we propose a new attack that
is based on our next key observation: many non-parametric methods divide the
instance space into convex polyhedra, and predict in a piecewise constant
manner in each. For example, for $1$-nearest neighbor, these polyhendra are the
Voronoi cells. This suggests the following attack: find the closest polyhedron
to an input where the classifier predicts a different label and output the
closest point in this region. We implement this strategy by solving a
collection of convex programs, and in cases where solution is computationally
expensive, we provide a heuristic method for finding an approximate solution.
We refer to these attacks as the exact and approximate {\em region-based
attack}.

We next provide some theoretical justification for our methods. For our defense,
we show that adversarial pruning can be interpreted as a finite-sample version of
a robust analogue to the Bayes Optimal classifier. We formally introduce this robust classifier, that we call the $r$-optimal classifier, and show that it  maximizes {\em{astuteness}}
(accuracy where it is robust with radius $r$). For our attack, we show that the exact region-based
attack is optimal, in the sense that it yields the closest adversarial
example to a test input. 

We empirically evaluate the adversarial pruning defense using the
region based attack and prior attacks. We provide a general and thorough evaluation, for $k$-nearest neighbors
($k$-NN), decision trees, and random forests.  We see that adversarial pruning
consistently improves robustness, outperforming adversarial training on several
datasets and is competitive with classifier-specific
defenses. For our attacks, we see that even without any classifier-specific
optimization, our new attacks either outperform or are
competitive with  prior attacks (in terms of perturbation amount).  This suggests that both the adversarial pruning
defense as well as the region based attack are good generic baselines for
evaluating the robustness of non-parametric methods. 
\section{Preliminaries}\label{sec:preliminaries}

     We begin with a brief introduction to non-parametric methods that are
local classifiers whose output depends on training data close to the test
instance. These methods are typically used with dense lower-dimensional data,
such as those in Figure~\ref{fig:ap-pics}. 
Examples are $k$-nearest neighbor ($k$-NN) and tree-based classifiers.  The
{\em $k$-NN classifier} outputs the plurality label among the $k$ training
examples closest to $\x$ in an $\ell_p$ metric.  A {\em tree ensemble} contains
$T$ decision trees whose leaves are labeled with vectors in $\R^{C}$. Each
input $\x$ determines $T$ root-to-leaf paths, corresponding to vectors
$\bfu^1,\ldots, \bfu^T$. The output is the largest coordinate in $\bfu^1 +
\cdots + \bfu^T$.  Random forests are a subclass of tree ensembles. 

In what follows,  $f: \R^d \to \labels$ denotes a  classifier with $C$
classes, where $[C] := \{1,2,\ldots, C\}$. The training data for $f$ is a
dataset $\calS = \{(\x^j, y^j)\}_{j=1}^n$ of $n$ labeled examples, with $\x^j
\in \R^d$ and $y^j \in [C]$. 
    
\smallskip\noindent{\textbf{Robustness.}}    We study robustness in an adversarial model. The adversary's goal is to modify a true input by a small amount and cause the classifier to output the wrong label. Two main threat models have been proposed. The {\em black-box} setting restricts the adversary to only querying a classifier $f$ on various inputs. In the {\em white-box} setting, the adversary has full access to $f$, including the model structure and parameters. 

Fix a classifier $f$ and a norm $\|\cdot \|$ on $\R^d$.  An {\em adversarial example} for $f$ at~$\x$ is any other input $\wt \x$ such that $f(\x) \neq f(\wt \x)$. An {\em optimal adversarial example} for $f$ at~$\x$ is an input $\wt \x$ that minimizes $\|\x - \wt \x\|$ subject to $f(\x) \neq f(\wt \x)$. In other words, an optimal adversarial example $\wt \x$ is a closest vector to $\x$ that receives a different label. In practice it is not always possible to find the optimal adversarial example, and hence the goal is to find $\wt \x$ that is as close  to $\x$ as possible. %
We also define the robustness radius, the minimum perturbation needed to change the classifier label. 
\begin{definition} Let $\calX \times [C]$ be a labeled space with norm $\|\cdot \|$. The {\em robustness radius} of $f$ at $\x \in \calX$ is
    	\[
    	\rho(f,\x) := \min_{\wt \x \in \calX}\{\|\x - \wt \x\| : f(\x) \neq f(\wt \x)\}
    	\]
\end{definition}

\section{Adversarial Pruning Defense}
\label{sec:ap}

When are non-parametric methods robust? Since these are local classifiers,
Figure~\ref{fig:ap-pics} suggests that they may be robust when training data from
different classes is well-separated, and may fail when they overlap.

The training data may not be separated, so we will preprocess the data. 
We remove a subset of the training set, so that the remaining data are well-separated. Then, we train a non-parametric classifier on the rest.
A remaining question is which subset of points to remove. For high
classification accuracy, we remove the minimum subset whose removal ensures
this property.

This process of removing examples from training set so that certain properties
hold is called {\em pruning}.
In this section, we first introduce the method used to prune the dataset.
In Section~\ref{sec:main-theory}, we justify our method by interpreting it in light of classical results in
statistical learning theory~\citep{CD14,cover1967nearest,lugosi94}.

Formally, given a robustness radius $r$ and training set $\calS$, we propose the following generic way to preprocess the training set and improve the robustness of classifiers:

{\bf Adversarial Pruning.} Given~$r$ and a set~$\calS$, compute a maximum subset $\spruned \subseteq \calS$ such that differently-labeled points have  distance at least $2r$. Then, train any nonparametric classifier on $\spruned$.

After computing $\spruned$ once for a dataset, then we may train any classifier on the pruned training set.
Our main hypothesis is that this will lead to more robust classifiers when using non-parametric methods.
We will demonstrate empirically that this works well, and we will argue that this defense method is a finite-sample approximation to the optimal robust classifier.

Observe that while adversarial pruning is similar to the defense in~\cite{WJC18},  they actually retain additional points with confident labels, which ensures that their method converges to being robust where the {\em Bayes Optimal} is robust. Their work builds on previous results of~\cite{gottlieb2014efficient} and~\cite{kontorovich1} that sharpen the risk analysis of 1-NN by using pruning. As we explain in Section~\ref{sec:main-theory}, our method instead can be interpreted as a finite sample version of a different and more appropriate limit. 

One drawback of this approach is that the metric must be fine-grained enough to distinguish between close and far pairs. For most datasets and norms (e.g, Euclidean distance) for which non-parametrics are used, this will be the case. However, for binary features and the $\ell_\infty$ distance, we have the problem that every pair of different points has distance exactly one, and therefore, the similarity structure is meaningless. To circumvent this, we preprocess the binary feature vectors using standard feature-extraction methods (e.g., PCA), and then operate on the resulting space.

\subsection{Computing the Robust Dataset}

We use known graph algorithms to efficiently compute $\spruned$. 
Each training example is a vertex in the graph. Edges connect pairs of differently-labeled examples
$\x$ and $\x'$ whenever $\|\x - \x'\| \leq 2r$. 
We remove as few examples as possible so that no more edges remain.
This is equivalent to computing the minimum vertex cover.
For binary labels, this graph is bipartite, and a minimum vertex cover can be derived from a maximum matching.
The fastest method to solve maximum
matching is the Hopcroft-Karp algorithm~\citep{hopcroft1973n}. For a graph with $n$ vertices and $m$ edges, it takes time $O(m \sqrt{n})$. Fortunately, in practice, the graph of close pairs is quite sparse (for small $r$ and high dimensional feature spaces, with relatively separated classes). For example, if $m = \widetilde{O}(n)$, then computing $\spruned$  takes time $\widetilde{O}(n^{3/2})$. For large datasets, we note that {\em linear time} approximation algorithms are known
\citep{duan2014linear}.

When there are more than two labels, that is $C \geq 3$, it is NP-Hard to compute the optimal pruned subset, but approximation algorithms are
known~\citep{gottlieb2014efficient, kontorovich1}. The greedy algorithm provably generates a 2-approximation. A suboptimal solution still ensures that different classes are
separated, and hence, the robustness of the classifier does not require
finding the optimal pruned dataset.

\section{Region-Based Attack}\label{sec:algorithm}

In this section, we develop a way to evaluate robustness of non-parametric methods.
For parametric algorithms, generic gradient-based attacks exist. %
Our goal is to develop an analogous general attack method, which works well for multiple non-parametrics. Moreover, we aim to develop a white-box attack that will serve as a better baseline than black-box attacks.

The main challenge of finding adversarial examples is that these classifiers
have complicated decision regions.
The central idea behind our attack is that for many classifiers, such as
$k$-NN or random forests, we can decompose the decision regions into convex
sets.

\begin{definition}
	An {\em $(s,m)$-decomposition} is a partition of $\R^d$ into convex polyhedra $P_1,\ldots, P_s$ such that each $P_i$ can be described by up to $m$ linear constraints, and $f$ is {\em $(s,m)$-decomposable} if there is an  $(s,m)$-decomposition such that $f$ is constant on $P_i$ for each $i \in [s]$.
\end{definition} 

Figure \ref{fig:rba-pics} demonstrates the decomposition for two examples.
Figure \ref{fig:rba-pics}(a) shows how $1$-NN is decomposed.
In particular, a Voronoi diagram for $n$ points is an $(n,
n-1)$-decomposition ($P_1, \ldots, P_n$ are Voronoi cells). If $k \geq 1$,
then a $k$-NN classifier is $\left(\binom{n}{k},k(n-k)\right)$-decomposable;
every $k$ points correspond to polyhedra defined by $k(n-k)$ hyperplanes
separating the $k$ points from the other $n-k$
points~\citep{aurenhammer1991voronoi}.

\begin{figure*}[t]
	\centering
	\subfloat[$1$-NN]{
		\includegraphics[width=.28\textwidth]{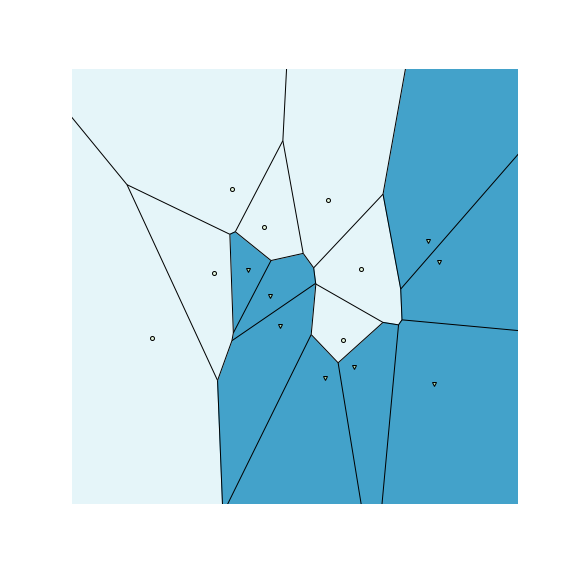}}
	\subfloat[Decision tree]{
		\includegraphics[width=.28\textwidth]{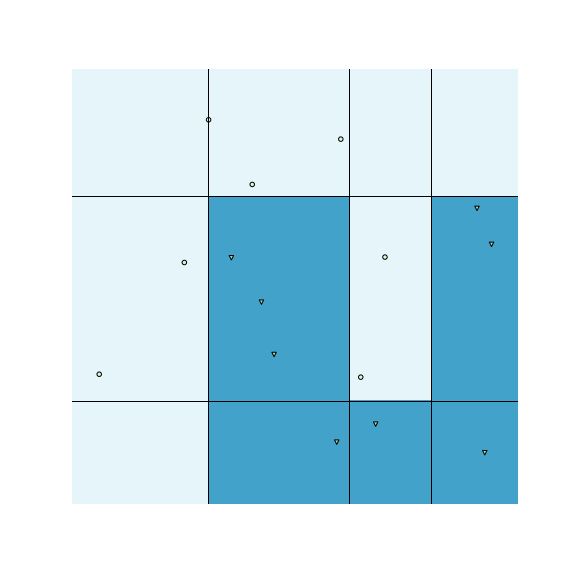}}
	\caption{
		$(s, m)$-decompositions of two non-parametrics.
	}
	\label{fig:rba-pics}
\end{figure*}

Tree-based classifiers also fit into our framework, and
Figure \ref{fig:rba-pics}(b) shows how a decision tree is decomposed.
Any decision tree of depth $D$ with $L$ leaves is $(L,D)$-decomposable; each root-to-leaf path
corresponds to a polyhedron $P_i$ defined by $D$ hyperplanes. Generally, if $f$ is an ensemble of $T$ trees, each with depth $D$ and $L$
leaves, then $f$ is $(L^T, DT)$-decomposable (proofs in Appendix~\ref{app:attack}). An
exponential dependence on $T$ is expected, since the
adversarial example problem for tree ensembles is NP-Hard~\citep{kantchelian}.

The existence of $(s,m)$-decompositions suggests the following attack.
Given a classifier $f$ and an input~$\x$, suppose we could find the closest
polyhedron $P_i$ in the decomposition where $f$ predicts a different label
than $f(\x)$.
Then, the closest point in $P_i$ would be the optimal adversarial example.
Our attack implements this strategy by searching over all polyhedra.

\subsection{Region-Based Attack} 

Let $f$ be an $(s,m)$-decomposable classifier with decomposition $P_1,\ldots, P_s$, where $f(\z) = y_i$ when $\z \in P_i$, for labels $y_i \in \labels$. 
To find an adversarial example for~$\x$, consider all polyhedra $P_i$ such that $f(\x) \neq y_i$. 
Then, output $\wt \x$ minimizing
\begin{equation}\label{eqn:opt-untargeted}
\underset{\ i : f(\x) \neq y_i \ }{\text{\ \ min\ \ }} 
\underset{\z \in P_i}{\text{\ min\ }} \ \ 
\|\x - \z\|.
\end{equation}

Each $P_i$ is described by $\leq m$ linear constraints, and the norm
objective is convex~\citep{boyd}.
Thus, we can solve each inner minimization problem in
(\ref{eqn:opt-untargeted}) separately by solving a convex program with $O(m)$
constraints.
This results in candidates $\z^i \in P_i$.
Taking the outer minimum over $i$ with $f(\x) \neq y_i$ leads to the optimal
adversarial example $\wt \x = \argmin_{\z^i} \ \|\x - \z^i\|$.

{\bf Efficiency.}
The running of the exact attack algorithm depends on two things: (i) the number of regions, which is based on the complexity of the classifier, and (ii) the number of constraints and dimensionality of the polyhedra. Due to advances in linear/quadratic program solvers, finding the adversarial example in a single region is quite efficient, i.e., the inner minimization problem in (\ref{eqn:opt-untargeted}) is easy. We find that the number of regions $s$ dominates the running time, i.e., the outer minimization problem in (\ref{eqn:opt-untargeted}) is hard. 
For $k$-NN, the number of convex polyhedra scales with $O(n^k)$. When $k=1$, this is efficiently solvable, because polyhedra have at most $n$ constraints, and the adversarial examples can be found quickly using a linear program for $\ell_\infty$ perturbations. Unfortunately, for $k>1$, this attack does not scale well, and we will develop an approximation algorithm for larger values of $k$.

For a single decision tree, again the exact attack is very efficient, depending only on the number of nodes in the tree. But for larger tree ensembles (e.g., large random forests), the optimal attack is very slow, as expected.

\subsection{Speeding Up the Search}
The exact attack is computationally intensive when $s$ is large; hence, finding optimal solutions is infeasible for random forests (with many trees) or $k$-NN (when $k$ is large). We next provide a computationally-efficient algorithm, which searches a constant number of regions.

The region-based attack for an $(s,m)$-decomposable $f$ requires solving up to $s$ convex programs, one for each polyhedron $P_i$ with a different label. If the number of polyhedra is large, then this may be computationally infeasible. Fortunately, (\ref{eqn:opt-untargeted}) has an obvious subdivision, based on the outer minimum over convex polyhedra. We use a relaxation that considers only a subset of polyhedra.  
We observe that each training point corresponds to a polyhedron---the one that $f$ uses to predict the label. 
When finding adversarial examples for $\x$, the natural choice is to utilize training data close to  $\x$.

{\bf Approximate Region-Based Attack.} Let $\calS$ be the training data. To find an adversarial example under $\ell_p$ for $\x$, we first compute the subset $\calS' \subseteq \calS$ of $s'$ points closest in $\ell_p$ distance to $\x$, while having different training labels than $f(\x)$. Next, we determine at most~$s'$ polyhedra $P_{i_1}, \ldots, P_{i_{s'}}$ containing points in $\calS'$ (as the polyhedra partition $\R^d$). We solve the inner optimization problem in (\ref{eqn:opt-untargeted}) for each $P_{i_j}$ to find candidates $\z^i$ for $i \in [s']$. Finally, we output $\wt \x = \argmin_{\z^i} \ \|\x - \z^i\|$, where the minimum is over these $s'$ candidates.  

As we only solve $s' \ll s$ convex programs,  the running time is greatly reduced compared to the optimal region-based attack. 
Empirically, this approximation finds adversarial examples with low perturbation. %

\section{Theoretical Justification}
\label{sec:main-theory}

We provide some theoretical results to support our methods. To understand the robustness of non-parametric methods, we first derive a theoretically optimal classifier that takes into account robustness as a core objective. Then, we show that adversarial pruning can be interpreted as a finite sample approximation to the optimally robust classifier. Finally, we analyze the exact and approximate region-based attacks.

\subsection{Adversarial Pruning vs. Optimal}

Under certain conditions, many non-parametric methods converge in the infinite sample limit to the {\em Bayes Optimal classifier}, the most accurate classifier for a data distribution. In this way, non-parametric classifiers may be viewed as {{finite-sample approximations}} to the Bayes Optimal. 
However, the Bayes Optimal may not be robust to adversarial examples. 

We next introduce a novel robust analogue to the Bayes Optimal. For a perturbation amount~$r$, we call it the {\em$r$-Optimal classifier}.
Surprisingly, to the best of our knowledge, such an analogue seems to be new in the context of adversarial examples.

Let $\mu$ denote a distribution on labeled examples $\calX \times [C]$ and fix a distance on $\calX$. 
What is the true objective of a robust classifier?  Prior work measures
astuteness under $\mu$, which is the probability that the classifier is both $r$-robust and accurate for a new sample $(\x, y)$~\citep{madry2017, WJC18}. 

\begin{definition}\label{def:astute}
	For distribution $\mu$ on $\calX \times [C]$, the {\em astuteness} of a classifier $f$ at radius $r$ is
	\[
	\astu_\mu(f,r) := \Pr_{(\x,y) \sim \mu(\x)}[\rho(f,\x) \geq r \mathrm{\ and\ } f(\x) = y].
	\]
\end{definition}

{\bf Robust Analogue to Bayes Optimal.}
We exhibit a classifier, the $r$-Optimal classifier, that achieves optimal astuteness. It is convenient to rewrite astuteness in terms of certain robust subsets of the input space. Then, we define the $r$-Optimal classifier using these subsets. Formally, for a classifier $f$ and label $j$, let 
$$S_j(f,r) := \{\x \in \calX \mid f(\x) = j \mbox{ and } \rho(f,\x) \geq r\}.$$

We now define the $r$-Optimal classifier and prove that it maximizes astuteness. This result hinges on the next lemma, which rewrites astuteness in a more convenient form. Let $\mu$ be a distribution on labeled examples $\calX \times [C]$. The following lemma expresses astuteness under $\mu$ using these subsets. 

\begin{lemma} \label{lem:astute-integral} 
$$\displaystyle
	\astu_\mu(f,r) =  \sum_{j=1}^{C} \int_{\x \in S_j(f,r)} p(y = j \mid  \x) d \mu(\x).
$$
\end{lemma}
\begin{proof}
	Recall the definition of the robust regions of a classifier,
	\[S_j(f,r) = \{\x \in \calX \mid f(\x) = j \mbox{ and } \rho(f,\x) \geq r\}.\]  
	Starting with the definition of astuteness, we compute the following.
	\begin{align*}
		&\astu_\mu(f,r) \\
		&= 
		\Pr_{(\x,y) \sim \mu}[\rho(f,\x) \geq r \mathrm{\ and\ } f(\x) = y]
		\\
		&=
		\int_{\x} p(y \mid \x) \cdot \mathbf{1}_{\{ \rho(f,\x) \geq r \} } \cdot 
		\mathbf{1}_{\{ f(\x) =y\} } \ d \mu(\x)
		\\
		&=
		\sum_{j=1}^{C} \int_{\x} p(y = j \mid \x) \cdot \mathbf{1}_{\{ \rho(f,\x) \geq r \} } \cdot \mathbf{1}_{\{ f(\x) =j \} }\ d \mu(\x)
		\\
		&=
		\sum_{j=1}^{C} \int_{\x \in S_j(f,r)} p(y = j \mid \x)\ d \mu(\x).
	\end{align*}
\end{proof}

How should we define the classifier that maximizes astuteness?
Lemma~\ref{lem:astute-integral}  implies that, to calculate astuteness, it suffices to consider the robust regions $S_j(f,r)$ for a classifier. As a consequence, we claim that in order to determine the optimal classifier, it suffices to find the optimal robust regions under $\mu$. We first formalize this intermediate goal using the following maximization problem.
\begin{flalign}
\label{eqn:roptimal}
&\max_{S_1,\ldots, S_C}\ \sum_{j=1}^{C} \int_{\x \in S_j} p(y = j \mid \x)\  d\mu(\x) \\ &\mbox{s.t. } d(S_j, S_{j'}) \geq 2r \mbox{ for all } j \neq j', \nonumber
\end{flalign}
where $d(S_j, S_{j'}) := \min_{u \in S_j, v \in S_{j'}} \| u-v\|$.
Notice that for any classifier $f$, the sets $S_j(f,r)$ for $j \in [C]$ have pairwise distance at least $2r$, implying that they  are feasible  solutions for (\ref{eqn:roptimal}).

Besides being distance $2r$ apart, an optimal solution $S_1^*,\ldots, S_C^*$ to (\ref{eqn:roptimal}) maximizes accuracy in the following sense.  The integral measures the probability that $(\x, y)\sim \mu$ has $y=j$ and $\x \in S_j^*$. In other words, $S_j^*$ has the highest frequency of points with label $j$ under $\mu$, subject to the  distance constraint. 

The sets $S_j^*$ form the basis for the optimal classifier's decision regions. To ensure the separation, we consider the distance $r$ ball around these sets. Formally, we have the following.

\begin{definition} Fix $r$ and $\mu$. Let $S_1^*,\ldots, S_C^*$ be optimizers of (\ref{eqn:roptimal}). The {\em $r$-Optimal classifier} $\ropt$ is any classifier such that $\ropt(\x) = j$ whenever $d(\x, S_j^*) \leq r$. 
\end{definition}

We remark that when $r = 0$, the $0$-Optimal classifier is the standard Bayes Optimal classifier. 

Finally, because $S_j(\ropt, r) = S_j^*$, Lemma~\ref{lem:astute-integral} then implies that $r$-Optimal classifier maximizes astuteness:

\begin{theorem}\label{thm:ropt}
	$\ropt = \argmax_f \astu_\mu(f,r).$
\end{theorem}
\begin{proof}
	Recall that the $r$-Optimal classifier $\ropt$ is defined in terms of an optimal solution $S_1^*,\ldots, S_C^*$ to the maximization problem (\ref{eqn:roptimal}).
	By definition, $\ropt(\x) = j$ whenever $d(S_j^*, \x) \leq r$. In other words, $S_j^* = S_j(\ropt,r)$.
	
	We will need the fact that for any classifier $f$, the sets $S_j(f,r)$ are a feasible solution to the above maximization problem. That is, for $j \neq j'$, the distance between $S_j(f,r)$ and $S_{j'}(f,r)$ is at least~$2r$. To see this, consider any two points $u \in S_j(f,r)$ and $v \in S_{j'}(f,r)$. Then, consider the line segment between them $w = \lambda u + (1-\lambda)v$, for $\lambda \in [0,1]$. By definition of the robustness radius, we know that $f(w) = f(u) = j$ whenever $d(w,u) \leq r$. Similarly, $f(w) = f(v) = j'$ whenever $d(w,v) \leq r$. Therefore, we must have that $d(u,v) \geq 2r$. As $u$ and $v$ were an arbitrary pair of points in $S_j(f,r)$ and $S_{j'}(f,r)$, we conclude that these subsets have distance at least $2r$, and this holds for all $j \neq j'$.

	Using  Lemma~\ref{lem:astute-integral}, we now compute the following.
	\begin{eqnarray*}
		\astu_\mu(f,r) 
		&=& \sum_{j=1}^{C} \int_{\x \in S_j(f,r)} p(y = j \mid \x)\ d \mu(\x)
		\\
		&\leq& 
		\sum_{j=1}^{C} \int_{\x \in S_j^*} p(y = j \mid \x)\ d \mu(\x)
		\\
		&=&
		\sum_{j=1}^{C} \int_{\x \in S_j(\ropt,r)} p(y = j \mid \x)\ d \mu(\x)
		\\
		&=& \astu_\mu(\ropt,r).
	\end{eqnarray*}
	
	The inequality uses that the sets $S_j(f,r)$ have pairwise distance at least $2r$, and therefore, they are feasible for the above maximization problem, which has optimal solution $S_j^* = S_j(\ropt,r)$.
\end{proof}

{\bf Finite Sample Approximation.} Prior work shows that $1$-NN applied to a variant of adversarial pruning leads to provably robust classifiers~\citep{WJC18}.
The main difference with our work is their method also selects a subset of confident training examples to keep in the pruned subset - which ensures that the classifier converges to being robust in regions where the Bayes Optimal is robust.
In contrast, our aim is to develop generic techniques, for multiple classifiers, and we show that our method can be interpreted as a finite sample approximation to the $r$-Optimal classifier -- the optimally astute classifier.

Adversarial pruning works by removing certain training points so that no oppositely labeled pairs of examples remain. We can view this process in the light of the $r$-optimal classifier as follows. To prune the dataset $\calS$, we solve the maximization problem:
\begin{flalign} \label{eqn:roptimalapprox}
&\max_{S_1,\ldots, S_C \subseteq \calS} \ \sum_{j=1}^{C} \sum_{\x^i \in S_j} \mathbf{1}_{\{y^i = j\}} \\ 
&\mbox{s.t. } \nonumber
 d(S_j, S_{j'}) \geq 2r \mbox{ for all } j\neq j'.\nonumber 
\end{flalign}
The solution to~\eqref{eqn:roptimalapprox} will be maximum subsets of training data with pairwise distance~$2r$. As long as the training set $\calS$ is representative of the underlying distribution $\mu$, these subsets will approximate the optimal $S_j^*$ sets. %
Hence, we posit that a non-parametric method trained on $\spruned$ should approximate the $r$-Optimal classifier.

\subsection{Attack Algorithm Analysis}

The run time of the region-based attack depends on the norm. We focus on $\ell_p$ with $p \in \{1,2,\infty\}$ as these
are the most relevant for adversarial examples. 
We prove the following theorem in Appendix~\ref{app:attack}.

\begin{theorem}\label{thm:main}
	If $f$ is $(s,m)$-decomposable, then the region-based attack outputs optimal adversarial examples in time $s \cdot \poly(m, d)$, for $\ell_p$ distance, $p \in \{1,2,\infty\}$.
\end{theorem}

As $k$-NN and tree ensembles are $(s,m)$-decomposable, the region-based
attack produces an optimal adversarial example for these.
Note that an optimal attack {\em certifies} the robustness radius.
Indeed, if on input $\x$ the region-based attack outputs $\wt \x$, then $\|\x
- \wt \x\| = \rho(f,\x)$.

We leave it as an interesting open question to develop provably optimal algorithms with better running time. For example, in the case of large tree ensembles, the attack searches over all combinations of one leaf from each tree. This seems wasteful, as many of these polyhedra may be empty (in fact, we find that most potential regions are infeasible for random forests trained on real datasets). 

{\bf Approximate Attack Guarantees.} 
We claim that the approximate region-based attack outputs a valid adversarial example when $f$ is $(s,m)$-decomposable. Each region is defined by $m$ constraints, and $f$ is constant on each region. We search in $s'$ regions, finding the best candidate $\z^i$ from each. Each considered region contains a training example with a different label than $f(\x)$. Therefore, the best adversarial example $\widetilde{\x}$ in that region receives a different label $f(\widetilde{\x}) \neq f(\x)$.  The analysis of the time complexity for finding candidates is $\poly(m, d)$ for each region $P_i$. Compared to the exact attack (Theorem~\ref{thm:main}) we only consider $s'$ regions, so the total time is only $s' \cdot \poly(m, d)$.  We find in practice that $s' = 50$ regions suffices for a good attack, and the time only scales with $m$ and $d$.  

\section{Experiments}\label{sec:experiments}

We investigate the effectiveness of our methods by evaluating multiple classifiers on nine datasets.
We address the following questions:
\begin{enumerate}
    \item Does adversarial pruning increase robustness across multiple non-parametric classifiers?
    \item How well does the region-based attack perform compared with prior work?
\end{enumerate}

\begin{table*}[t]
	\small
	\centering
	\setlength{\tabcolsep}{3.0pt}
	\begin{tabular}{l|ccccc|cccc|ccc|cc}
        \toprule
                & \multicolumn{5}{c}{1-NN} & \multicolumn{4}{c}{3-NN} & \multicolumn{3}{c}{DT} & \multicolumn{2}{c}{RF} \\
                & Direct & BBox & Kernel & \makecell{RBA\\Exact} & \makecell{RBA\\Appr.} & Direct & BBox & Kernel & \makecell{RBA\\Appr.} & Papernot's & BBox & \makecell{RBA\\Exact} &    BBox & \makecell{RBA\\Appr.} \\
        \midrule
        austr. &             $.442$ &            $.336$ &                       $.379$ &           $\mathbf{.151}$ &               $\mathbf{.151}$ &             $.719$ &            $.391$ &                       $.464$ &               $\mathbf{.278}$ &                $.140$ &            $.139$ &       $\mathbf{.070}$ &      $\mathbf{.364}$ &                     $.446$ \\
        cancer &             $.223$ &            $.364$ &                       $.358$ &           $\mathbf{.137}$ &               $\mathbf{.137}$ &             $.329$ &            $.376$ &                       $.394$ &               $\mathbf{.204}$ &                $.459$ &            $.334$ &       $\mathbf{.255}$ &               $.451$ &            $\mathbf{.383}$ \\
        covtype &             $.130$ &            $.199$ &                       $.246$ &           $\mathbf{.066}$ &                        $.067$ &             $.200$ &            $.259$ &                       $.280$ &               $\mathbf{.108}$ &                $.254$ &            $.083$ &       $\mathbf{.051}$ &               $.233$ &            $\mathbf{.214}$ \\
        diabetes &             $.074$ &            $.112$ &                       $.165$ &           $\mathbf{.035}$ &               $\mathbf{.035}$ &             $.130$ &            $.143$ &                       $.191$ &               $\mathbf{.078}$ &                $.237$ &            $.133$ &       $\mathbf{.085}$ &      $\mathbf{.181}$ &                     $.184$ \\
        f-mnist06 &             $.080$ &            $.140$ &                       $.187$ &           $\mathbf{.029}$ &                        $.030$ &             $.129$ &            $.169$ &                       $.202$ &               $\mathbf{.051}$ &                $.189$ &            $.134$ &       $\mathbf{.079}$ &               $.206$ &            $\mathbf{.188}$ \\
        f-mnist35 &             $.187$ &            $.244$ &                       $.259$ &           $\mathbf{.075}$ &                        $.077$ &             $.234$ &            $.238$ &                       $.266$ &               $\mathbf{.094}$ &                $.262$ &            $.185$ &       $\mathbf{.115}$ &      $\mathbf{.188}$ &                     $.246$ \\
        fourclass &             $.109$ &            $.124$ &                       $.137$ &           $\mathbf{.090}$ &               $\mathbf{.090}$ &             $.101$ &            $.113$ &                       $.134$ &               $\mathbf{.096}$ &                $.288$ &            $.197$ &       $\mathbf{.137}$ &               $.159$ &            $\mathbf{.133}$ \\
        halfmoon &             $.070$ &            $.129$ &                       $.102$ &           $\mathbf{.058}$ &               $\mathbf{.058}$ &             $.105$ &            $.132$ &                       $.115$ &               $\mathbf{.096}$ &                $.098$ &            $.148$ &       $\mathbf{.085}$ &               $.182$ &            $\mathbf{.149}$ \\
        mnist17 &             $.161$ &            $.251$ &                       $.262$ &           $\mathbf{.070}$ &                        $.073$ &             $.221$ &            $.261$ &                       $.269$ &               $\mathbf{.097}$ &                $.219$ &            $.171$ &       $\mathbf{.123}$ &      $\mathbf{.250}$ &            $\mathbf{.250}$ \\
        \bottomrule
\end{tabular}

	\caption{
        The Empirical Robustness for different attacks on four  classifiers {\bf (lower is better; best is in bold)}.
        }
	\label{table:compare_attacks_avg_pert-main}
\end{table*}

\subsection{Experimental Setup}

\textbf{Classifiers and Datasets.} We evaluate three non-parametric classifiers:
$k$-nearest neighbor ($k$-NN), decision tree (DT)
 and random forest (RF)
\citep{breiman2001random, breiman2017classification, cover1967nearest}.
We use nine standard binary classification
datasets.
All features are scaled to be in [0,1].  We evaluate in $\ell_\infty$ to be consistent with prior work.
We reduce the feature dimension of the image datasets (f-mnist and mnist)
with PCA to 25 dimensions for two reasons: (i) non-parametrics are normally used for low
dimensional spaces, (ii)
adversarial pruning requires non-binary features for $\ell_\infty$.
Details are in Appendix~\ref{app:exp}; code
 in a public repository.\footnote{\url{https://github.com/yangarbiter/adversarial-nonparametrics/}}

\textbf{Performance Measures.}
Besides measuring accuracy, we evaluate attacks using {{empirical
robustness}}, following prior work~\citep{chen2019robust, kantchelian}.
Intuitively, we want to measure the perturbation distance to the nearest adversarial example (as opposed to fixing $r$ and evaluating error). Formally, the {\em empirical robustness} for attack $A$ on $f$ at input $\x$ is $\emprob(A,f,\x) := \|\x-\wt \x_A\|_\infty$, where $A$ outputs $\wt \x_A$ as the adversarial example for $f$ at $\x$.
Observe that larger empirical robustness means worse attacks, and  the minimal
empirical robustness of $f$ at $\x$ is the robustness radius $\rho(f, \x)$.
To fairly compare classifiers having different accuracies, we actually
compute $\emprob(A, f, S, t)$ over $t$ test inputs.
To do so, we draw $t$ random samples $S_t$ from $S$ that are classified
correctly by $f$, and we report the average of $\emprob(A, f, \x)$ over $\x
\in S_t$. We set $t=100$ to balance efficiency and thoroughness.

Again, for defenses, we use perturbation distance to evaluate robustness.
Each defense method $D$ produces a classifier $f_D$. We evaluate a defense $D$ by assigning it
a score, the \defenderscore. The {\em \defenderscore} with respect to an attack $A$, a
test set $S$ and test size $t$ is the ratio 
\[\mbox{\defenderscore}(D, A, f, S, t) = \frac{\emprob(A, f_D, S, t)}{\emprob(A,
f, S, t)},\] where $f$ is the undefended classifier. A larger \defenderscore implies a better defense. 
Whenever feasible, we use the optimal attack while calculating the \defenderscore.

\textbf{Attack Algorithms.}  
For 1-NN and DT, we apply the exact region-based attack (RBA-Exact).
For 3-NN and RF, the RBA-Exact attack is computationally intensive, and we use the approximate region-based attack (RBA-Approx). For 3-NN, it uses $s'=50$ polyhedra, and for RF, it uses $s'=100$ polyhedra.
We compare RBA-Exact and RBA-Approx against several
baselines.  A general attack that applies to all methods is
the black-box attack (BBox)~\citep{cheng2019query};
this attack seems to be the state-of-the-art for
non-parametrics.
For $k$-NN, we compare against two white-box attacks, the direct attack (Direct) and kernel substitution
attack (Kernel)~\citep{papernot2016transferability}.
The direct attack perturbs the test instance towards the center of the $k$
nearest oppositely-labeled training examples.
The kernel substitution attack uses a soft nearest neighbor to build a
substitution model and applies the
projected gradient descent attack~\citep{Kurakin:2016}.
For DT, the RBA-Exact attack is optimal, and so is the attack by \citet{kantchelian};
we only report RBA-Exact because these achieve the same results.
We also evaluate the heuristic DT attack by \citet{papernot2016transferability}.
For RF, both optimal attacks are infeasible, and we only evaluate BBox and RBA-Approx.

\textbf{Defense Methods.}
For our defense, we train each classifier on the dataset pre-processed with adversarial pruning~(AP);
we use $\ell_\infty$ to determine examples to prune.
For the separation $r$ of AP, we found that $r=0.3$ balances robustness vs. accuracy.
We set $r = 0.3$ for all datasets (Appendix~\ref{app:defense} has other $r$ settings).
A generic baseline is adversarial training (AT), where the training data is
augmented with examples generated by the corresponding attack algorithm.
AT has been reported to be ineffective for 1-NN and boosted decision tree
\citep{WJC18,chen2019robust}, but we include it for completeness.
For AT, we retrain the classifier after attacking each training point once; we augment the training data with adversarial examples that are distance at most $0.3$ from the original input.
The parameter $0.3$ matches the parameter $r$ for AP.
For $1$-NN, an available baseline defense is \citet{WJC18}, but for general $k$-NN, we are not aware of other defenses.  For DT
and RF, we compare against the best known defense algorithm, Robust Splitting
(RS)~\citep{chen2019robust}.
We set the RS parameter to $0.3$ as well. 

\begin{table*}[t]
	\small
	\centering
	\setlength{\tabcolsep}{6.0pt}
	\begin{tabular}{l|ccc|cc|ccc|ccc}
        \toprule
                &\multicolumn{3}{c}{1-NN} &\multicolumn{2}{c}{3-NN} &\multicolumn{3}{c}{DT} &\multicolumn{3}{c}{RF} \\
                & AT & WJC & AP & AT & AP & AT & RS & AP & AT & RS & AP \\
        \midrule
        austr. &                    $0.64$ &           $\mathbf{1.65}$ &           $\mathbf{1.65}$ &                   $0.68$ &               $\mathbf{1.20}$ &                  $2.36$ &            $\mathbf{5.86}$ &                         $2.37$ &                     $1.07$ &            $\mathbf{1.12}$ &                     $1.04$ \\
        cancer &                    $0.82$ &                    $1.05$ &           $\mathbf{1.41}$ &                   $1.06$ &               $\mathbf{1.39}$ &                  $0.85$ &                     $1.09$ &                $\mathbf{1.19}$ &                     $0.87$ &            $\mathbf{1.54}$ &                     $1.26$ \\
        covtype &                    $0.61$ &           $\mathbf{4.38}$ &           $\mathbf{4.38}$ &                   $0.88$ &               $\mathbf{3.31}$ &                  $1.47$ &                     $2.73$ &                $\mathbf{4.51}$ &                     $1.02$ &                     $1.01$ &            $\mathbf{2.13}$ \\
        diabetes &                    $0.83$ &           $\mathbf{4.69}$ &           $\mathbf{4.69}$ &                   $0.87$ &               $\mathbf{2.97}$ &                  $0.93$ &                     $1.53$ &                $\mathbf{2.22}$ &                     $1.19$ &                     $1.25$ &            $\mathbf{2.22}$ \\
        f-mnist06 &                   $0.90$ &                    $1.93$ &           $\mathbf{2.59}$ &                  $0.88$ &               $\mathbf{1.75}$ &                  $1.33$ &                     $2.33$ &                $\mathbf{2.57}$ &                     $1.04$ &                     $1.10$ &            $\mathbf{1.77}$ \\
       f-mnist35 &                   $0.83$ &                    $1.05$ &           $\mathbf{1.19}$ &                   $0.83$ &               $\mathbf{1.15}$ &                  $0.97$ &            $\mathbf{3.03}$ &                         $2.06$ &                     $0.99$ &                     $1.23$ &            $\mathbf{1.41}$ \\
        fourclass &                    $0.93$ &           $\mathbf{3.09}$ &           $\mathbf{3.09}$ &                   $0.89$ &               $\mathbf{3.09}$ &                  $1.06$ &                     $1.23$ &                $\mathbf{3.04}$ &                     $1.03$ &                     $1.92$ &            $\mathbf{3.59}$ \\
        halfmoon &                    $1.05$ &                    $2.00$ &           $\mathbf{2.78}$ &                   $0.93$ &               $\mathbf{1.92}$ &                  $1.54$ &                     $1.98$ &                $\mathbf{2.58}$ &                     $1.04$ &                     $1.01$ &            $\mathbf{1.82}$ \\
        mnist-17 &                   $0.88$ &                    $1.06$ &           $\mathbf{1.39}$ &                   $0.80$ &               $\mathbf{1.13}$ &                  $1.11$ &            $\mathbf{3.97}$ &                         $1.32$ &                     $0.88$ &                     $0.92$ &            $\mathbf{1.26}$ \\
        \bottomrule
\end{tabular}
	\caption{
		\defenderscore using
		different defenses for four different classifiers {\bf (higher is better; best is in bold)}. The
		\defenderscore for undefended classifiers is 1.00 (greater than 1.00 is more robust). We use RBA-Exact for
        $1$-NN and DT, and RBA-Approx for $3$-NN and RF. We use RBA-Approx for AT on large datasets. 
	}
	\label{table:compare_defense_avg_pert-main}
\end{table*}

\subsection{Results} 

We separately evaluate attacks and defenses, in Table~\ref{table:compare_attacks_avg_pert-main} and Table~\ref{table:compare_defense_avg_pert-main}, respectively. We also provide an accuracy vs. perturbation distance experiment in Figure~\ref{fig:compare_defenses-main}.

\textbf{Effectiveness of Attacks.}
Table \ref{table:compare_attacks_avg_pert-main} exhibits empirical robustness across
four undefended classifiers and nine datasets.
Recall that a smaller empirical robustness implies a more effective attack.
For 1-NN, we see that RBA-Exact works as expected, achieving the smallest
empirical robustness.
For 3-NN, our RBA-Approx attack is more effective than prior attacks, with a
much lower empirical robustness.
This indicates that RBA-Approx can be a strong attack for $k >1$, where
previously no consistently effective baseline is known.
For DT, RBA-Exact again has the best performance.
The improvement in many cases shows that the optimal attack for 1-NN and DT
can be significantly better than heuristics, which will lead to a more
informative defense evaluation.
For RF, RBA-Approx wins on five of the nine datasets, and BBox wins on four.
Overall, our RBA-Approx attack is competitive with the state-of-the-art attack for RF, and better for $3$-NN.

\begin{figure*}[t]
	\centering
	\subfloat[1-NN]{
		\includegraphics[width=.24\textwidth]{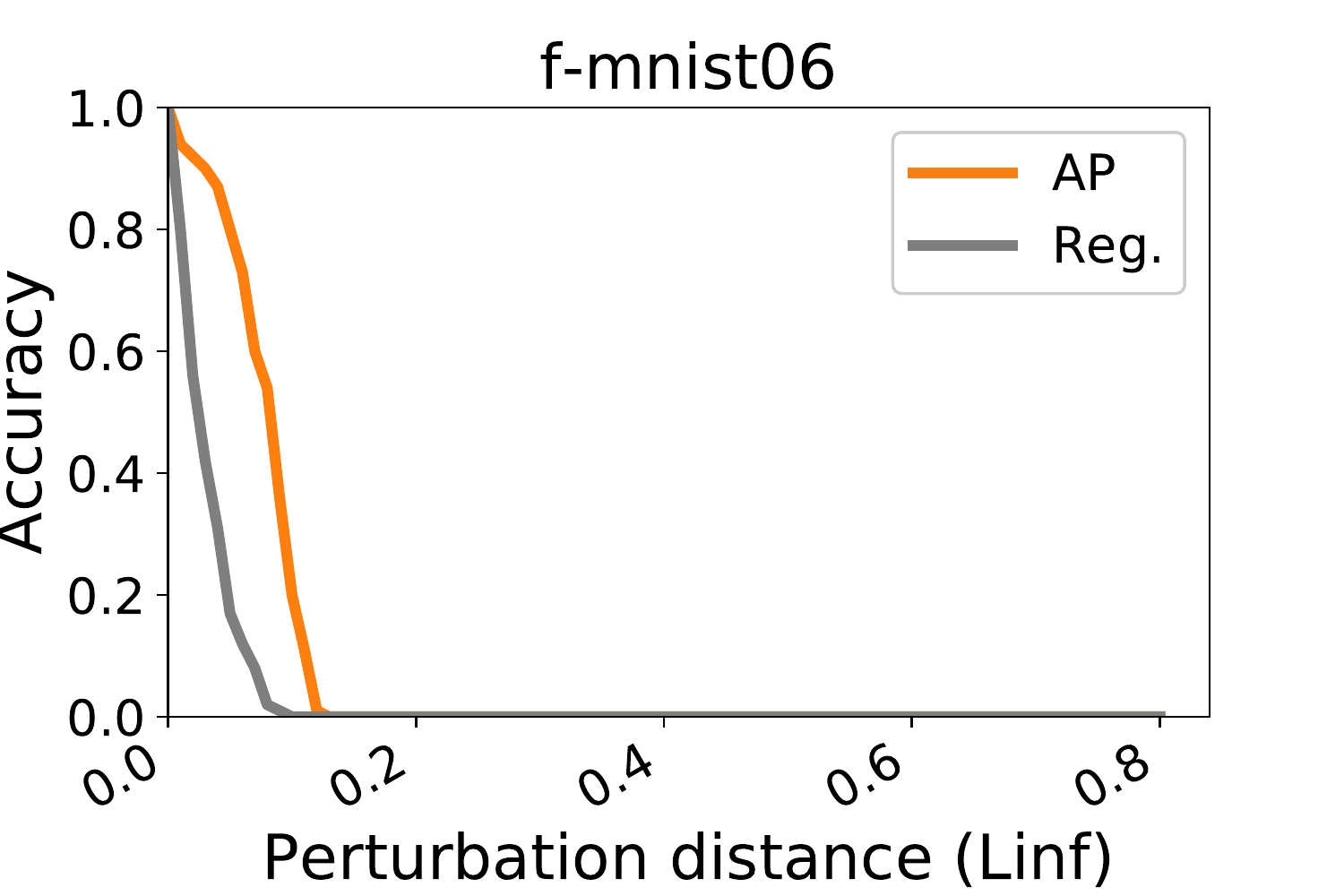}}
	\subfloat[3-NN]{
		\includegraphics[width=.24\textwidth]{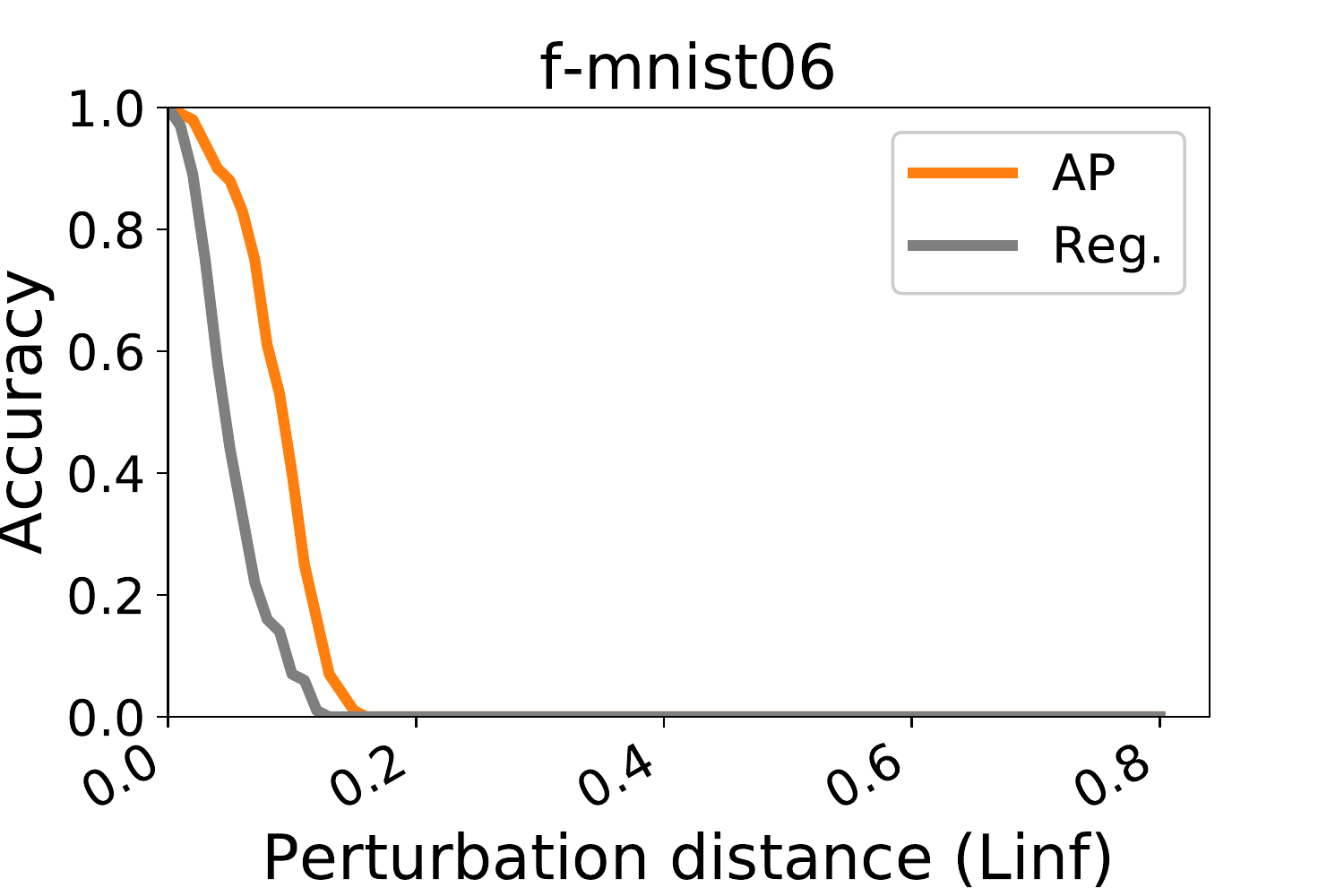}}
	\subfloat[Decision tree]{
		\includegraphics[width=.24\textwidth]{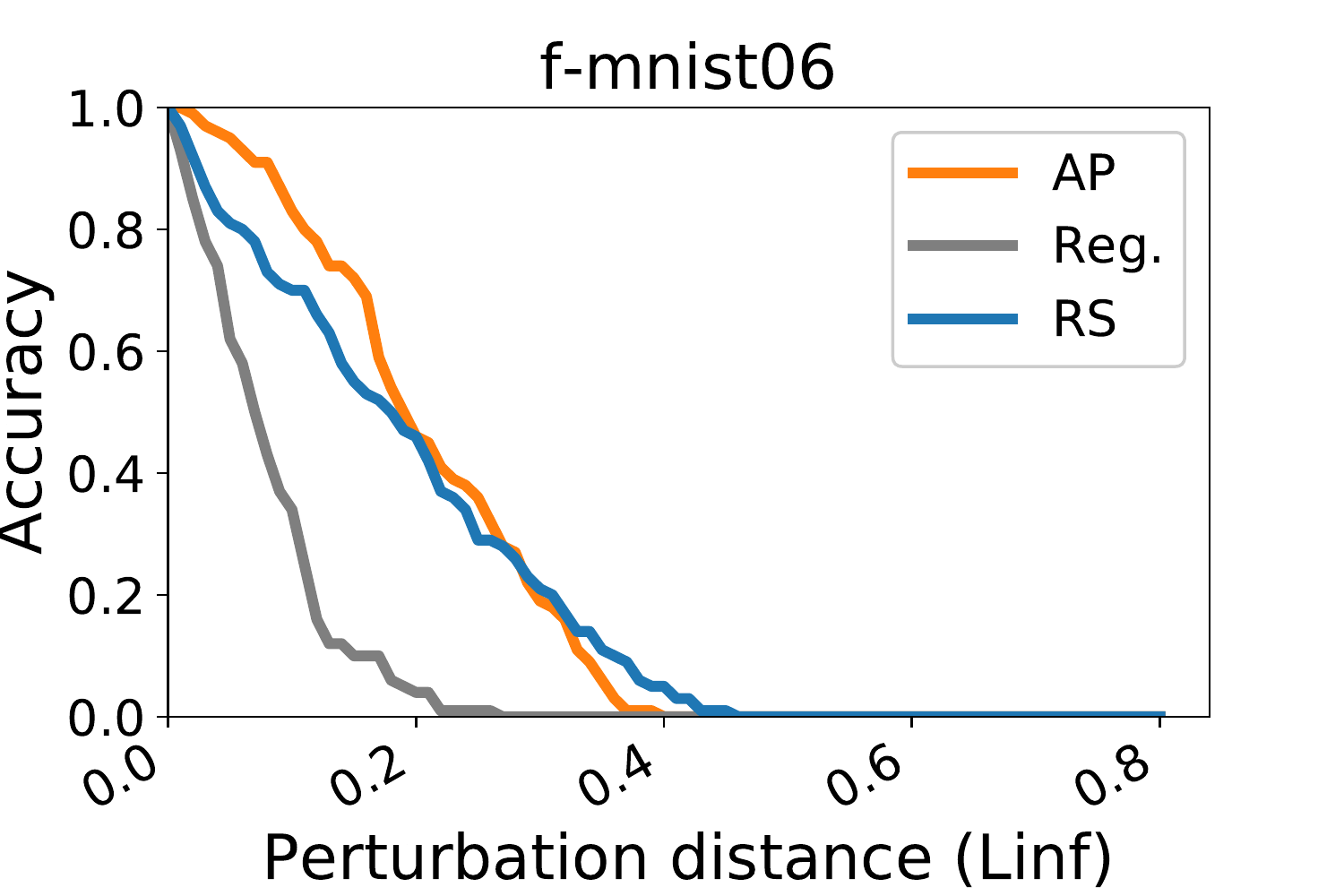}}
	\subfloat[Random forest]{
		\includegraphics[width=.24\textwidth]{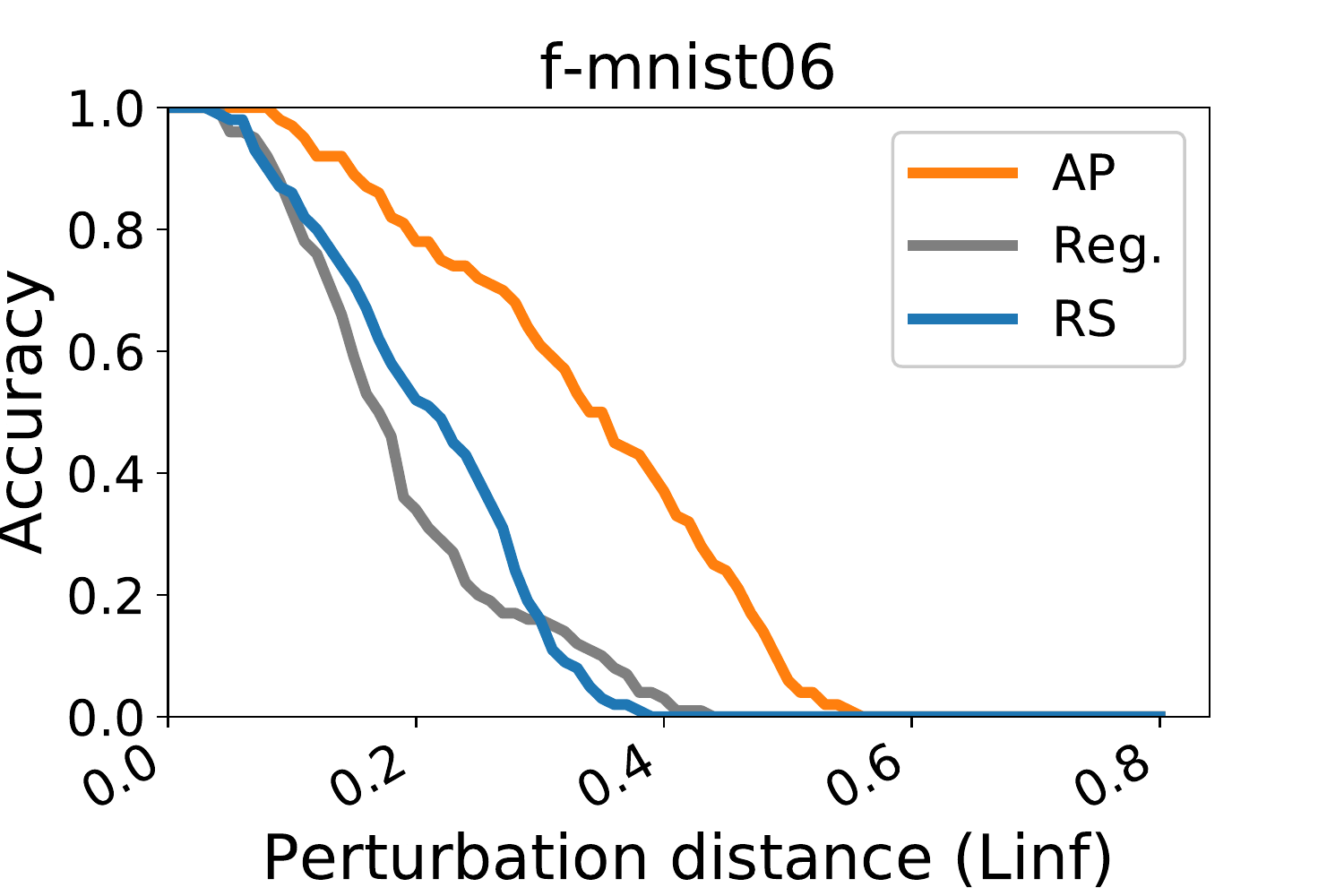}}
	
		\subfloat[1-NN]{
			\includegraphics[width=0.24\textwidth]{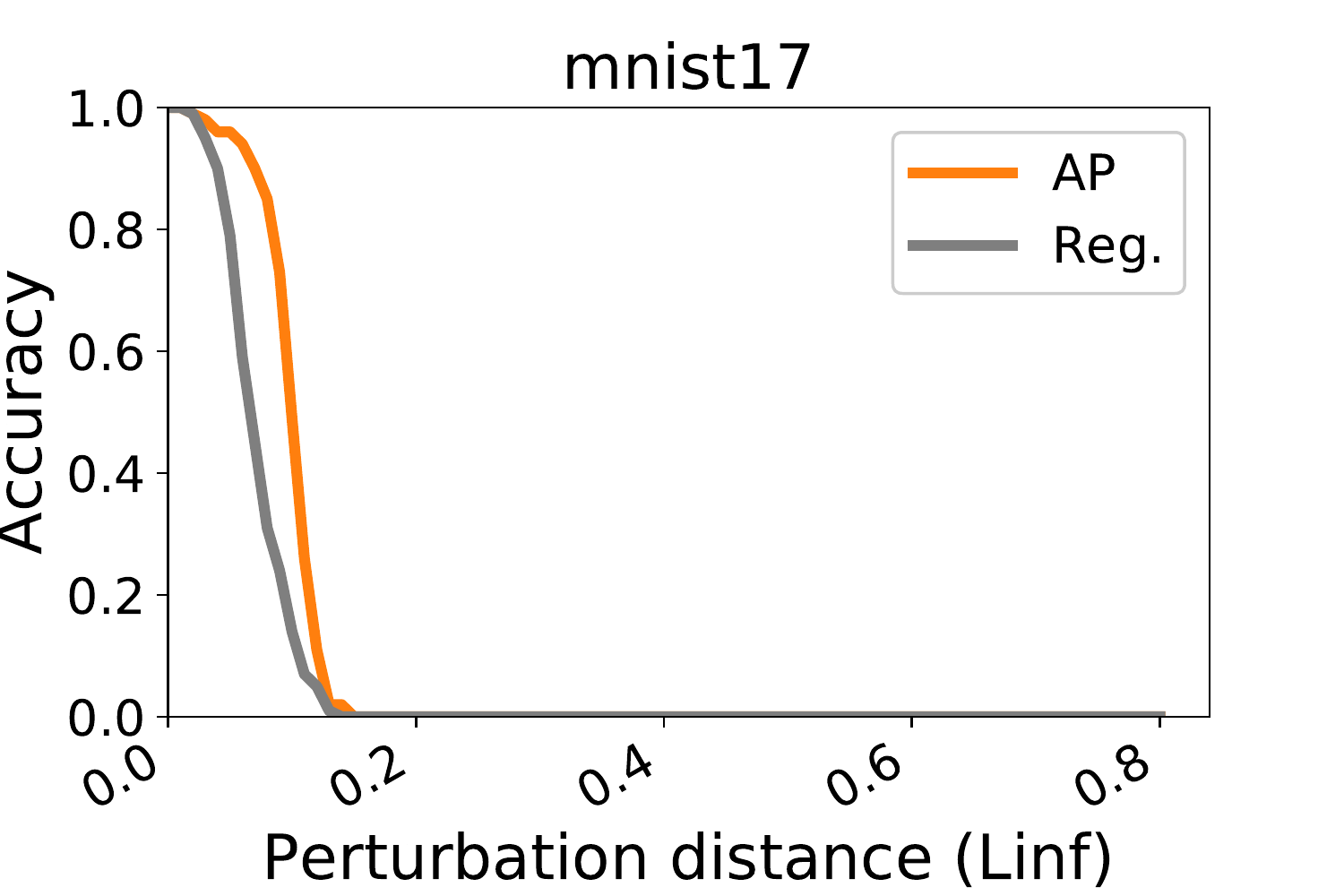}}
		\subfloat[3-NN]{
			\includegraphics[width=0.24\textwidth]{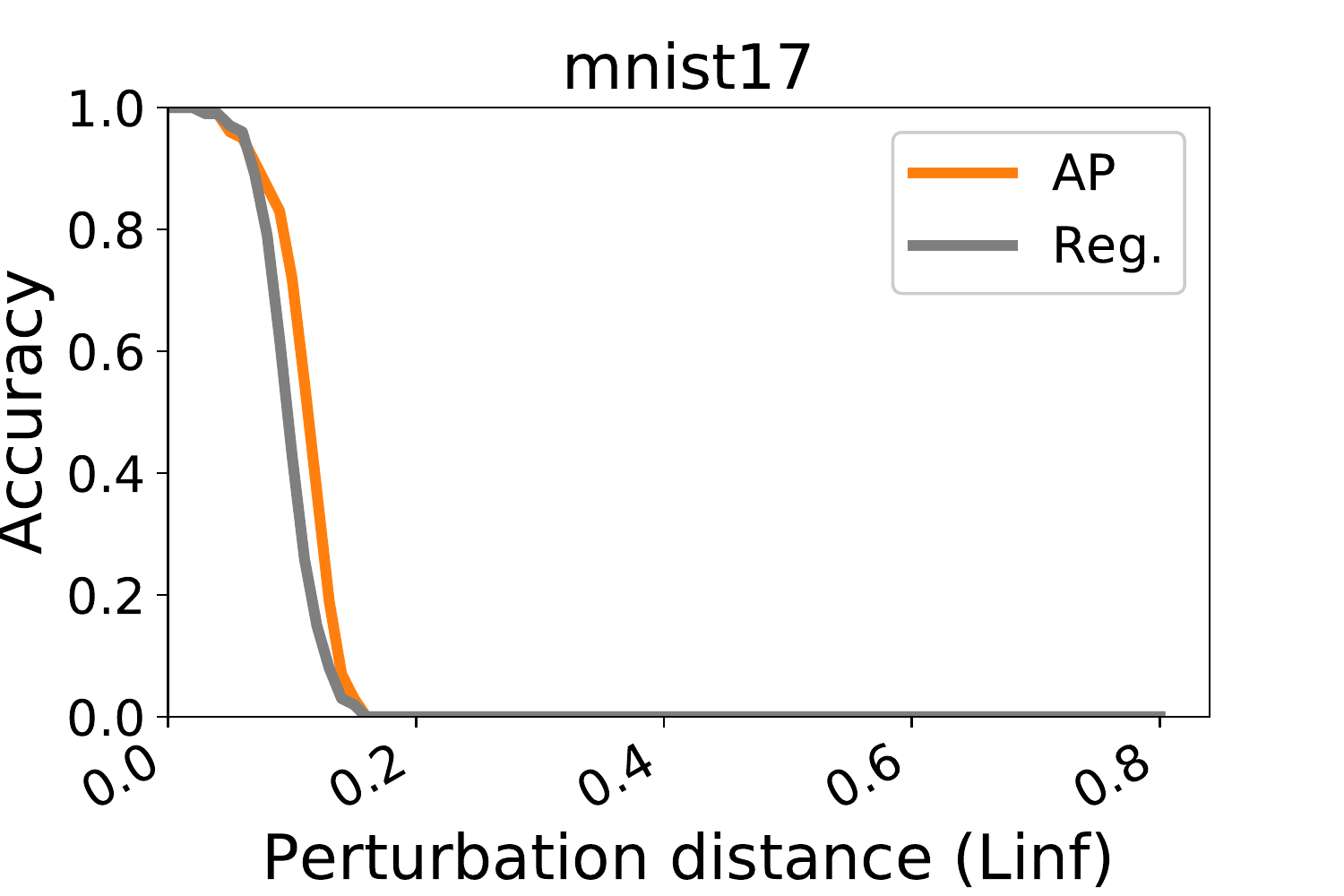}}
		\subfloat[Decision tree]{
			\includegraphics[width=0.24\textwidth]{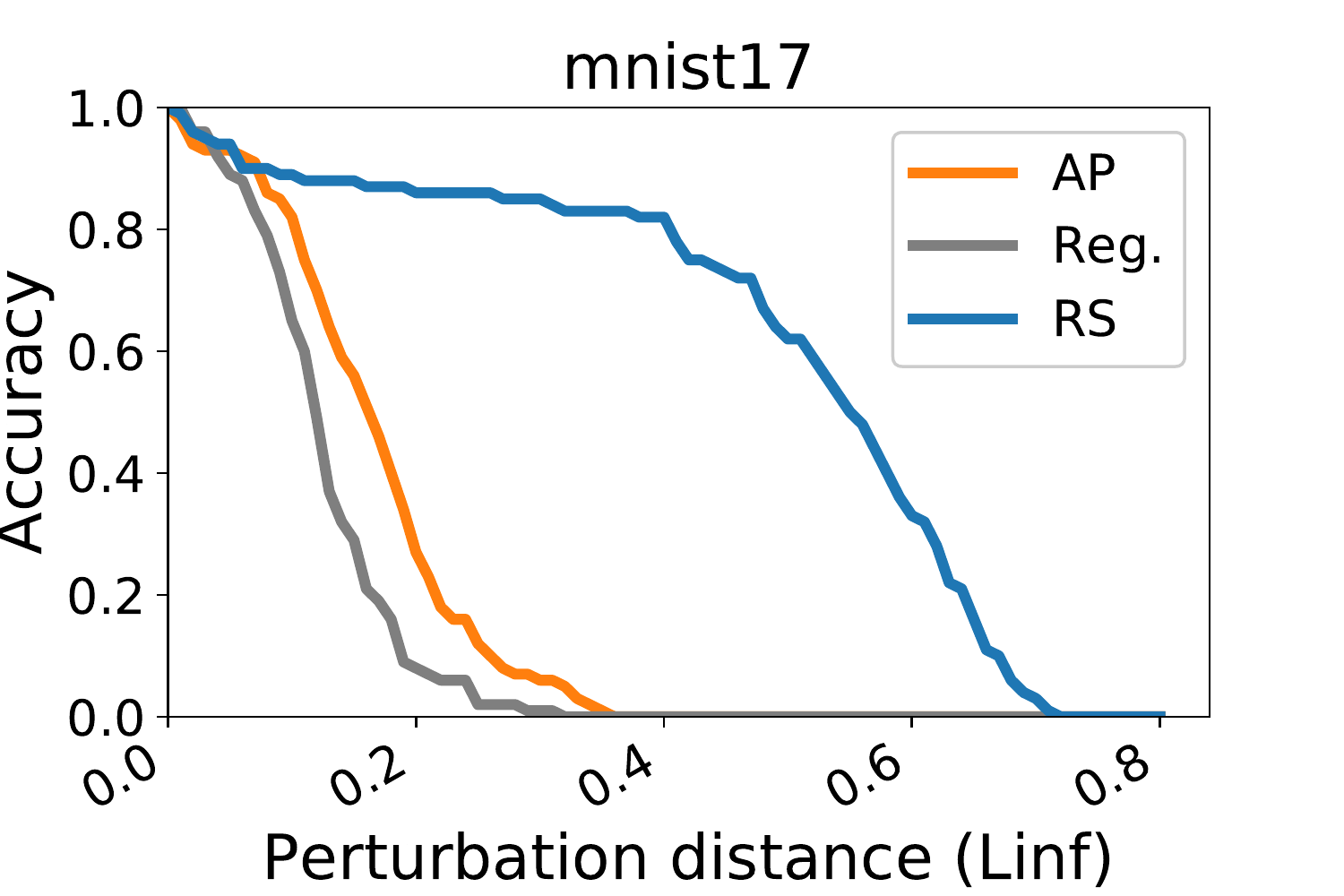}}
		\subfloat[Random forest]{
			\includegraphics[width=0.24\textwidth]{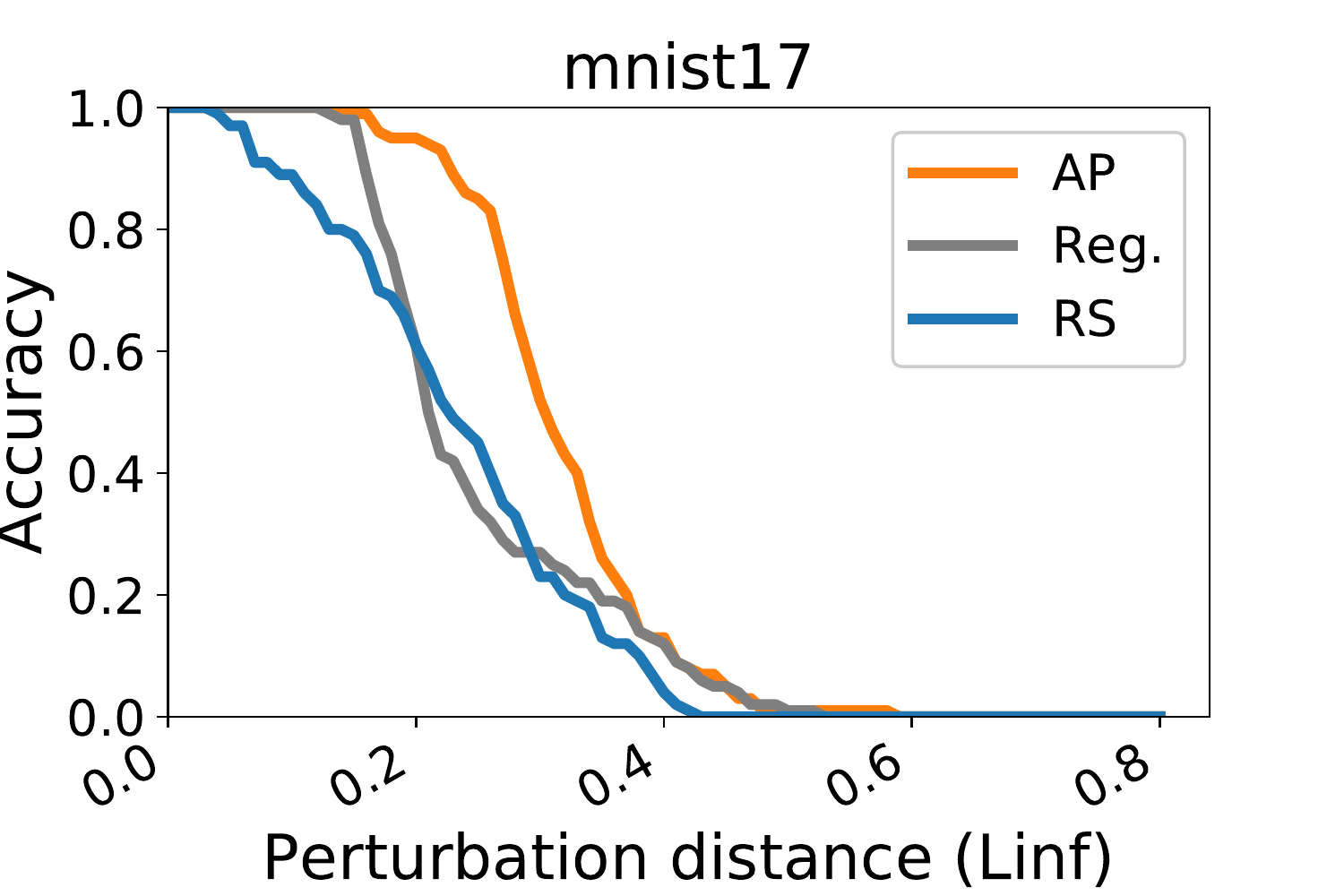}}
	\caption{Accuracy (y-axis) vs. perturbation distance (x-axis) for four classifiers on
		Fashion MNIST classes 0 vs. 6 (top row, subfigures (a)-(d)) and MNIST classes 1 vs. 7 (bottom row, subfigures (e)-(h)). We used the $\ell_\infty$ distance after applying PCA to 25 dimensions  {\bf(larger accuracy is better)}.
		Other datasets appear in Appendix~\ref{app:defense-figs}.
		In the legend, 
		Reg. $=$ regular (undefended) classifier,
		AP $=$ adversarial pruning, and
        RS $=$ robust splitting.
	}
	\label{fig:compare_defenses-main}
\end{figure*} 

\textbf{Effectiveness of Defenses.} Table
\ref{table:compare_defense_avg_pert-main} shows \defenderscore across four classifiers and several defense methods.
For each dataset, the AP defense trains all four classifiers on the same
pruned version of the dataset.
For all classifiers, we see that AP results in a greater than one
\defenderscore, indicating that classifiers trained with AP are more robust.
In contrast, AT usually achieves \defenderscore less than one, worse than the
undefended classifier; this corroborates previous results~\citep{WJC18}.
For $1$-NN, observe that AP is slightly better than the defense of \citet{WJC18}.
We believe that this is because their method converges to Bayes Optimal,
while AP approximates the $r$-Optimal classifier.
For the DT and RF experiments, we see that RS and AP perform competitively, each winning out on some datasets.
Overall, AP performs slightly better than RS.
We remark that we have evaluated 1-NN and DT against the optimal attack.
This provides concrete evidence that AP leads to a more robust
classifier.

\subsection{Discussion} 
From the results, we see that our generic attack and defense either outperform or perform competitively with prior work on many datasets. We note that there can be a big difference in the perturbation distance depending on the attack algorithms. We also see that
our adversarial pruning achieves more robustness compared both to undefended variants and to the classifiers trained using adversarial training. Surprisingly, the pruned subset is computed ahead of time, yet it improves the robustness of many different classifiers. 

The main conclusion from the experiments is that our work provides a new and suitable baseline for many methods. This is analogous to how AT and PGD are generic baselines for parametrics. In particular, if a new non-parametric algorithm is developed, then
AP and RBA may be used to evaluate robustness. Our work also opens to the door to combine AP with classifier-specific defenses, e.g. robust boosting~\citep{chen2019robust}. We note that our methods can sometimes be slow, but we expect that classier-specific optimizations and  techniques will readily improve the running time. 

\section{Related Work}\label{sec:related}

The bulk of previous research on robust classifiers has focused on parametric models, with many generic attacks~\citep{carliniwagner, bbox2,  papernot2017, papernot2015limitations,szegedy2013intriguing}, as well as several defenses~\citep{hein2017, barett, madry2017, distillation, raghunathan2018certified, duchi18, zhang2019theoretically}. In contrast, adversarial examples for non-parametric classifiers have been studied in a more case-by-case basis.

For tree ensembles, \citet{kantchelian} formulate an optimal attack as a Mixed Integer Linear Program (superseding an earlier attack~\citep{papernot2016transferability}) and prove NP-Hardness for many trees. \citet{chen2019robust} increase the robustness of {\em boosted} ensembles by introducing a more robust splitting criteria during training. Concurrent work also studies the robustness of decision stumps (i.e., random forests with depth-one trees), and we leave it as future work to compare our methods to theirs~\citep{andriushchenko2019provably}.

For $k$-NN, prior work on adversarial examples only considers suboptimal
attacks, such the direct attack and variants
thereof~\citep{amsaleg2017vulnerability, sitawarin2019robustness, WJC18}.
Concurrent work~\citep{khoury2019safeml} on Voronoi-based adversarial
training for neural networks also introduces the optimal attack for $1$-NN
(i.e., Region-Based attack restricted to 1-NN).  In terms of defenses,
\citet{WJC18} increase $1$-NN robustness by strategically removing training
points.  Besides only testing $1$-NN against suboptimal attacks, they do not
consider other non-parametrics; additionally, their
defense is shown to be robust in the large sample limit only where the Bayes
Optimal is robust.  Our methods are thus more general, and our defense can be
interpreted as a finite sample approximation to the $r$-Optimal classifier.

Outside the realm of adversarial examples, pruning has been used to improve the accuracy and generalization (but not robustness) of 1-NN~\citep{gates1972reduced, gottlieb2014near, hart1968condensed,  kontorovich2017nearest}. It would be interesting to revisit these works in the context of adversarial robustness, and in particular, in terms of the $r$-Optimal classifier.

Related attacks and defenses have been developed for ReLU
networks~\citep{croce2019provable, jordan2019provable,
tjeng2019evaluating, xiao2019training}. These results do not directly pertain
to non-parametrics, as ReLUs are fundamentally different. The geometric
attacks and defenses are similar in spirit to ours. Optimizations based on
the dual formulation may improve the efficiency of our
methods~\citep{tjeng2019evaluating, xiao2019training}. It would be
interesting to explore the relationship between our defense method
(adversarial pruning) and the ReLU defense methods and robustness
certificates.
For example, do robust ReLU networks approximate or converge to the $r$-Optimal classifier?

\section{Conclusion}\label{sec:conclusion}

We consider adversarial examples for non-parametric methods, with a focus on {\em generic} attacks and defenses. We provide a new attack, the region-based attack, which often outperforms previous attacks. We also provide a new method of defense, adversarial pruning, which should serve as a strong baseline for evaluating the robustness of many classifiers. 
On the theory side, we prove that the region-based attack outputs the optimal adversarial example. We also introduce and analyze a novel robust analogue to the Bayes Optimal. We prove that the $r$-Optimal classifier maximizes astuteness. On the experimental side, we demonstrate that our methods are better than or competitive with prior work, while being considerably more general.

{\bf Acknowledgments.}
We thank Somesh Jha, Ruslan Salakhutdinov and Michal Moshkovitz for helpful discussions. Part of this research is supported by ONR under N00014-16-1-261, UC Lab Fees under LFR 18-548554 and NSF under 1804829 and 1617157.

\newpage
\bibliographystyle{apalike}
\bibliography{nonparametricbib}
\clearpage

\appendix
\onecolumn

\section{Attack Algorithm: Theoretical Results and Omitted Proofs}
\label{app:attack}

In this section, we analyze the exact and approximate region-based attacks. To do so, we provide details about the decompositions for $k$-NN and tree ensemble classifiers. We also prove Theorem~\ref{thm:main} in general, and we give a corollary for the classifiers that we consider. Finally, we discuss our approximate attack, providing more details and an analysis.

Before getting into these details, we observe that our attack actually holds for the more general class of linear decision trees, which we now define.

\subsubsection*{Defining Linear Decision Trees}
A {\em linear decision tree} is a binary tree consisting of (i) internal nodes associated with affine functions and (ii) leaf nodes associated with labels in $\labels$. The value $f(\x)$ is determined by following the root to a leaf, going left or right depending on whether $\x$ satisfies or violates the linear constraint in the current node; then, $f(\x)$ is the label of the leaf. Such trees generalize (standard) decision trees, which restrict each constraint to depend on a single variable.

An {\em ensemble of linear decision trees} is collection of trees with the modification that the leaves are labeled with vectors in $\R^{C}$. The value $f(\x)$ is determined by a two-stage process. First, find the root-to-leaf path associated with each tree separately, resulting in a collection of vectors $\bfu^1,\ldots, \bfu^T \in \R^{C}$, where $T$ is the number of trees. Then, letting $\bfu = \bfu^1 + \cdots + \bfu^T$, the output $f(\x)$ equals the index of the largest coordinate $i \in \labels$ in the vector $\bfu$. Note that for binary labels, this is  equivalent to the definition of having scalar leaf labels and outputting the sign of the sum.

\subsection{Decompositions for Specific Classifiers}

We now describe the decompositions for tree ensembles and $k$-NN. Parameters for the decompositions will directly determine the running time of the optimal attack algorithm. 

\subsubsection*{Decomposition for Tree Ensembles}

\begin{lemma}\label{lem:tree-decomposition}
    If $f$ is an ensemble of $T$ linear decision trees, each with depth at most $D$ and with at most $L$ leaves, then $f$ is $(L^T, TD)$-decomposable.
\end{lemma}
\begin{proof} We first describe the decomposition for a single  tree, then generalize to an ensemble of trees. Let $\calT$ be a linear decision tree with depth $D$ leaves $(\ell_1, \ell_2,\ldots, \ell_m)$. The polyhedron $P_i$ will be the set of $\z$ that reach leaf $\ell_i$ in $\calT$. The hyperplane description for $P_i$ can be computed as follows. Each internal node $v$ from the root of $\calT$ to the leaf $\ell_i$ contains a linear constraint $a_v(\z) \leq b_v$. On the path to $\ell_i$, group all the violated (resp. satisfied) constraints $a_v, b_v$ as rows of the matrix $A^-$ and entries of the vector $\b^-$ (resp. $A^+$ and $\b^+$). Then, all $\z$ that reach $\ell_i$ are exactly the vectors that satisfy $A^- \z > \b^-$ and $A^+ \z \leq \b^+$. Therefore, these at most $D$ constraints determine $P_i$ precisely.
    
    Now, consider ensembles of $T$ trees with depth at most $D$ and at most $L$ leaves. The polyhedra correspond to combinations of one leaf from each tree.
    Each leaf contributes at most $D$ constraints, for at most $TD$ total constraints. There are at most $L^T$ choices for one leaf from each of $T$ trees.
\end{proof}

\subsubsection*{Decomposition for $k$-NN}

The decomposition for $k$-NN is a standard fact, known as the $\kth$ order Voronoi diagram, and it is a classical result in machine learning and computational geometry (see for example Chapter 12 in the book \citep{manning2010introduction}, or the survey \citep{aurenhammer1991voronoi}, or the paper \citep{mulmuley1991levels}). We sketch a proof for completeness.

\begin{lemma}\label{lem:knn-decomposition}
    If $f$ is a $k$-NN classifier for a dataset of size $n$, then $f$ is 
    $(\binom{n}{k}, k(n-k))$-decomposable.
\end{lemma}
\begin{proof}(Sketch). Let $\calS$ be the training dataset on $n$ points.
    We define $\binom{n}{k}$ convex polyhedra, one for each subset $U \subseteq \calS$ containing $|U|= k$ points. The polyhedron $P_U$ is the subset of $\R^d$ such that if $\z \in P_U$, then the $k$ nearest neighbors to $\z$ from the dataset $\calS$ in the $\ell_2$ distance are the $k$ points in~$U$. By definition, the $k$-NN classifier will be constant on each polyhedron $P_U$, as the output label is completely determined by the $k$ nearest neighbors for $\z$, which is the set $U$.
    
    We show that $P_U$ can be defined by $k(n-k)$ hyperplanes as follows. For each of the $k$ points $\x \in U$, we use the $(n-k)$ bisecting hyperplanes separating $\x$ from each of the $n-k$ points  not in~$U$ (that is, separating $\x$ from the points $\calS\setminus U$). This is a total of $k(n-k)$ linear constraints, and we define $P_U$ as the intersection of these $k(n-k)$ halfspaces. Clearly, $P_U$ is a convex polyhedron. 
    
    To see the nearest neighbor property, consider any $\z \in P_U$. For every $\x \in U$, the constraints defining $P_U$ include the $(n-k)$ bisecting hyperplanes that separate $\x$ from the $n-k$ points outside of $U$. In particular, $\z$ is closer to $\x$ than to these $n-k$ other points. To put this another way, $\z$ is in the Voronoi cell for $\x$ in the reduced dataset consisting only of $\x$ and the other $n-k$ points (that is, $\x \cup (\calS \setminus U)$). As this is true for each of the $k$ points in $U$, we have that $\z$ is closer to each of the $k$ points in $U$ than to the other $n-k$ points. Therefore, we conclude that $U$ consists of the $k$ nearest neighbors to $\z$.
\end{proof}

\subsection{Analyzing the Region-Based Attack}

We have just shown that $f$ is decomposable when it is the classifier determined by $k$-NN or a linear decision tree (or, more generally, an ensemble of linear decision trees). The consequence of this is that Theorem~\ref{thm:main} implies an efficient and optimal algorithm for a wide-range of non-parametric classifiers. We first discuss the specific convex programs, then finish the proof of the theorem.

\subsubsection*{Norms as Convex Objectives}\label{sec:objective}

Recall that if a classifier is $(s,m)$-decomposable,  then there exists $s$ polyhedra $P_1,\ldots, P_s$ such that each $P_i$ is the intersection of at most $m$ halfspaces. Moreover, the classifier is constant on each of these convex regions, predicting label $y_i$ at all points in $P_i$.

For an input $\x$, let $\calI_\x$ be the indices of polyhedra $P_i$ such that $f(\x) \neq y_i$. Then, the region-based attack optimizes over all polyhedra $P_i$ for $i \in \calI_\x$ by solving the inner minimization of Equation \eqref{eqn:opt-untargeted}, namely
\begin{equation}\label{eq:inner-min}
\underset{\z \in P_i}{\text{\ min\ }} \|\x - \z\|_p.
\end{equation}

Given that $P_i$ is a polyhedron, the constraint $\mathbf{z} \in P_i$ can be expressed using the $m$ linear constraints that define $P_i$. Then, the norm minimization can be expressed as a convex objective. In particular, the problem (\ref{eq:inner-min}) can be solved with a linear program for $p \in \{1,\infty\}$ or a
quadratic program for $p = 2$ using standard techniques~\citep{boyd}.
The following are the specific LP formulations for $p \in \{1,\infty\}$.

{\bf $\ell_\infty$ norm.}
Let $t \in \mathbb{R}$ be single variable.
When $p=\infty$, the problem in (\ref{eq:inner-min}) can be solved in $\R^d$ using the following linear program with $d+1$ variables and $m+2d$ linear constraints.
\begin{equation}
\label{eq:linf}
\begin{aligned}
& \underset{\mathbf{z}, t}{\text{minimize}}
& & t \\
& \text{subject to} 
& & \mathbf{z} \in P_i & \\
& & & (\mathbf{z} - \mathbf{x})_j \leq t & \forall\ j \in [d] \\
& & & (\mathbf{z} - \mathbf{x})_j \geq -t & \forall\ j \in [d] \\
\end{aligned}
\end{equation}

{\bf $\ell_1$ norm.}
Let $\mathbf{t} \in \mathbb{R}^d$ be vector.
When $p=1$, the problem in (\ref{eq:inner-min}) can be solved in $\R^d$ using the following linear program with $2d$ variables and $m+2d$ linear constraints.
\begin{equation}
\label{eq:l1}
\begin{aligned}
& \underset{\mathbf{z}, \mathbf{t}}{\text{minimize}}
& & \mathbf{1}^T\mathbf{t} & \\
& \text{subject to} 
& & \mathbf{z} \in P_i & \\
& & & (\mathbf{z} - \mathbf{x})_j \leq \mathbf{t}_j &\forall\ j \in [d] \\
& & & (\mathbf{z} - \mathbf{x})_j \geq -\mathbf{t}_j &\forall\ j \in [d] \\
\end{aligned}
\end{equation}

\subsubsection*{Finishing the Analysis of the Exact Region-Based Attack}

\begin{proof}[Proof of Theorem~\ref{thm:main}]
    We first claim that the attack produces the optimal adversarial example when $f$ is any $(s,m)$-decomposable classifier. By assumption, there is a partition of $\R^d$ into polyhedra $P_1,\ldots, P_s$ such that $f$ is constant on each $P_i$ region. Let~$y_i$ be the label that $f$ gives to all points in $P_i$ for each $i \in [s]$. On input $\x$, the algorithm considers $i \in \calI_\x$, where $\calI_\x \subseteq [s]$ are the indices such that $f(\x) \neq y_i$. Thus, the point $\z^i \in P_i$ closest to $\x$ will have \[f(\z^i) = y_i \neq f(\x).\] Finally, the algorithm's output is
    \[
        \argmin_{\{\z^i \mid  i\in\calI_\x\}} \ \|\z^i - \x\|.
    \]
    As the regions $P_i$ partition $\R^d$, this is the closest point to $\x$ that receives a different label under $f$. 
    
    We now analyze the running time. For the $\ell_p$ distance, $p \in \{1,2,\infty\}$, finding each candidate point $\z^i$ requires solving a convex program with $O(m)$ constraints and $O(d)$ variables. This can be done in $\poly(d,m)$ time using standard optimization techniques (e.g., the interior point method). The number of convex programs is $|\calI_\x| \leq s$. Therefore, the total running time is at most $s \cdot \poly(d,m)$.
\end{proof}

\begin{remark}[Targeted Attack] So far, we have considered untargeted attacks, allowing adversarial examples to have any label other than $f(\x)$. An important variation is a {\em targeted attack}, which specifies a label $\ell \in [C]$, and the goal is to output a close point $\wt \x$ such that $f(\wt \x) = \ell$. We note that the region-based attack can be easily modified for this by only searching over $\calI_\x^\ell = \{i \in [s] \mid y_i = \ell \}$. This may significantly reduce the running time in practice, as $|\calI_\x^\ell|$ may be much smaller than $|\calI_\x|$.
\end{remark}

We specialize the above theorem to ensembles of linear decision trees and the $k$-NN classifier.

\begin{corollary}\label{cor:main}
    Let $n$ be the size of the training set. If $f : \R^d \to \labels$ is a classifier determined by $k$-NN with $k = O(1)$ or an ensemble of $O(1)$ linear decision trees with depth $\poly(n)$ and $\poly(n)$ total leaves, then the region-based attack outputs the optimal adversarial example in time $\poly(d,n)$.
\end{corollary}
\begin{proof}
When $f$ is an ensemble of $T$ linear decision trees, each with depth $D$ and $L$ leaves, Lemma~\ref{lem:tree-decomposition} implies that $f$ is $(L^T,TD)$-decomposable. Assuming that $T$ is a constant and $L$ and $D$ are polynomial means that $f$ is $(\poly(n),\poly(n))$-decomposable. Applying Theorem~\ref{thm:main}, the running time of the exact region-based attack is thus $\poly(d,n)$.

When $f$ is the $k$-NN classifier, Lemma~\ref{lem:knn-decomposition} implies that $f$ is $(\binom{n}{k}, k(n-k))$-decomposable. Assuming that $k$ is a constant means that $f$ is $(\poly(n),O(n))$-decomposable. Applying Theorem~\ref{thm:main}, the running time of the exact region-based attack is thus $\poly(d,n)$.
\end{proof}

\section{More Experimental Details}
\label{app:exp}

The experiment is run on desktop with Intel - Core i7-9700K 3.6 GHz 8-Core Processor and 32 GB of RAM.

\subsection{Classifier Implementation Details}

The implementation for DT, RF and $k$-NN
are based on \texttt{scikit-learn} \citep{scikit-learn}.
For DT and RF, the splitting criterion is set to ``entropy''.
For computational efficiency, we fixed the maximum depth of DT and RF to be five.
For reproducibility, all other hyper-parameters are set to the default
parameter settings of the specific implementation.

\subsection{Attack and Defense Implementation Details}

For kernel substitution attack, we set the approximation parameter $c=1.0$ and attack
the substitution model with Projected Gradient Descent (PGD)~\citep{madry2017}.
For both Region-Based Attacks (RBA-Exact and RBA-Approx), the underlying LP solver that we use is Gurobi~\citep{gurobi}.
For kernel substitution attack, we use PGD
implemented in \texttt{Cleverhans} \citep{papernot2017cleverhans}.
The implementation of the black-box attack by \citep{cheng2019query}~(BBox) is provided by authors in their public
repository.\footnote{\url{https://github.com/cmhcbb/attackbox}}

For $k$-NN, we do not compare with the gradient-based
extension~\citep{sitawarin2019robustness} attack directly in
Section \ref{sec:experiments} since it is under a different setting.
Their algorithm only works if $k$-NN uses the cosine distance instead of
$\ell_2$ distance.

\subsection{Dataset Details}
\label{app:data_detail}

For each dataset, we reserve 200 examples for testing.
We evaluate the testing accuracy on these 200 examples.
To compute empirical robustness $\emprob(A, f_D, S, t)$ 
and $\mbox{\defenderscore}(D, A, f, S, t)$,
we randomly select 100 correctly predicted examples for each classifier.
For efficiency purposes, the feature dimension for fashion-mnist (f-mnist),
mnist is reduced to 25 using principle component analysis (PCA).
The original covtype is sub-sampled to 2200 examples.
mnist17 represents a subset of mnist dataset for the binary classification
problem distinguishing between 1 and 7.
Similarly, f-mnist35 is the task of distinguishing between 3rd and 5th class, and
f-mnist06 is the task of distinguishing between 0th and 6th class. 
The features are scaled to $[0, 1]$ so the solver will avoid numerical
rounding errors.

\begin{table}[h!]
  \tiny
  \centering
  \setlength{\tabcolsep}{2.0pt}
  \tiny
\centering
\setlength{\tabcolsep}{.5pt}
\begin{tabular}{lccccc}
\toprule
{} &  \parbox{10mm}{\centering \# train} &  \parbox{13mm}{\centering \# test \\ (perturb.)} &  \parbox{13mm}{\centering \# test \\ (accuracy)} &  \parbox{10mm}{features} &  classes \\
\midrule
austr.     &                                                490 &                                                100 &                                                200 &           14 &           2 \\
cancer     &                                                483 &                                                100 &                                                200 &           10 &           2 \\
covtype    &                                               2000 &                                                100 &                                                200 &           54 &           2 \\
diabetes   &                                                568 &                                                100 &                                                200 &            8 &           2 \\
f-mnist35  &                                              12000 &                                                100 &                                                200 &           25 &           2 \\
f-mnist06  &                                              12000 &                                                100 &                                                200 &           25 &           2 \\
fourclass  &                                                662 &                                                100 &                                                200 &            2 &           2 \\
halfmoon   &                                               2000 &                                                100 &                                                200 &            2 &           2 \\
mnist17    &                                              13007 &                                                100 &                                                200 &           25 &           2 \\
\bottomrule
\end{tabular}

  \caption{Dataset statistics.}
  \label{table:dataset-stats}
\end{table}

\subsection{Additional Experiment Results}\label{app:defense}

Table \ref{table:1nn_robustness}, \ref{table:3nn_robustness}, \ref{table:dt-robustness},
\ref{table:rf-robustness} show
additional experiment results with adversarial pruning~(AP) as defense.
In these tables, for AP with separation parameter $r=0.5$, we have some invalid values.
These values are caused by setting a too large value of $r$ which results int that
the adversarial pruned datasets to be highly unbalanced in label or even
making the dataset have a single label left.
If the training accuracy goes below $0.5$ or the prediction of the
classifier outputs only one label, we will put the value being "-" in the
table.
For diabetes with $3$-NN, its caused by $3$-NN only predicts one label.

Testing accuracy is a sanity check that we are not giving away all accuracy for robustness.
The higher the empirical robustness is means the classifier is more robust to the given attack.
When considering the strength of the attack, empirical robustness is lower the better.
When considering the strength of the defense, \defenderscore is higher the better.
For \defenderscore higher mean that after defense (AP), the classifier become more robust, thus higher the better.

\begin{table*}[h!]
\tiny
\centering
\setlength{\tabcolsep}{2.50pt}
\begin{tabular}{lccc|cccc|cccc|cccc}
\toprule
        & \multicolumn{3}{c}{1-NN} & \multicolumn{4}{c}{AP (separation parameter $r$=.1)} & \multicolumn{4}{c}{AP (separation parameter $r$=.3)} & \multicolumn{4}{c}{AP (separation parameter $r$=.5)} \\
        & \parbox{9mm}{\centering ER} & \parbox{8mm}{\centering test \\ accuracy} & \# train &      \parbox{9mm}{\centering ER} & \parbox{8mm}{\centering test \\ accuracy} & \# train & \defenderscore &      \parbox{9mm}{\centering ER} & \parbox{8mm}{\centering test \\ accuracy} & \# train & \defenderscore &      \parbox{9mm}{\centering ER} & \parbox{8mm}{\centering test \\ accuracy} & \# train & \defenderscore \\
\midrule
austr. &                      $.151$ &                                    $.805$ &    $490$ &                           $.162$ &                                    $.800$ &    $484$ &        $1.073$ &                           $.249$ &                                    $.820$ &    $458$ &        $1.649$ &                           $.311$ &                                    $.825$ &    $427$ &        $2.060$ \\
cancer &                      $.137$ &                                    $.950$ &    $483$ &                           $.137$ &                                    $.950$ &    $483$ &        $1.000$ &                           $.193$ &                                    $.950$ &    $473$ &        $1.409$ &                           $.261$ &                                    $.965$ &    $458$ &        $1.905$ \\
covtype &                      $.066$ &                                    $.725$ &   $2000$ &                           $.072$ &                                    $.700$ &   $1904$ &        $1.091$ &                           $.289$ &                                    $.685$ &   $1417$ &        $4.379$ &                           $.346$ &                                    $.675$ &   $1384$ &        $5.242$ \\
diabetes &                      $.035$ &                                    $.695$ &    $568$ &                           $.035$ &                                    $.700$ &    $535$ &        $1.000$ &                           $.164$ &                                    $.660$ &    $379$ &        $4.686$ &                           $.375$ &                                    $.660$ &    $370$ &       $10.714$ \\
f-mnist06 &                      $.029$ &                                    $.800$ &  $12000$ &                           $.031$ &                                    $.820$ &  $11509$ &        $1.069$ &                           $.075$ &                                    $.765$ &   $7348$ &        $2.586$ &                           - &                                    $.495$ &   $6000$ &         - \\
f-mnist35 &                      $.075$ &                                   $1.000$ &  $12000$ &                           $.075$ &                                   $1.000$ &  $11999$ &        $1.000$ &                           $.089$ &                                    $.980$ &  $10477$ &        $1.187$ &                           $.104$ &                                    $.945$ &   $8139$ &        $1.387$ \\
fourclass &                      $.090$ &                                   $1.000$ &    $662$ &                           $.107$ &                                    $.960$ &    $559$ &        $1.189$ &                           $.278$ &                                    $.750$ &    $453$ &        $3.089$ &                            - &                                    $.565$ &    $442$ &          - \\
halfmoon &                      $.058$ &                                    $.920$ &   $2000$ &                           $.151$ &                                    $.915$ &   $1702$ &        $2.603$ &                           $.161$ &                                    $.840$ &   $1144$ &        $2.776$ &                           - &                                    $.480$ &   $1004$ &         - \\
mnist17 &                      $.070$ &                                    $.975$ &  $13007$ &                           $.072$ &                                    $.975$ &  $13004$ &        $1.029$ &                           $.097$ &                                    $.965$ &  $11128$ &        $1.386$ &                           $.118$ &                                    $.810$ &   $6783$ &        $1.686$ \\
\bottomrule
\end{tabular}
\caption{
The number of training data left after adversarial pruning (AP), testing accuracy, empirical robustness,
and \defenderscore with different separation parameter of AP for 1-NN.
}
\label{table:1nn_robustness}
\end{table*}

\begin{table*}[h!]
\tiny
\centering
\setlength{\tabcolsep}{2.50pt}
\begin{tabular}{lccc|cccc|cccc|cccc}
\toprule
        & \multicolumn{3}{c}{3-NN} & \multicolumn{4}{c}{AP (separation parameter $r$=.1)} & \multicolumn{4}{c}{AP (separation parameter $r$=.3)} & \multicolumn{4}{c}{AP (separation parameter $r$=.5)} \\
        & \parbox{9mm}{\centering ER} & \parbox{8mm}{\centering test \\ accuracy} & \# train &      \parbox{9mm}{\centering ER} & \parbox{8mm}{\centering test \\ accuracy} & \# train & \defenderscore &      \parbox{9mm}{\centering ER} & \parbox{8mm}{\centering test \\ accuracy} & \# train & \defenderscore &      \parbox{9mm}{\centering ER} & \parbox{8mm}{\centering test \\ accuracy} & \# train & \defenderscore \\
\midrule
austr. &                      $.278$ &                                    $.805$ &    $490$ &                           $.317$ &                                    $.810$ &    $484$ &        $1.140$ &                           $.333$ &                                    $.815$ &    $458$ &        $1.198$ &                           $.371$ &                                    $.825$ &    $427$ &        $1.335$ \\
cancer &                      $.204$ &                                    $.975$ &    $483$ &                           $.204$ &                                    $.975$ &    $483$ &        $1.000$ &                           $.283$ &                                    $.960$ &    $473$ &        $1.387$ &                           $.350$ &                                    $.970$ &    $458$ &        $1.716$ \\
covtype &                      $.108$ &                                    $.750$ &   $2000$ &                           $.117$ &                                    $.735$ &   $1904$ &        $1.083$ &                           $.357$ &                                    $.685$ &   $1417$ &        $3.306$ &                           $.394$ &                                    $.680$ &   $1384$ &        $3.648$ \\
diabetes &                      $.078$ &                                    $.755$ &    $568$ &                           $.078$ &                                    $.750$ &    $535$ &        $1.000$ &                           $.232$ &                                    $.655$ &    $379$ &        $2.974$ &                            - &                                    $.660$ &    $370$ &          - \\
f-mnist06 &                      $.051$ &                                    $.795$ &  $12000$ &                           $.050$ &                                    $.825$ &  $11509$ &         $.980$ &                           $.089$ &                                    $.750$ &   $7348$ &        $1.745$ &                           - &                                    $.495$ &   $6000$ &         - \\
f-mnist35 &                      $.094$ &                                   $1.000$ &  $12000$ &                           $.093$ &                                   $1.000$ &  $11999$ &         $.989$ &                           $.108$ &                                    $.985$ &  $10477$ &        $1.149$ &                           $.121$ &                                    $.950$ &   $8139$ &        $1.287$ \\
fourclass &                      $.096$ &                                    $.995$ &    $662$ &                           $.127$ &                                    $.960$ &    $559$ &        $1.323$ &                           $.297$ &                                    $.750$ &    $453$ &        $3.094$ &                            - &                                    $.565$ &    $442$ &          - \\
halfmoon &                      $.096$ &                                    $.940$ &   $2000$ &                           $.159$ &                                    $.920$ &   $1702$ &        $1.656$ &                           $.184$ &                                    $.845$ &   $1144$ &        $1.917$ &                            - &                                    $.480$ &   $1004$ &          - \\
mnist17 &                      $.097$ &                                    $.985$ &  $13007$ &                           $.094$ &                                    $.985$ &  $13004$ &         $.969$ &                           $.110$ &                                    $.960$ &  $11128$ &        $1.134$ &                           $.141$ &                                    $.795$ &   $6783$ &        $1.454$ \\
\bottomrule
\end{tabular}
\caption{
The number of training data left after adversarial pruning (AP), testing accuracy, empirical robustness,
and \defenderscore with different separation parameter of AP for 3-NN.
}
\label{table:3nn_robustness}
\end{table*}

\begin{table*}[h!]
\tiny
\centering
\setlength{\tabcolsep}{2.50pt}
\begin{tabular}{lccc|cccc|cccc|cccc}
\toprule
        & \multicolumn{3}{c}{DT} & \multicolumn{4}{c}{AP (separation parameter $r$=.1)} & \multicolumn{4}{c}{AP (separation parameter $r$=.3)} & \multicolumn{4}{c}{AP (separation parameter $r$=.5)} \\
        & \parbox{9mm}{\centering ER} & \parbox{8mm}{\centering test \\ accuracy} & \# train &      \parbox{9mm}{\centering ER} & \parbox{8mm}{\centering test \\ accuracy} & \# train & \defenderscore &      \parbox{9mm}{\centering ER} & \parbox{8mm}{\centering test \\ accuracy} & \# train & \defenderscore &      \parbox{9mm}{\centering ER} & \parbox{8mm}{\centering test \\ accuracy} & \# train & \defenderscore \\
\midrule
austr. &                      $.070$ &                                    $.855$ &    $490$ &                           $.194$ &                                    $.835$ &    $484$ &        $2.771$ &                           $.166$ &                                    $.835$ &    $458$ &        $2.371$ &                           $.450$ &                                    $.835$ &    $427$ &        $6.429$ \\
cancer &                      $.255$ &                                    $.930$ &    $483$ &                           $.255$ &                                    $.930$ &    $483$ &        $1.000$ &                           $.303$ &                                    $.965$ &    $473$ &        $1.188$ &                           $.358$ &                                    $.960$ &    $458$ &        $1.404$ \\
covtype &                      $.051$ &                                    $.715$ &   $2000$ &                           $.051$ &                                    $.740$ &   $1904$ &        $1.000$ &                           $.230$ &                                    $.680$ &   $1417$ &        $4.510$ &                           $.221$ &                                    $.665$ &   $1384$ &        $4.333$ \\
diabetes &                      $.085$ &                                    $.715$ &    $568$ &                           $.085$ &                                    $.720$ &    $535$ &        $1.000$ &                           $.189$ &                                    $.670$ &    $379$ &        $2.224$ &                           $.378$ &                                    $.670$ &    $370$ &        $4.447$ \\
f-mnist06 &                      $.079$ &                                    $.805$ &  $12000$ &                           $.092$ &                                    $.825$ &  $11509$ &        $1.165$ &                           $.203$ &                                    $.770$ &   $7348$ &        $2.570$ &                           - &                                    $.495$ &   $6000$ &         - \\
f-mnist35 &                      $.115$ &                                    $.995$ &  $12000$ &                           $.110$ &                                    $.995$ &  $11999$ &         $.957$ &                           $.237$ &                                    $.940$ &  $10477$ &        $2.061$ &                           $.281$ &                                    $.925$ &   $8139$ &        $2.443$ \\
fourclass &                      $.137$ &                                    $.900$ &    $662$ &                           $.138$ &                                    $.910$ &    $559$ &        $1.007$ &                           $.416$ &                                    $.680$ &    $453$ &        $3.036$ &                            - &                                    $.565$ &    $442$ &          - \\
halfmoon &                      $.085$ &                                    $.950$ &   $2000$ &                           $.167$ &                                    $.895$ &   $1702$ &        $1.965$ &                           $.219$ &                                    $.670$ &   $1144$ &        $2.576$ &                            - &                                    $.480$ &   $1004$ &          - \\
mnist17 &                      $.123$ &                                    $.975$ &  $13007$ &                           $.126$ &                                    $.970$ &  $13004$ &        $1.024$ &                           $.162$ &                                    $.955$ &  $11128$ &        $1.317$ &                           $.316$ &                                    $.830$ &   $6783$ &        $2.569$ \\
\bottomrule
\end{tabular}
\caption{
The number of training data left after adversarial pruning (AP), testing accuracy, empirical robustness,
and \defenderscore with different separation parameter of AP for DT.
}
\label{table:dt-robustness}
\end{table*}

\begin{table*}[h!]
\tiny
\centering
\setlength{\tabcolsep}{2.50pt}
\begin{tabular}{lccc|cccc|cccc|cccc}
\toprule
        & \multicolumn{3}{c}{RF} & \multicolumn{4}{c}{AP (separation parameter $r$=.1)} & \multicolumn{4}{c}{AP (separation parameter $r$=.3)} & \multicolumn{4}{c}{AP (separation parameter $r$=.5)} \\
        & \parbox{9mm}{\centering ER} & \parbox{8mm}{\centering test \\ accuracy} & \# train &      \parbox{9mm}{\centering ER} & \parbox{8mm}{\centering test \\ accuracy} & \# train & \defenderscore &      \parbox{9mm}{\centering ER} & \parbox{8mm}{\centering test \\ accuracy} & \# train & \defenderscore &      \parbox{9mm}{\centering ER} & \parbox{8mm}{\centering test \\ accuracy} & \# train & \defenderscore \\
\midrule
austr. &                      $.446$ &                                    $.845$ &    $490$ &                           $.426$ &                                    $.855$ &    $484$ &         $.955$ &                           $.465$ &                                    $.840$ &    $458$ &        $1.043$ &                           $.496$ &                                    $.835$ &    $427$ &        $1.112$ \\
cancer &                      $.383$ &                                    $.970$ &    $483$ &                           $.383$ &                                    $.970$ &    $483$ &        $1.000$ &                           $.481$ &                                    $.965$ &    $473$ &        $1.256$ &                           $.496$ &                                    $.955$ &    $458$ &        $1.295$ \\
covtype &                      $.214$ &                                    $.750$ &   $2000$ &                           $.226$ &                                    $.700$ &   $1904$ &        $1.056$ &                           $.456$ &                                    $.680$ &   $1417$ &        $2.131$ &                           $.481$ &                                    $.695$ &   $1384$ &        $2.248$ \\
diabetes &                      $.184$ &                                    $.755$ &    $568$ &                           $.175$ &                                    $.740$ &    $535$ &         $.951$ &                           $.409$ &                                    $.660$ &    $379$ &        $2.223$ &                           $.710$ &                                    $.660$ &    $370$ &        $3.859$ \\
f-mnist06 &                      $.188$ &                                    $.790$ &  $12000$ &                           $.215$ &                                    $.785$ &  $11509$ &        $1.144$ &                           $.333$ &                                    $.755$ &   $7348$ &        $1.771$ &                           - &                                    $.495$ &   $6000$ &         - \\
f-mnist35 &                      $.246$ &                                   $1.000$ &  $12000$ &                           $.236$ &                                    $.995$ &  $11999$ &         $.959$ &                           $.346$ &                                    $.925$ &  $10477$ &        $1.407$ &                           $.289$ &                                    $.925$ &   $8139$ &        $1.175$ \\
fourclass &                      $.133$ &                                    $.980$ &    $662$ &                           $.181$ &                                    $.865$ &    $559$ &        $1.361$ &                           $.478$ &                                    $.665$ &    $453$ &        $3.594$ &                            - &                                    $.565$ &    $442$ &          - \\
halfmoon &                      $.149$ &                                    $.930$ &   $2000$ &                           $.198$ &                                    $.900$ &   $1702$ &        $1.329$ &                           $.271$ &                                    $.755$ &   $1144$ &        $1.819$ &                            - &                                    $.480$ &   $1004$ &          - \\
mnist17 &                      $.250$ &                                    $.970$ &  $13007$ &                           $.230$ &                                    $.965$ &  $13004$ &         $.920$ &                           $.314$ &                                    $.945$ &  $11128$ &        $1.256$ &                           $.359$ &                                    $.800$ &   $6783$ &        $1.436$ \\
\bottomrule
\end{tabular}
\caption{
The number of training data left after adversarial pruning (AP), testing accuracy, empirical robustness,
and \defenderscore with different separation parameter of AP for RF.
}
\label{table:rf-robustness}
\end{table*}

\subsubsection{Defense figures}\label{app:defense-figs}

Figures \ref{fig:defense-cmp} and \ref{fig:defense-cmp2} show the complete
experiment results for the experiment in Figure \ref{fig:compare_defenses-main}.
The accuracy (y-axis) is measured on the 100 correctly predicted
testing examples sampled initially.

\begin{figure*}[ht!]
\centering
\subfloat[1-NN]{
    \includegraphics[width=0.24\textwidth]{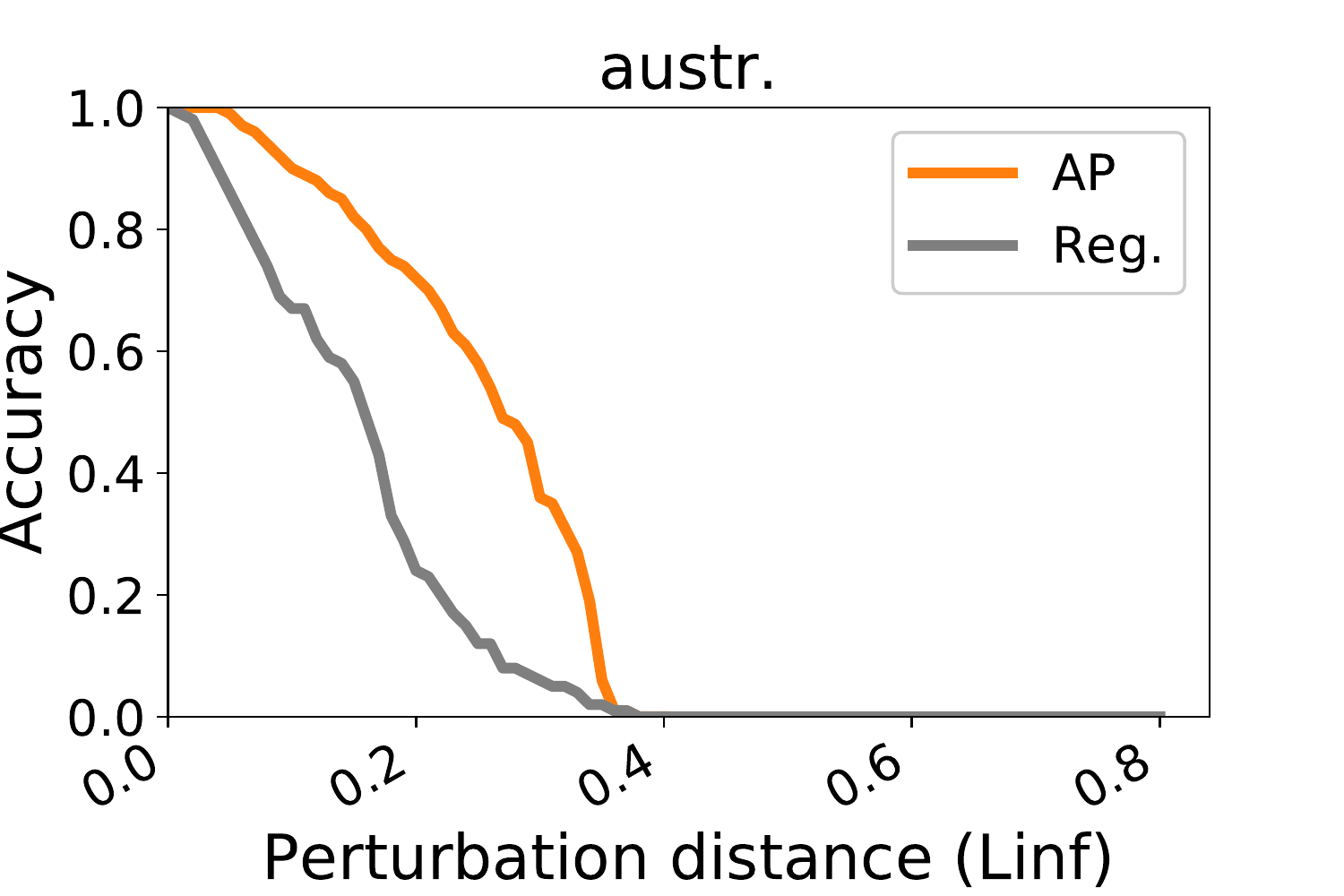}}
\subfloat[3-NN]{
    \includegraphics[width=0.24\textwidth]{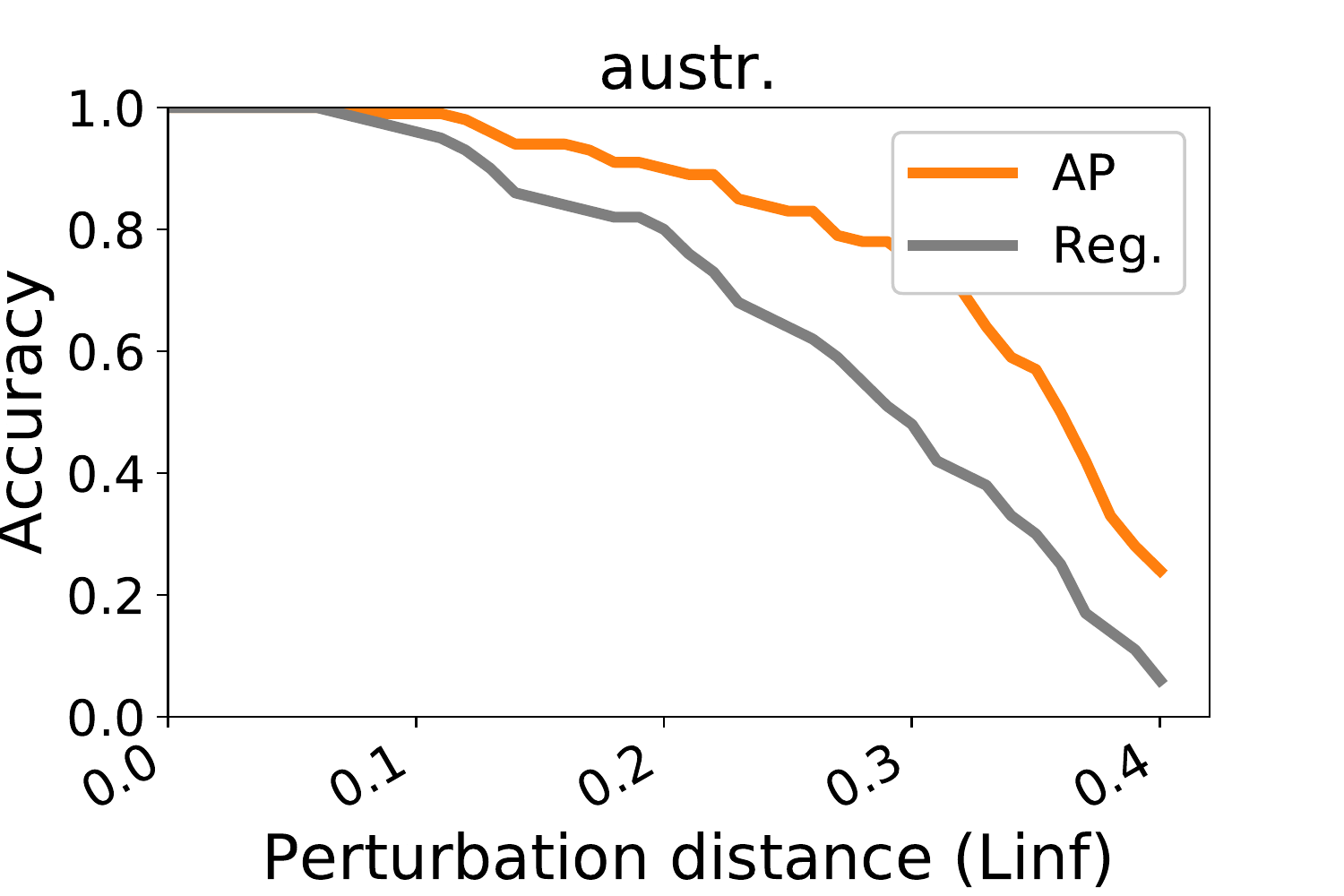}}
\subfloat[Decision tree]{
    \includegraphics[width=0.24\textwidth]{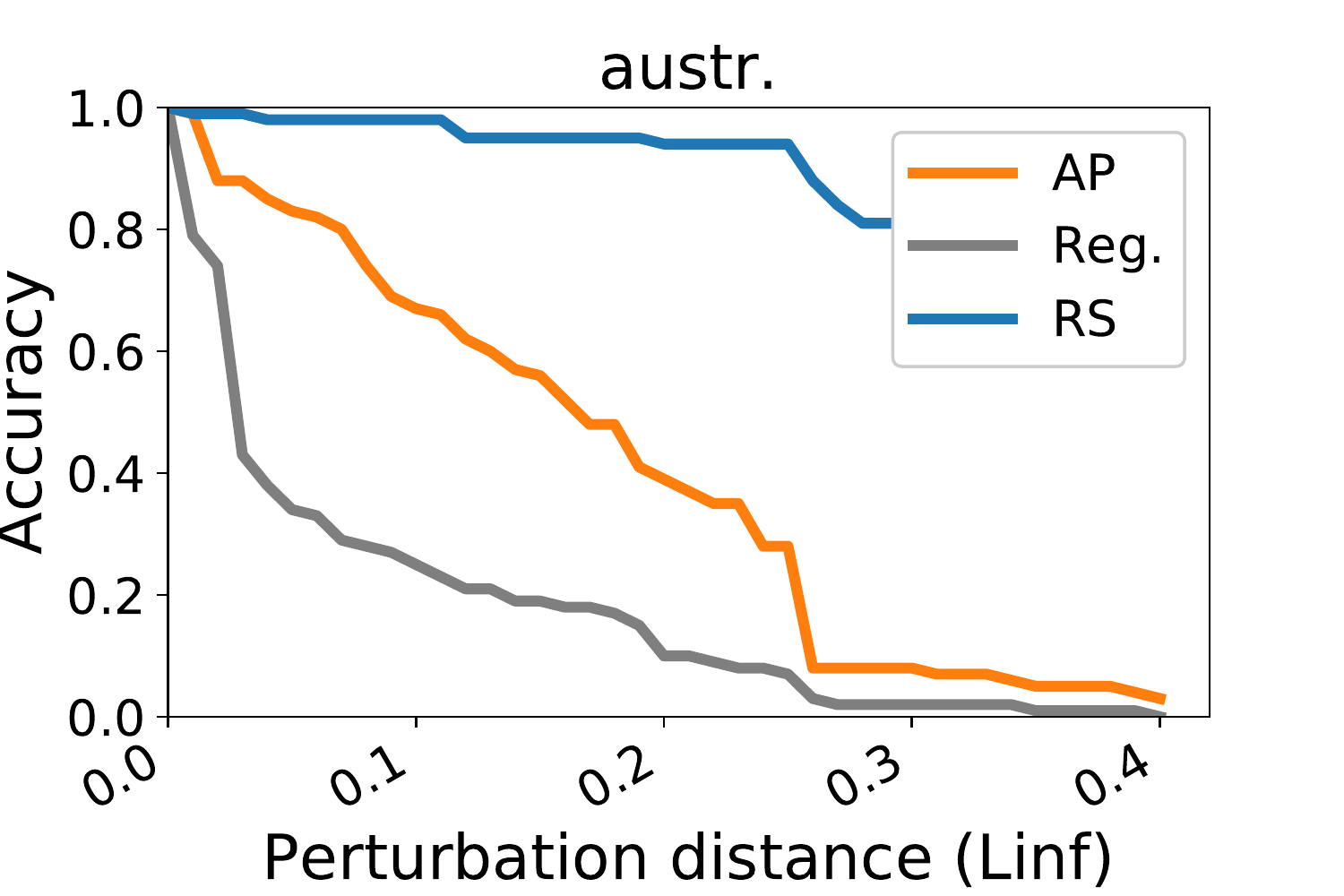}}
\subfloat[Random forest]{
    \includegraphics[width=0.24\textwidth]{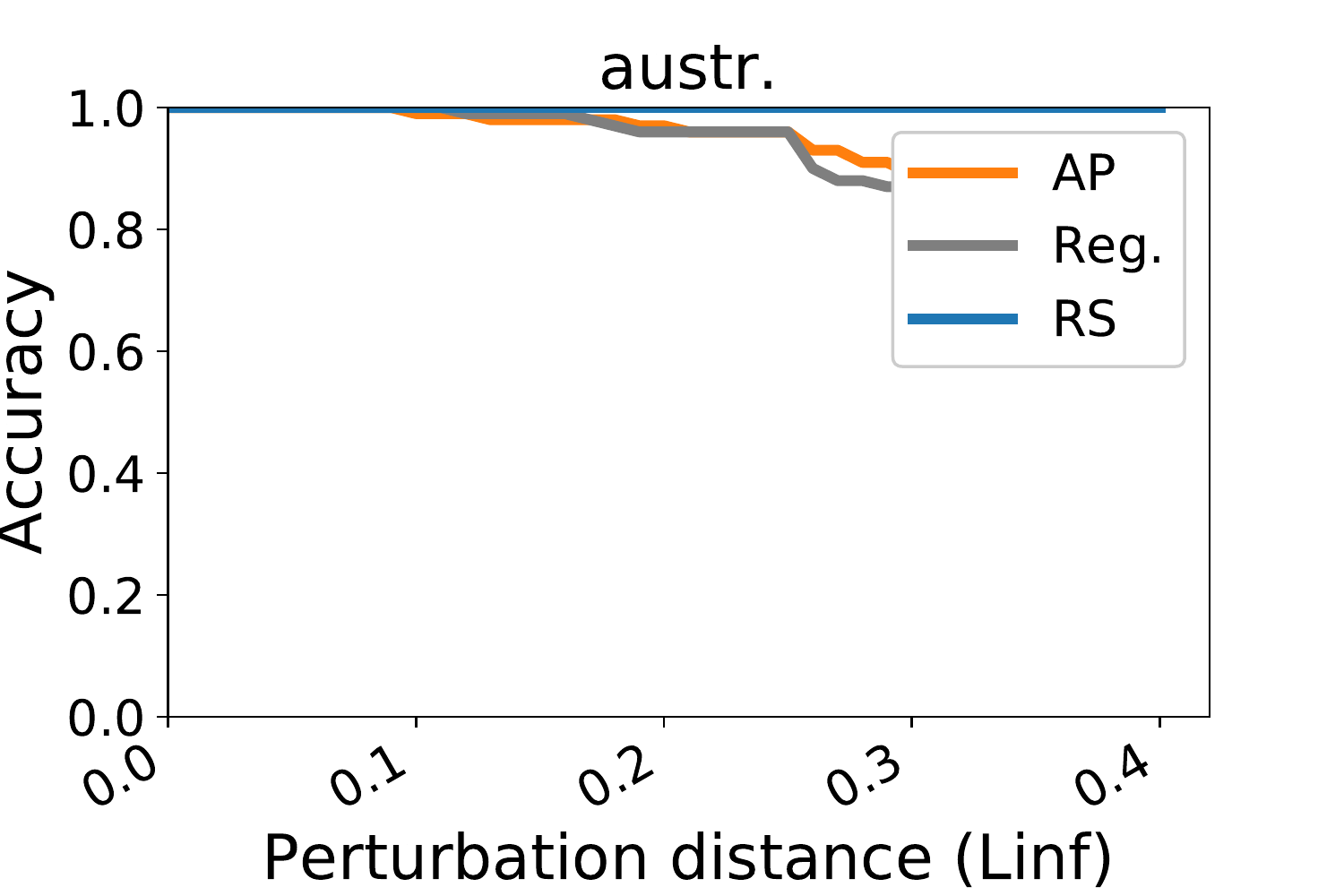}}

\subfloat[1-NN]{
    \includegraphics[width=0.24\textwidth]{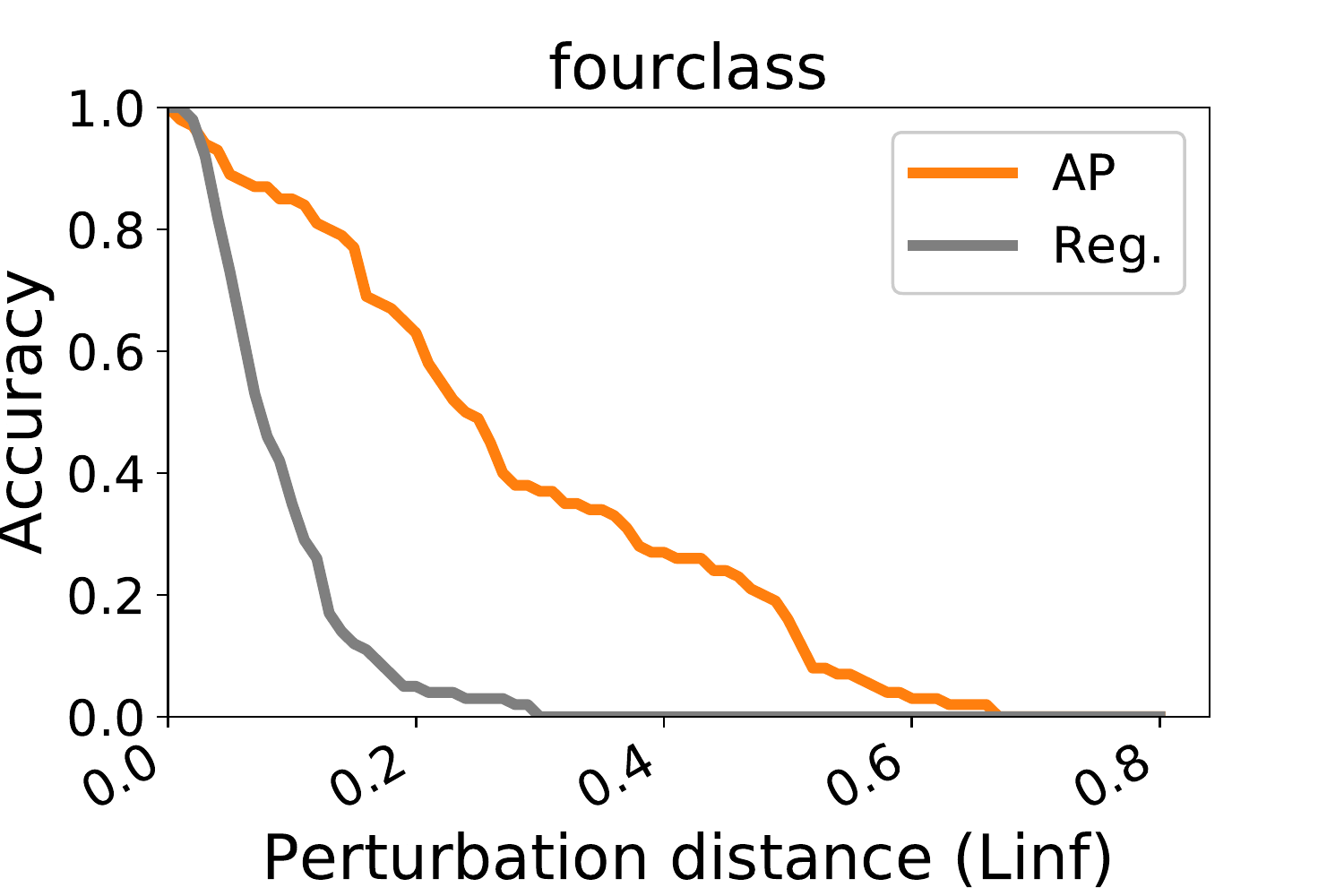}}
\subfloat[3-NN]{
    \includegraphics[width=0.24\textwidth]{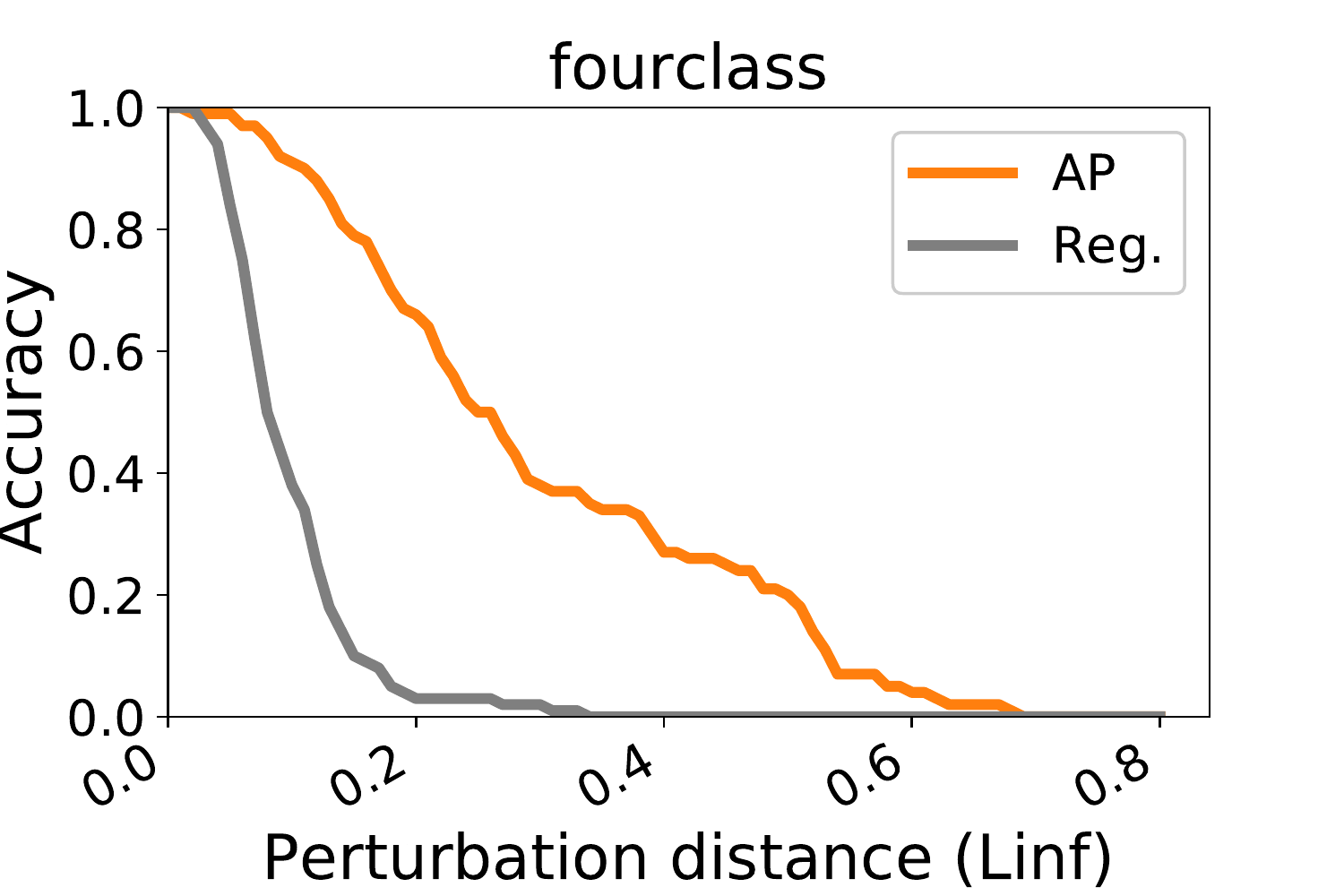}}
\subfloat[Decision tree]{
    \includegraphics[width=0.24\textwidth]{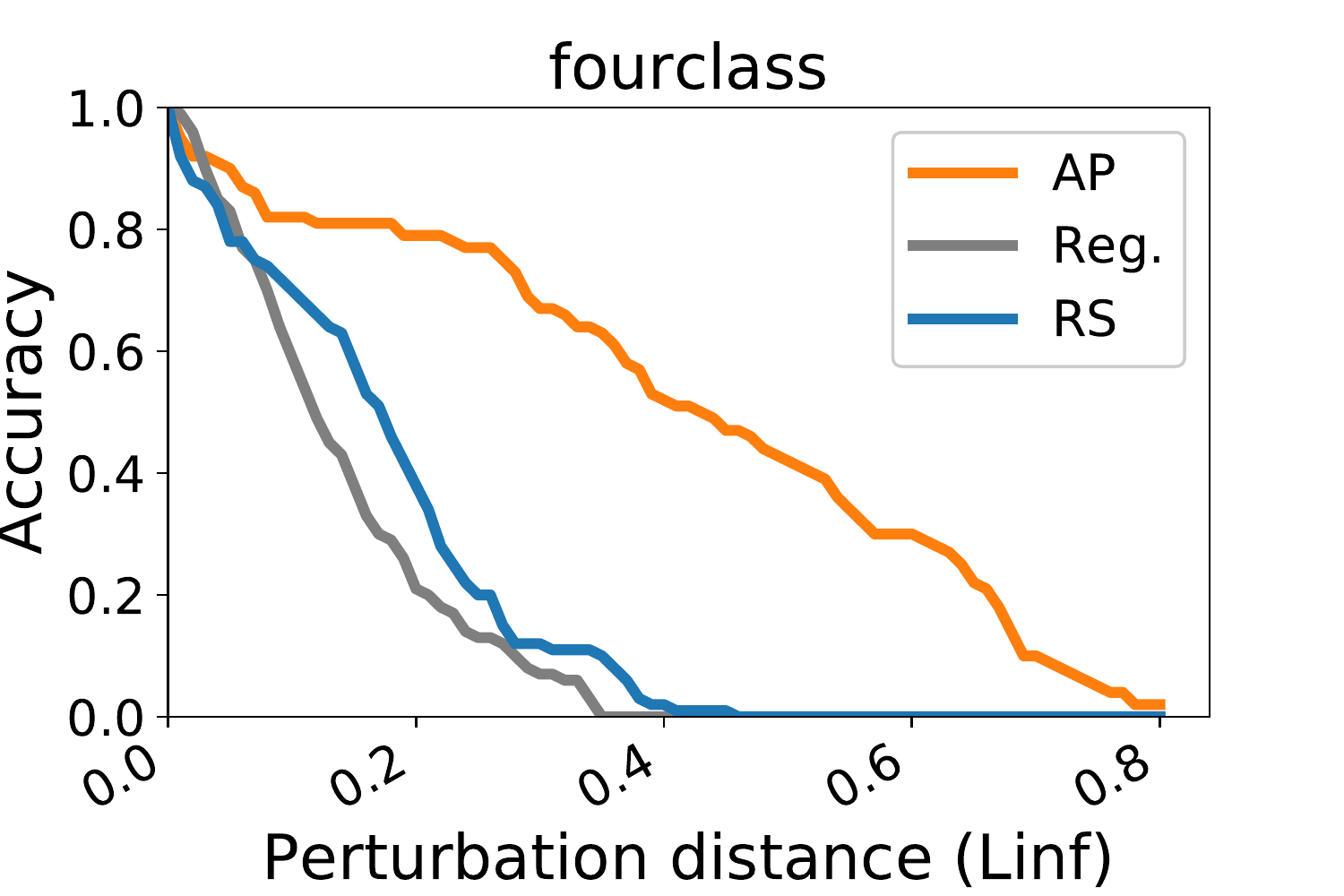}}
\subfloat[Random forest]{
    \includegraphics[width=0.24\textwidth]{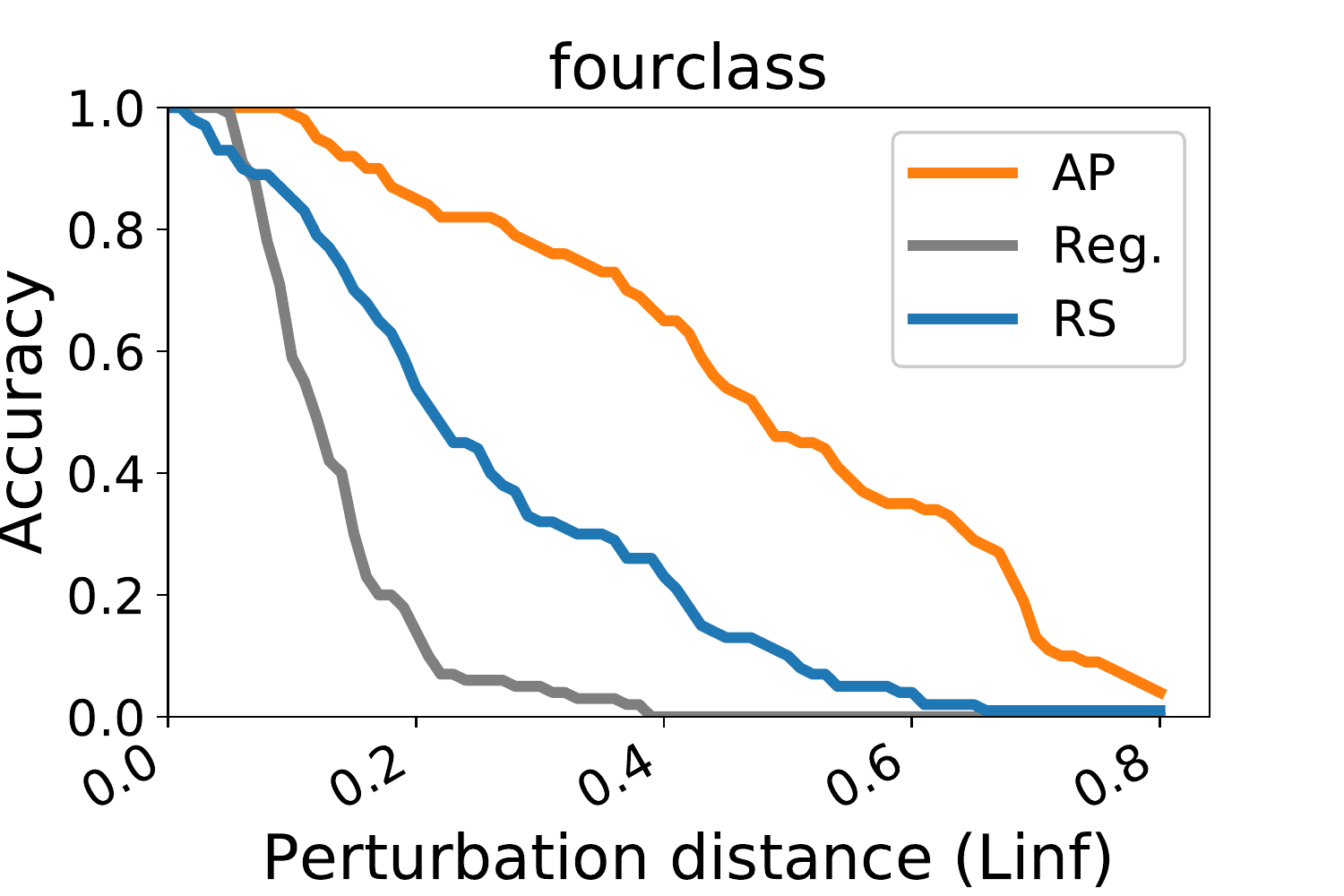}}

\subfloat[1-NN]{
    \includegraphics[width=0.24\textwidth]{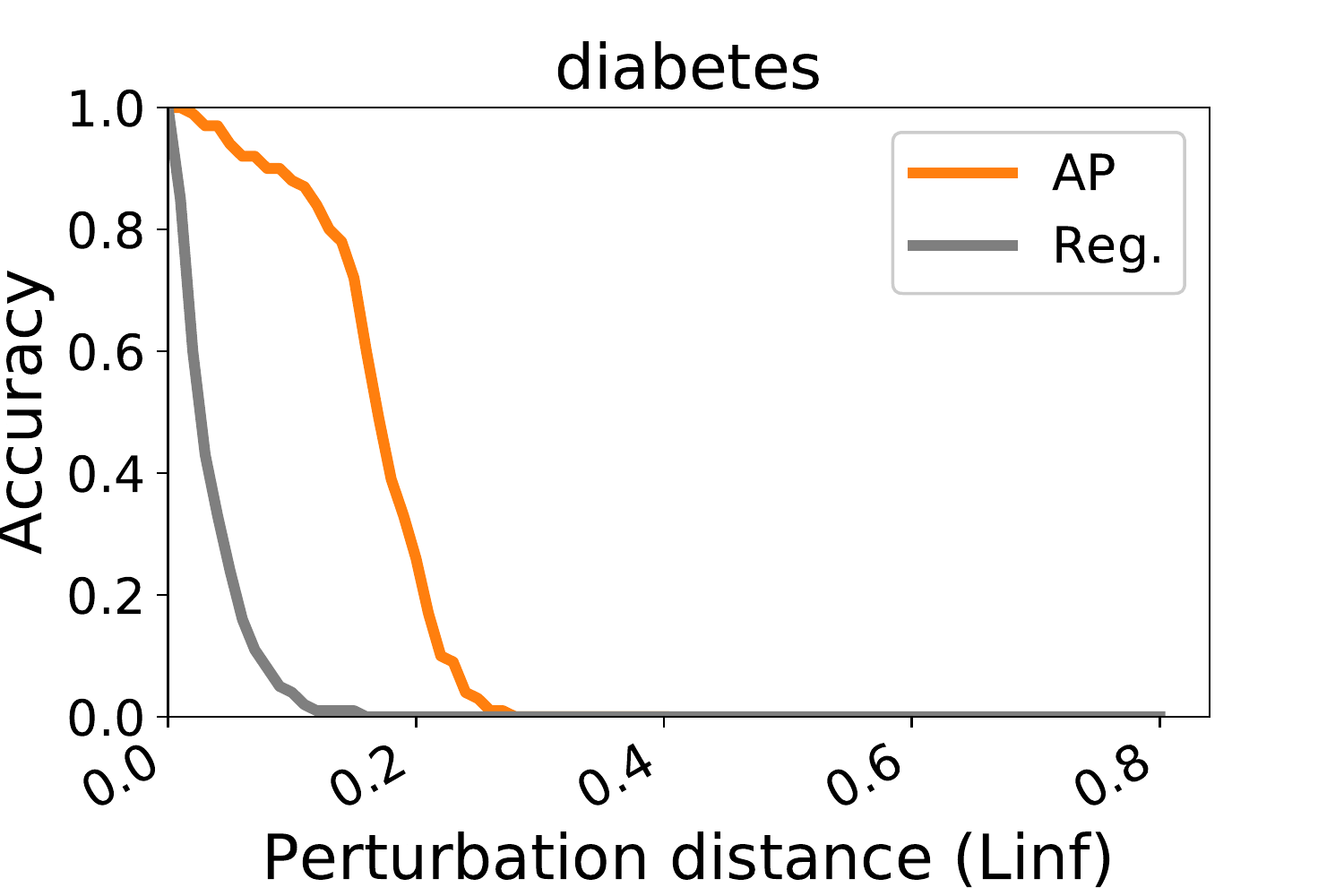}}
\subfloat[3-NN]{
    \includegraphics[width=0.24\textwidth]{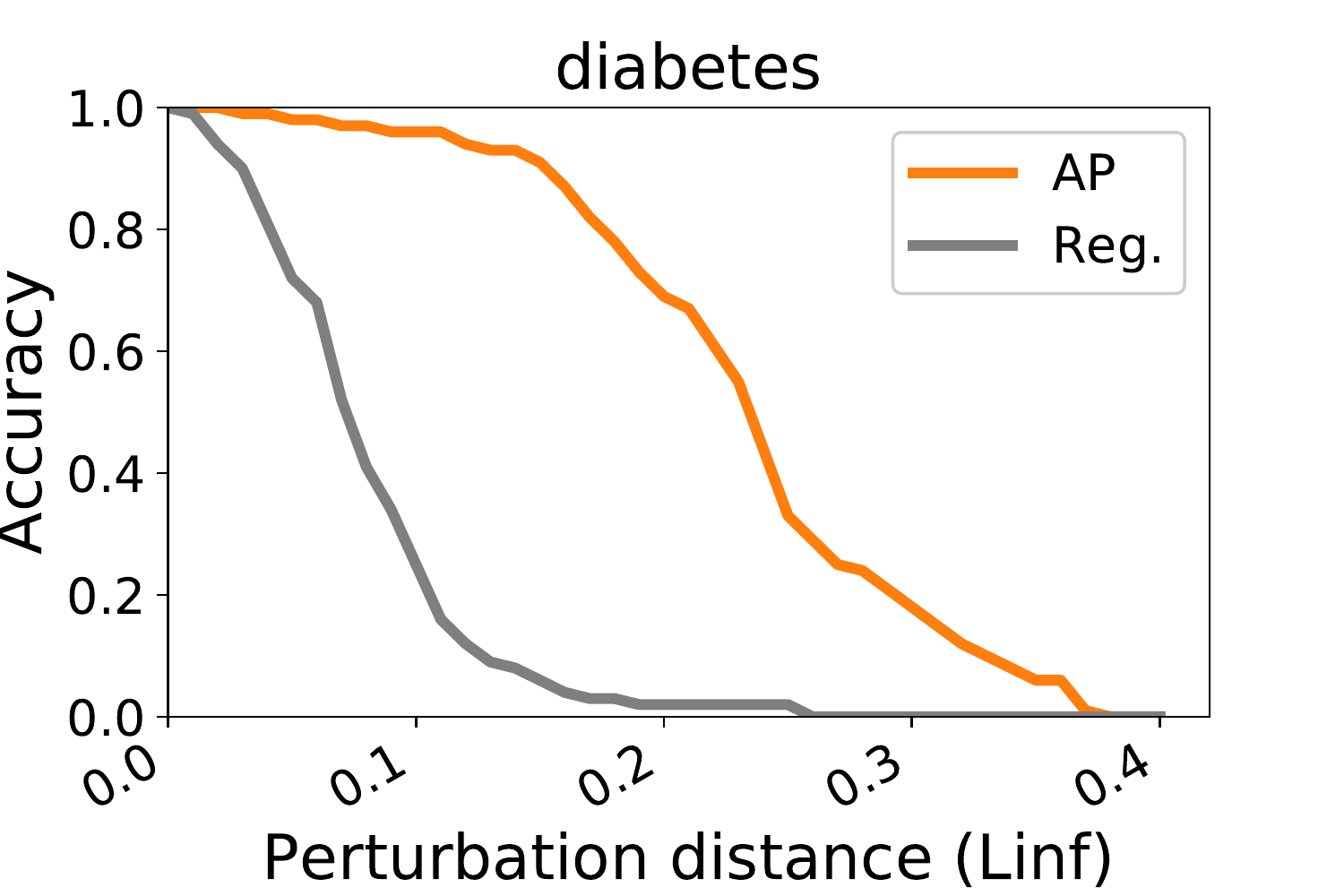}}
\subfloat[Decision tree]{
    \includegraphics[width=0.24\textwidth]{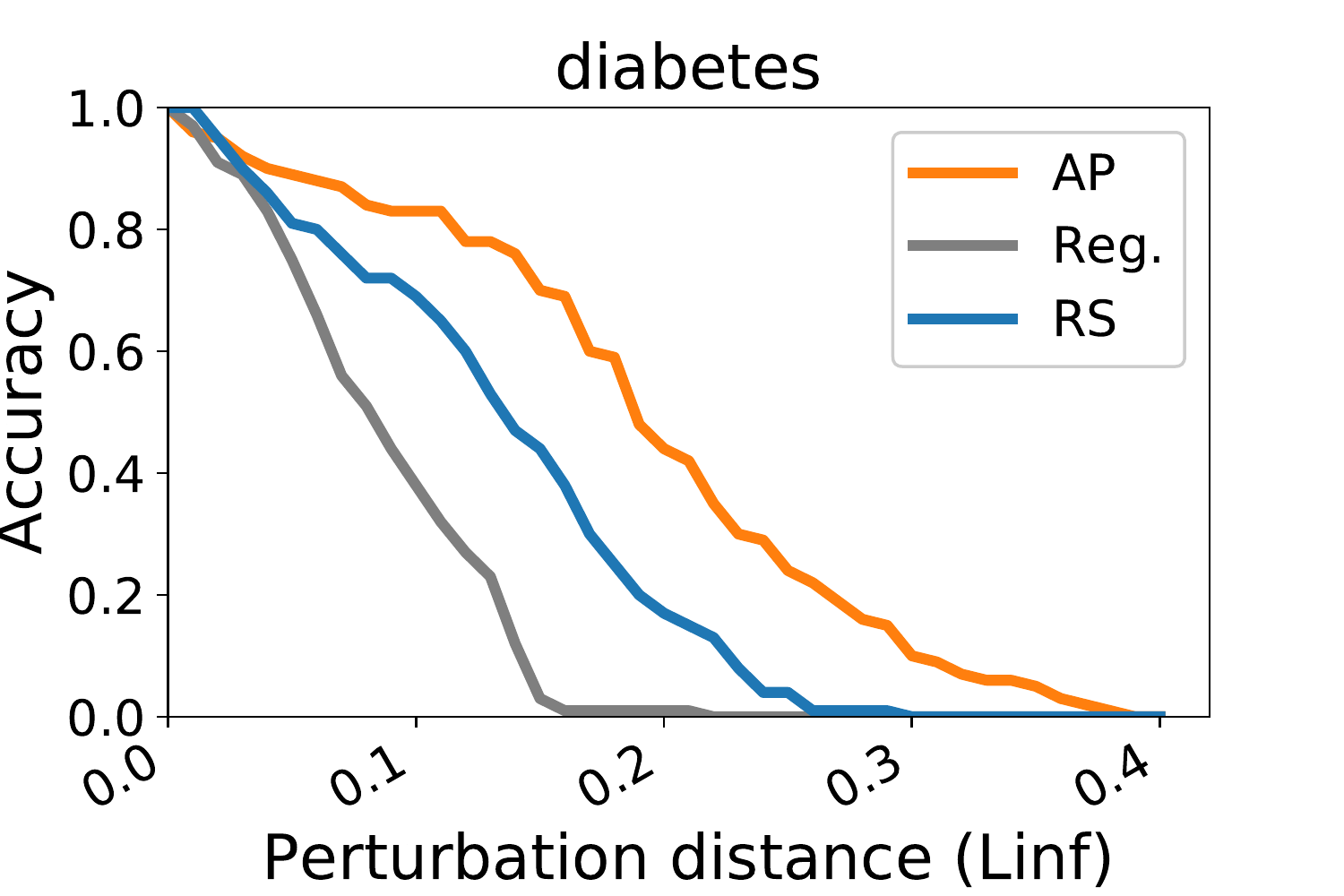}}
\subfloat[Random forest]{
    \includegraphics[width=0.24\textwidth]{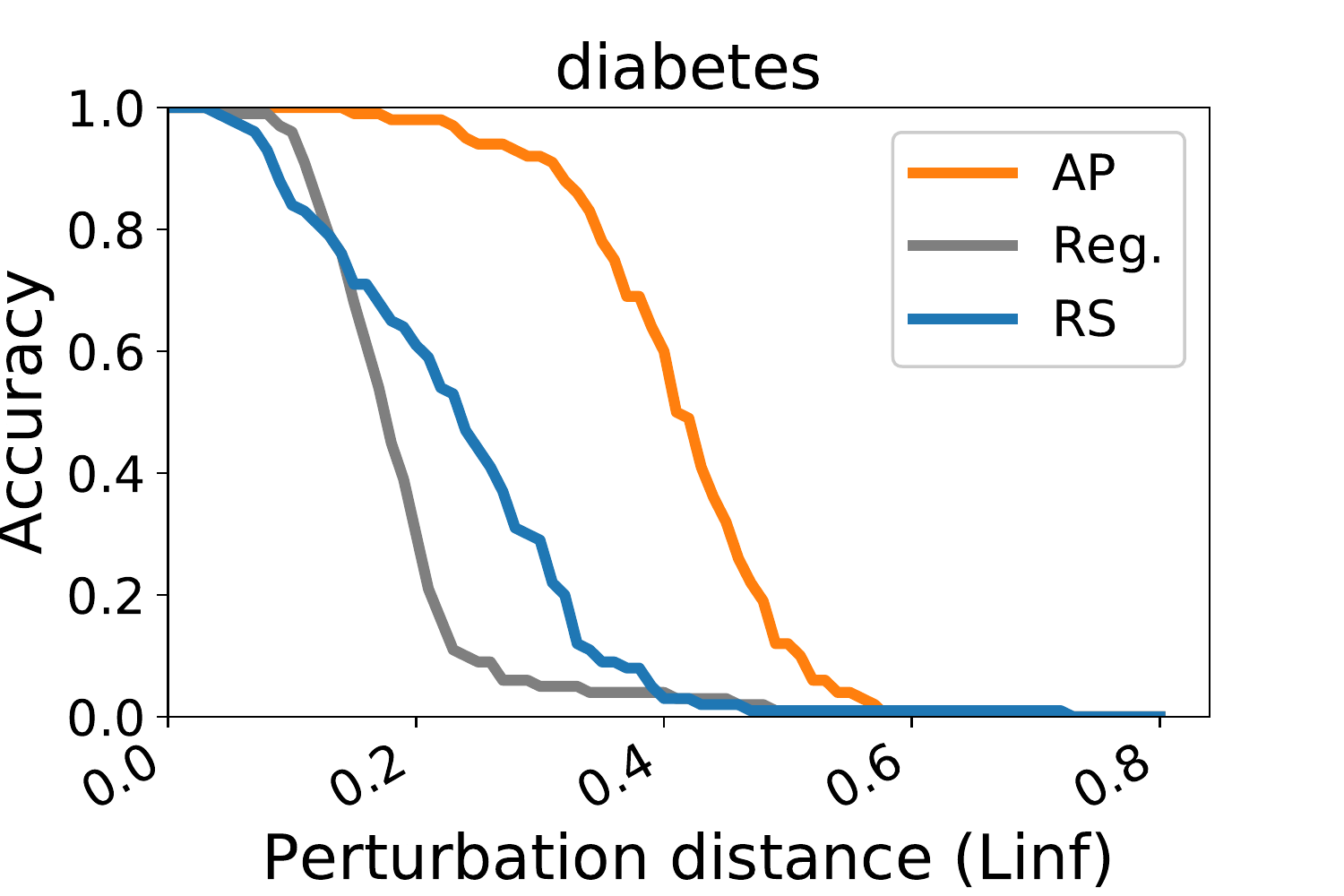}}

\subfloat[1-NN]{
    \includegraphics[width=0.24\textwidth]{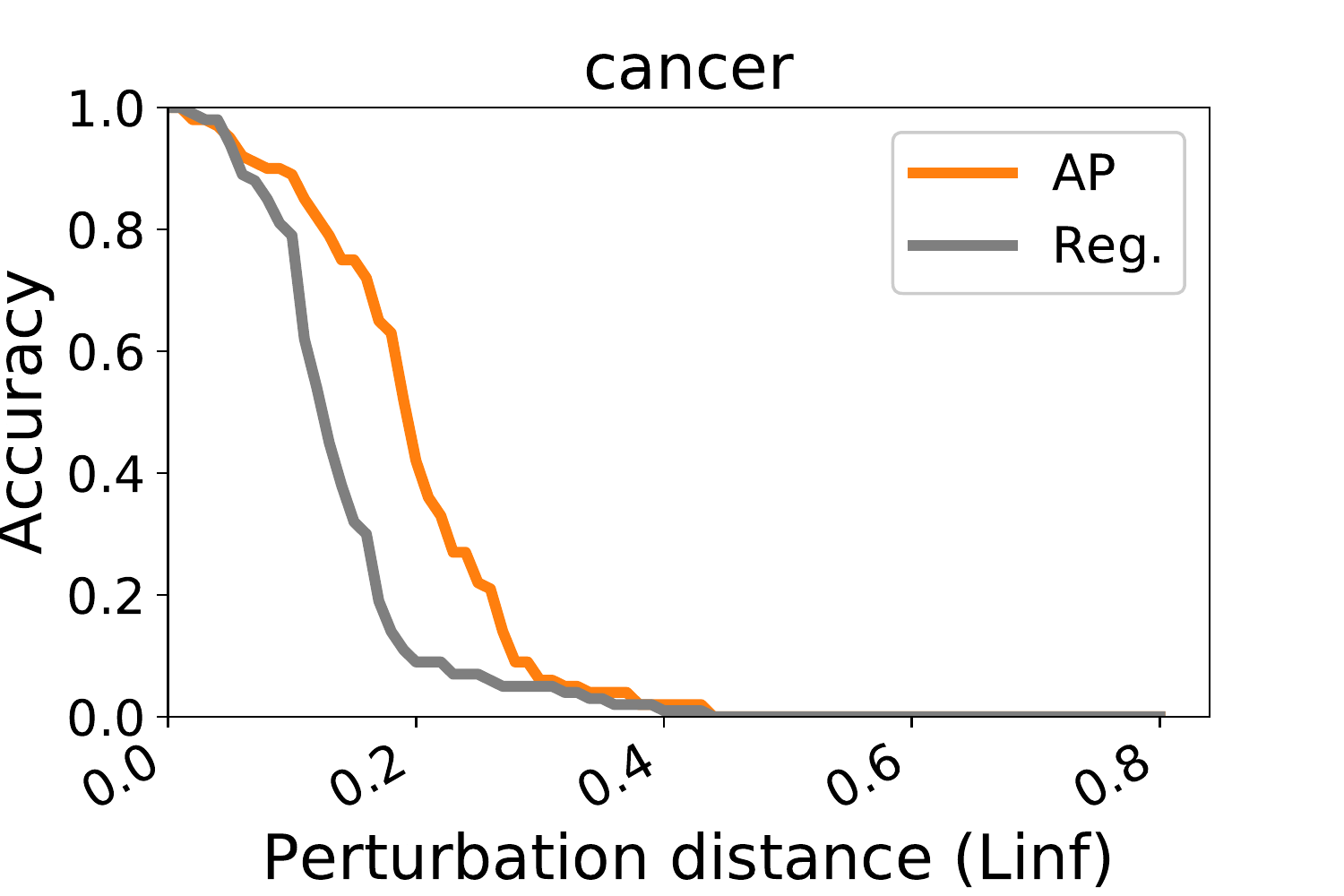}}
\subfloat[3-NN]{
    \includegraphics[width=0.24\textwidth]{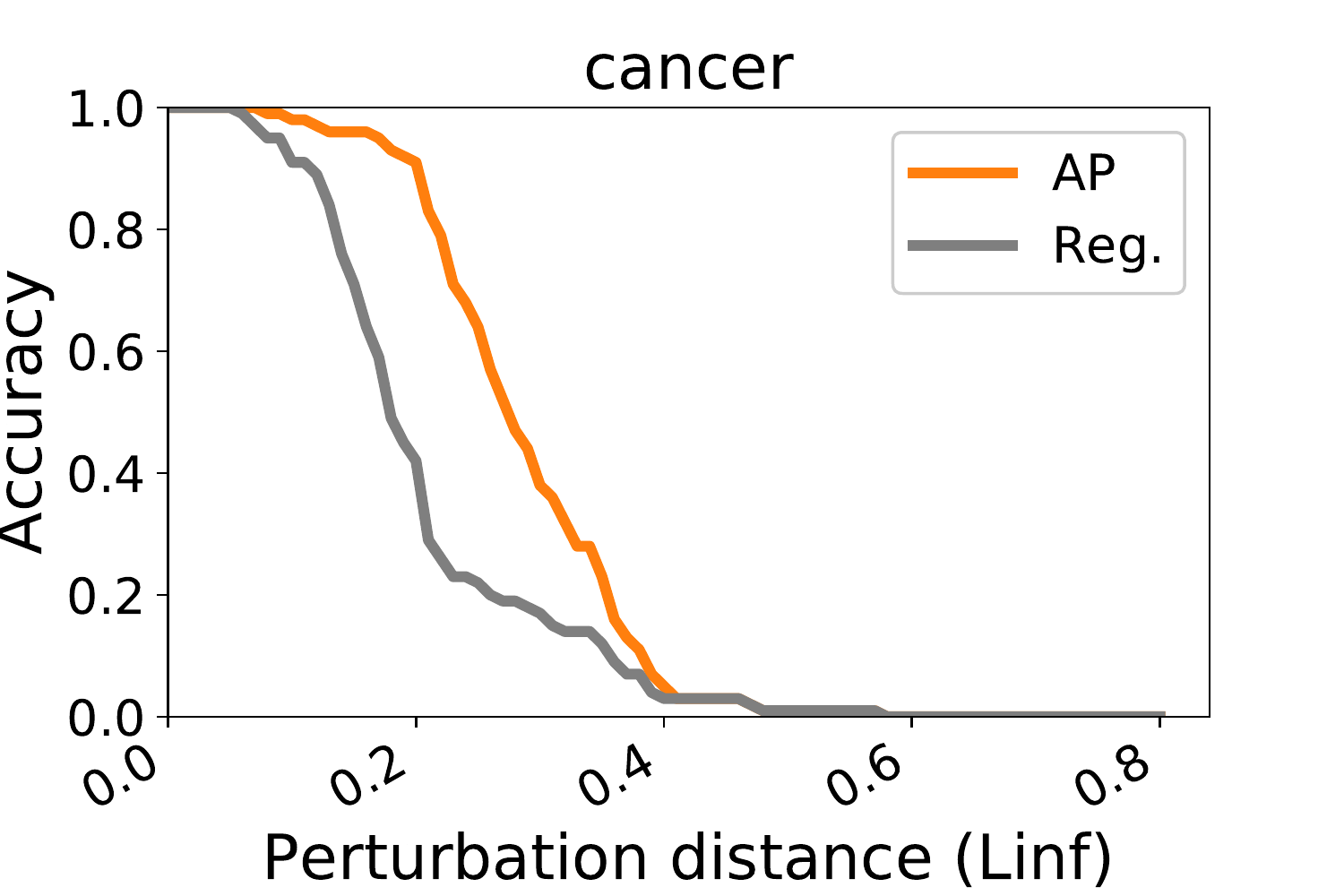}}
\subfloat[Decision tree]{
    \includegraphics[width=0.24\textwidth]{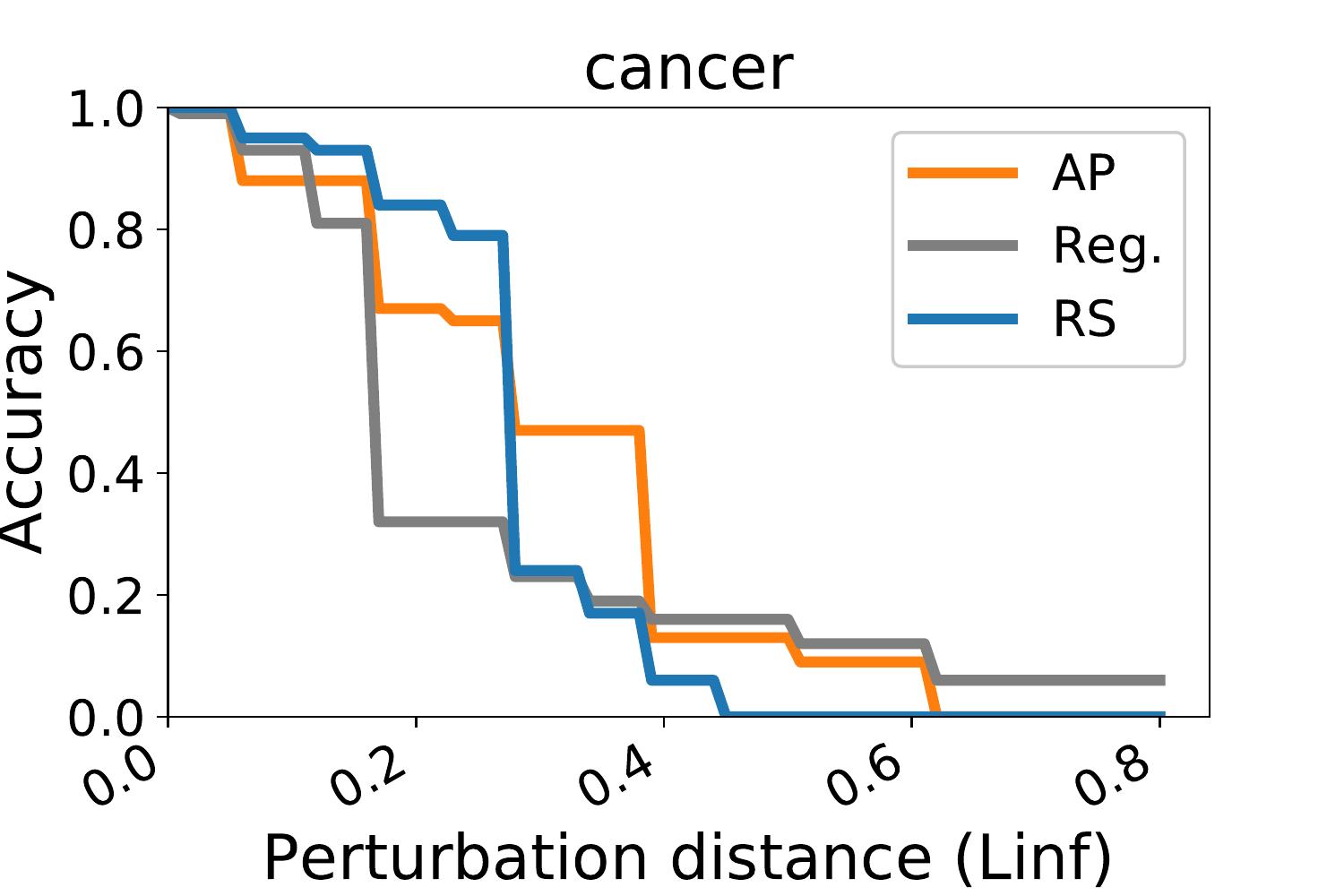}}
\subfloat[Random forest]{
    \includegraphics[width=0.24\textwidth]{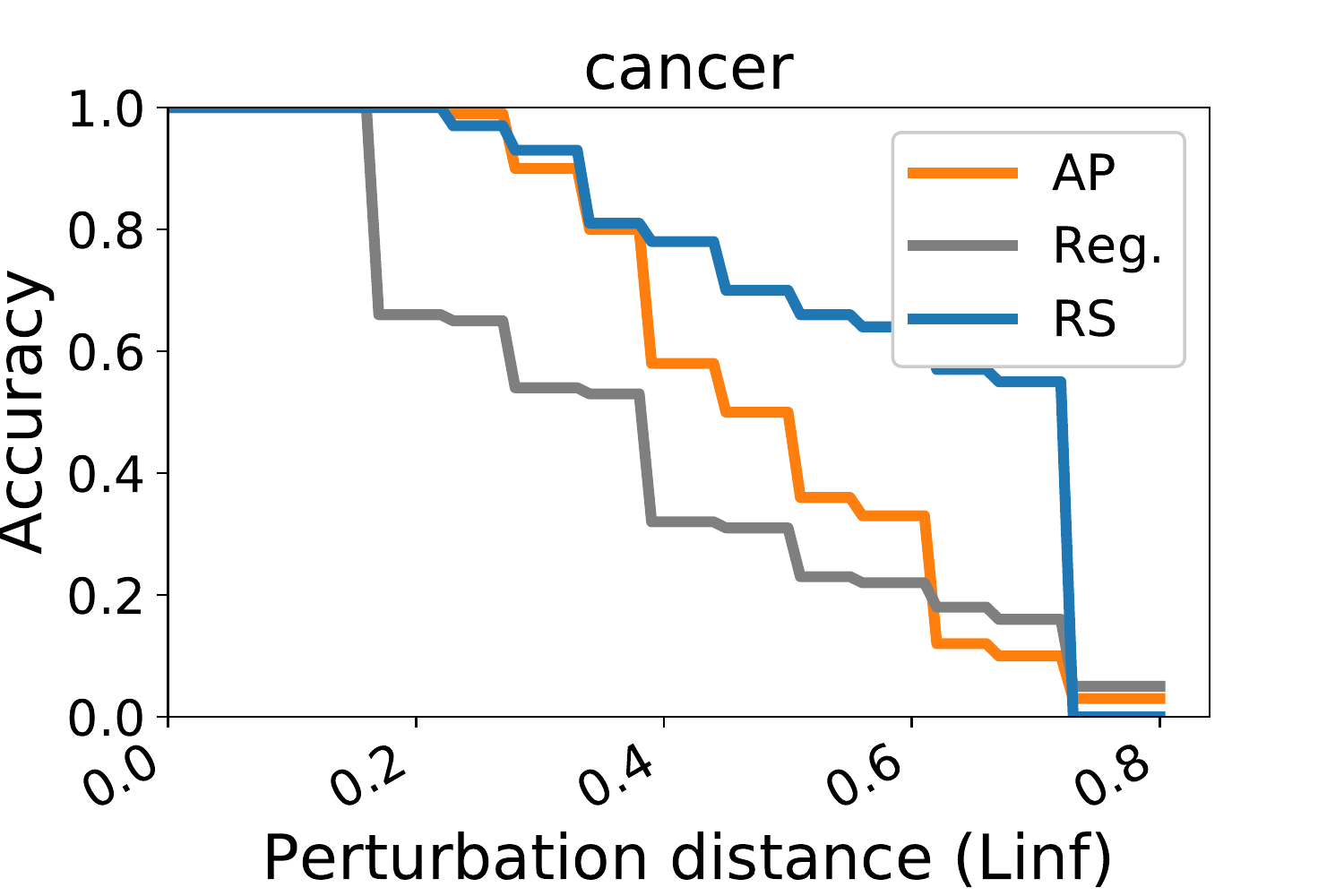}}

\subfloat[1-NN]{
    \includegraphics[width=0.24\textwidth]{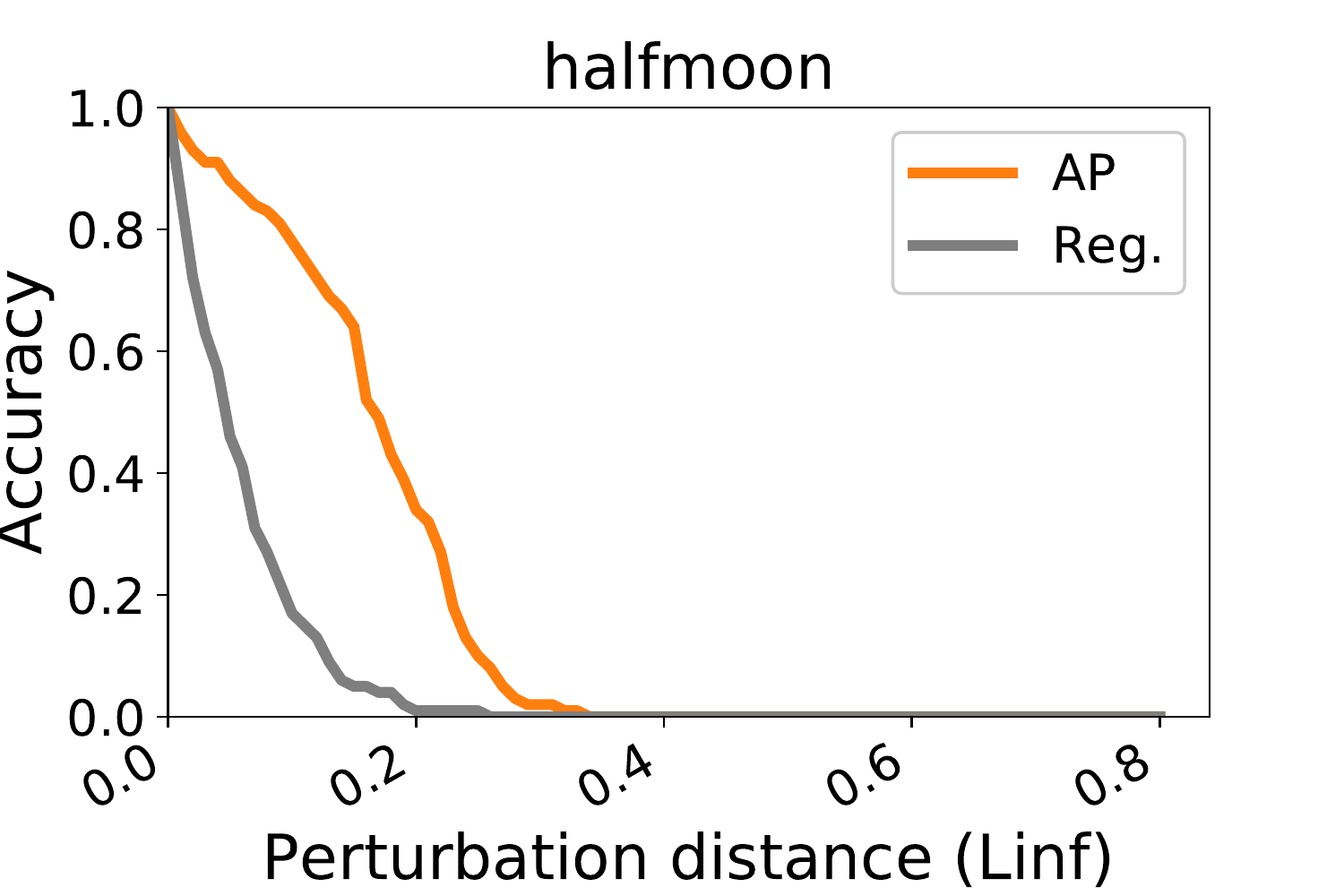}}
\subfloat[3-NN]{
    \includegraphics[width=0.24\textwidth]{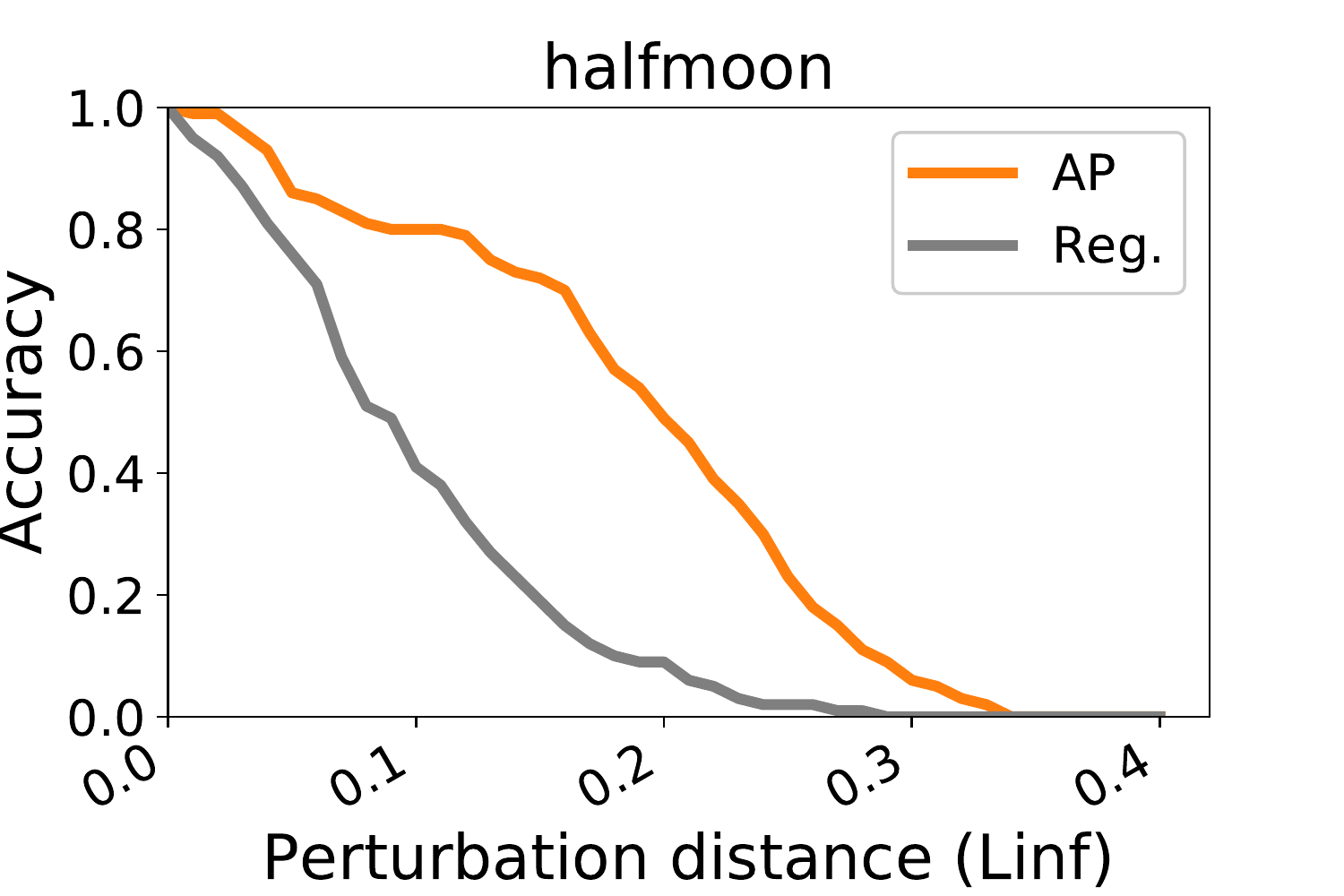}}
\subfloat[Decision tree]{
    \includegraphics[width=0.24\textwidth]{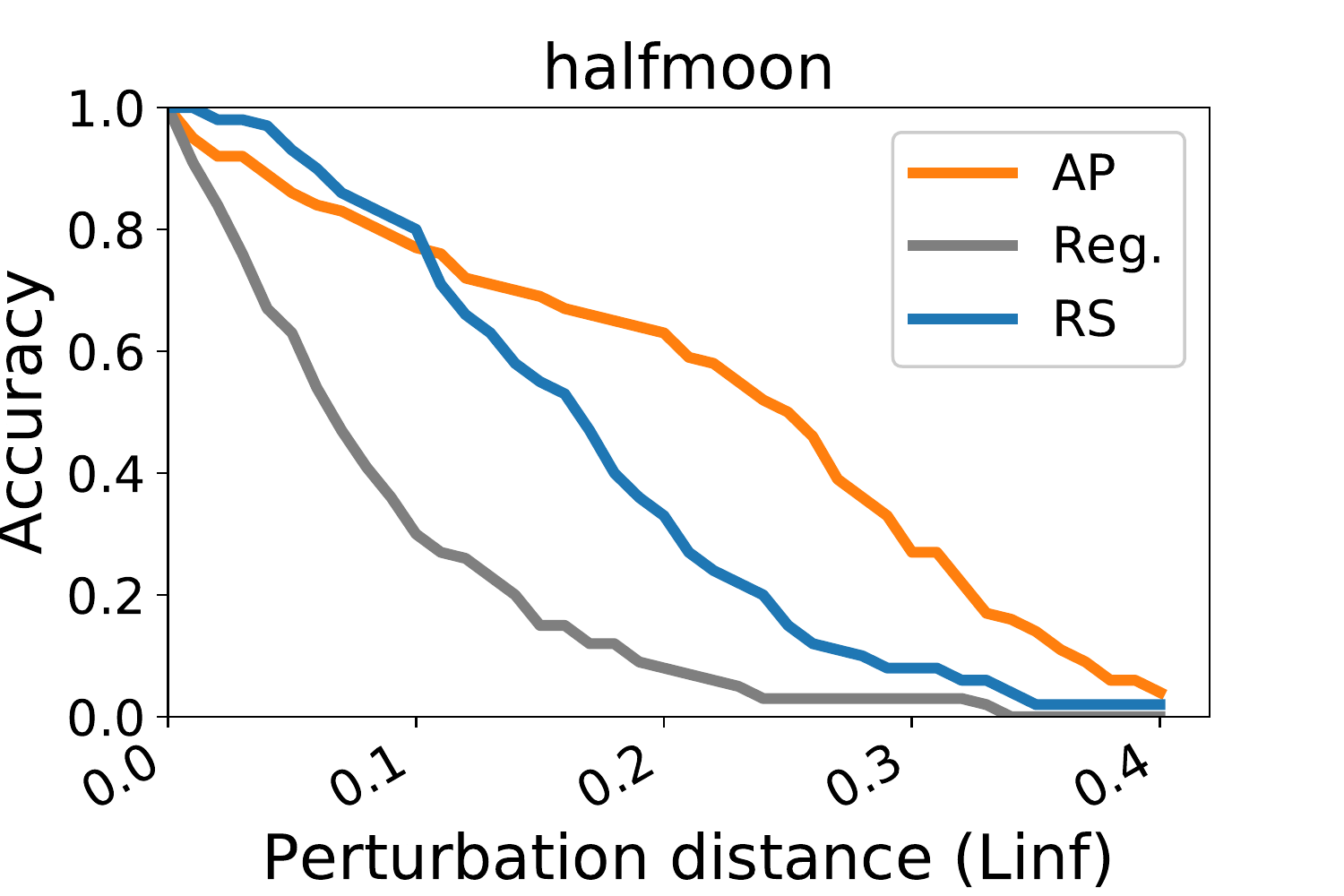}}
\subfloat[Random forest]{
    \includegraphics[width=0.24\textwidth]{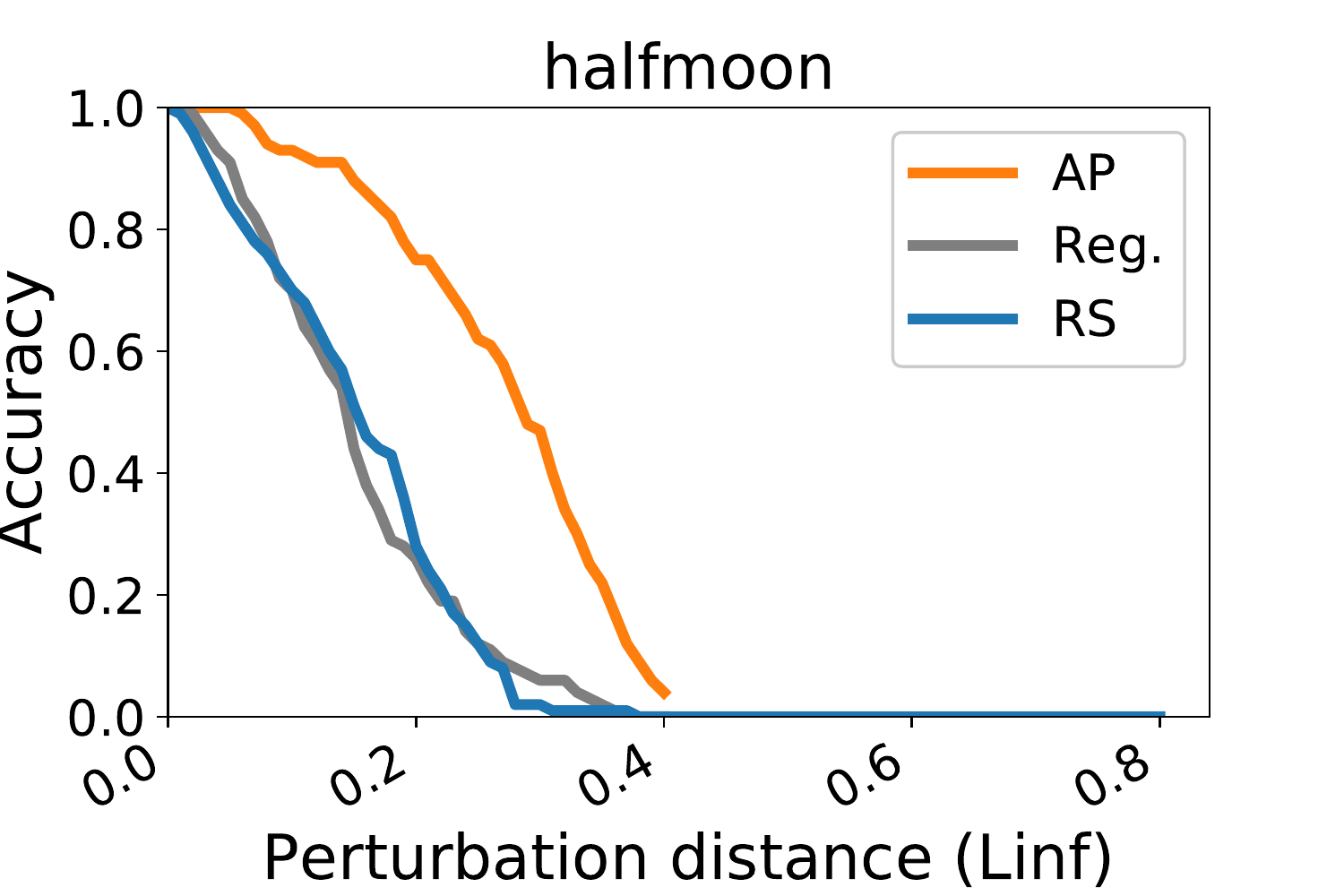}}

\caption{The maximum perturbation distance allowed versus the accuracy on
the 100 correctly predicted test examples (see Section \ref{app:data_detail} for details).}
\label{fig:defense-cmp}
\end{figure*}

\begin{figure*}[ht!]
\centering
\subfloat[1-NN]{
    \includegraphics[width=0.24\textwidth]{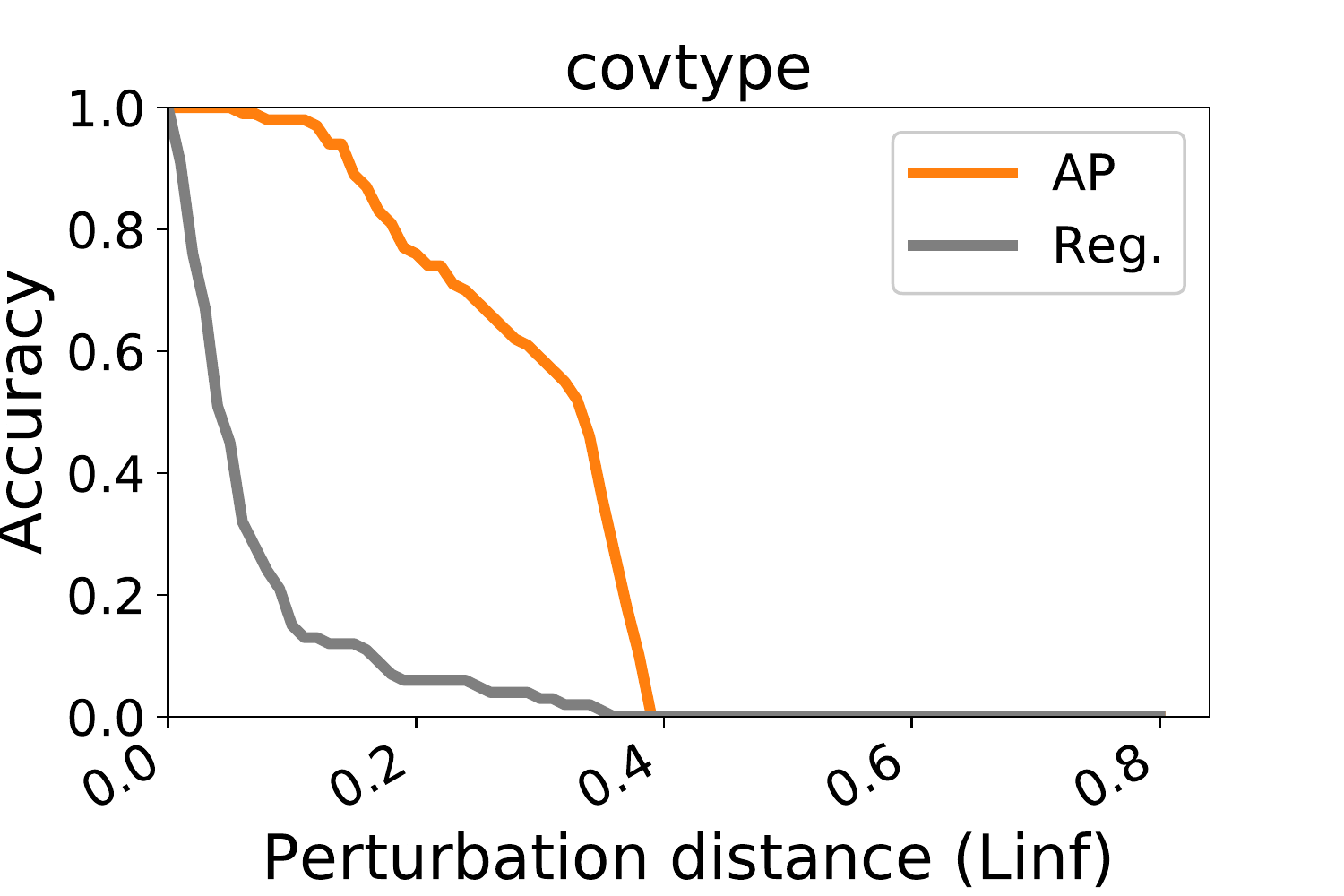}}
\subfloat[3-NN]{
    \includegraphics[width=0.24\textwidth]{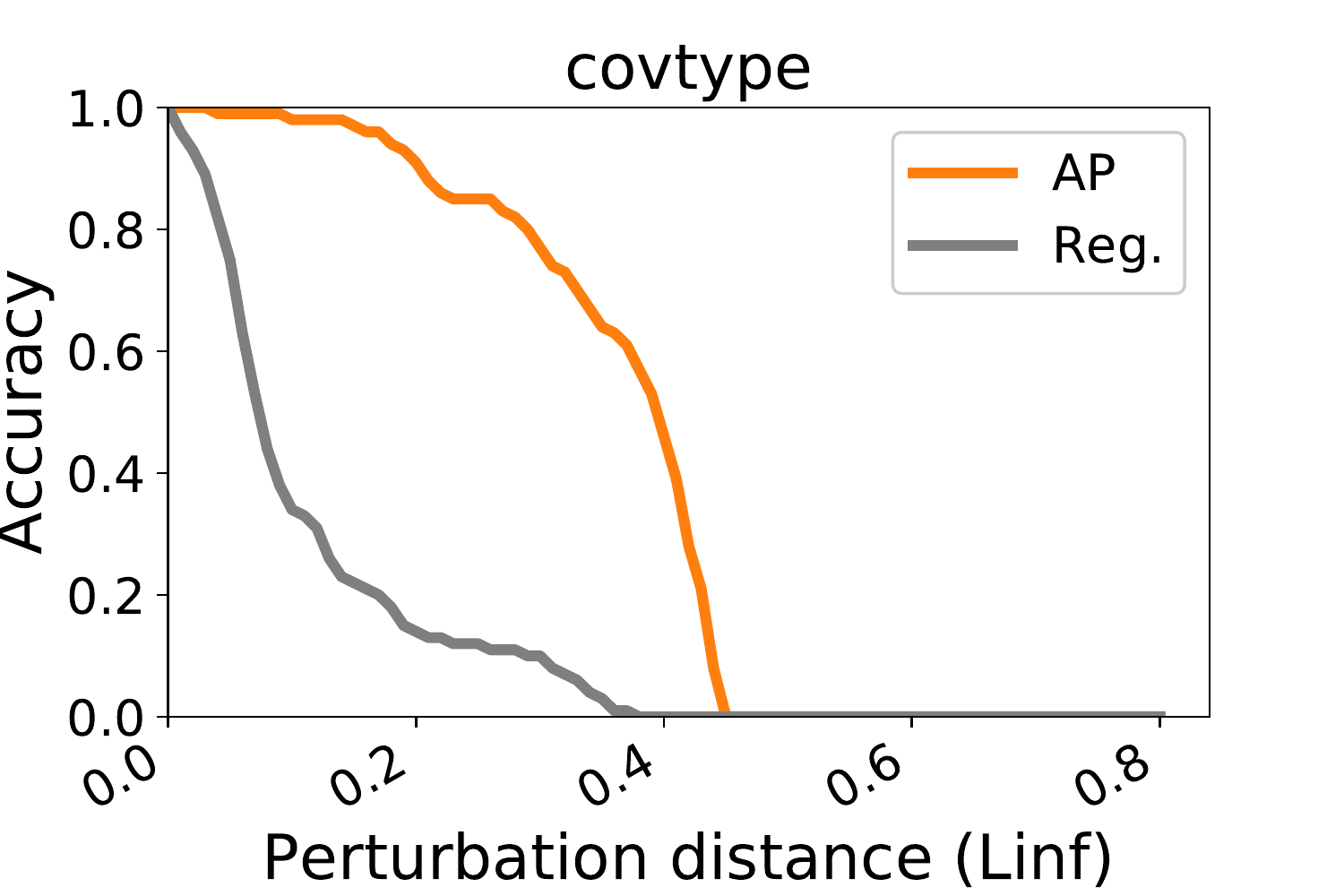}}
\subfloat[Decision tree]{
    \includegraphics[width=0.24\textwidth]{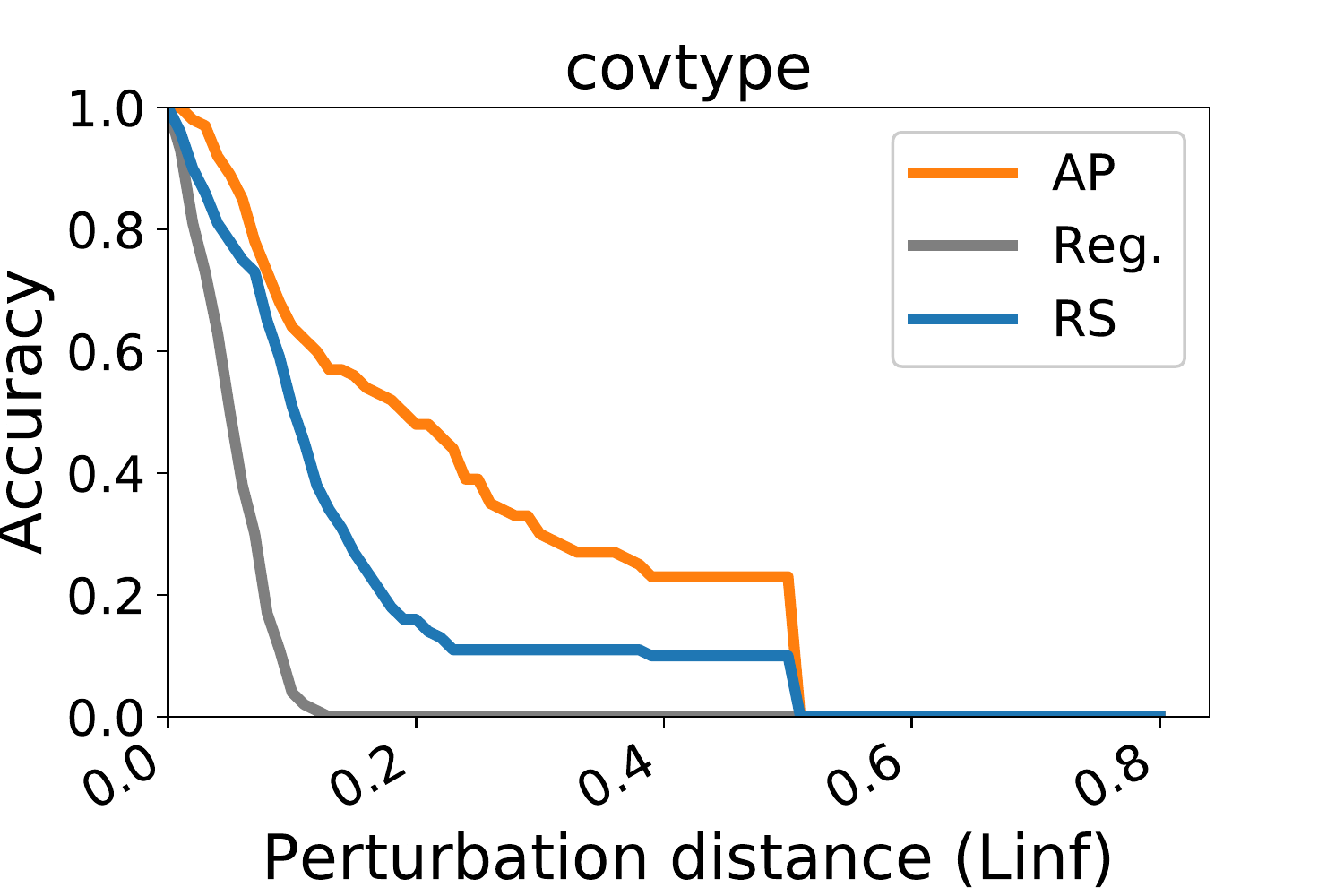}}
\subfloat[Random forest]{
    \includegraphics[width=0.24\textwidth]{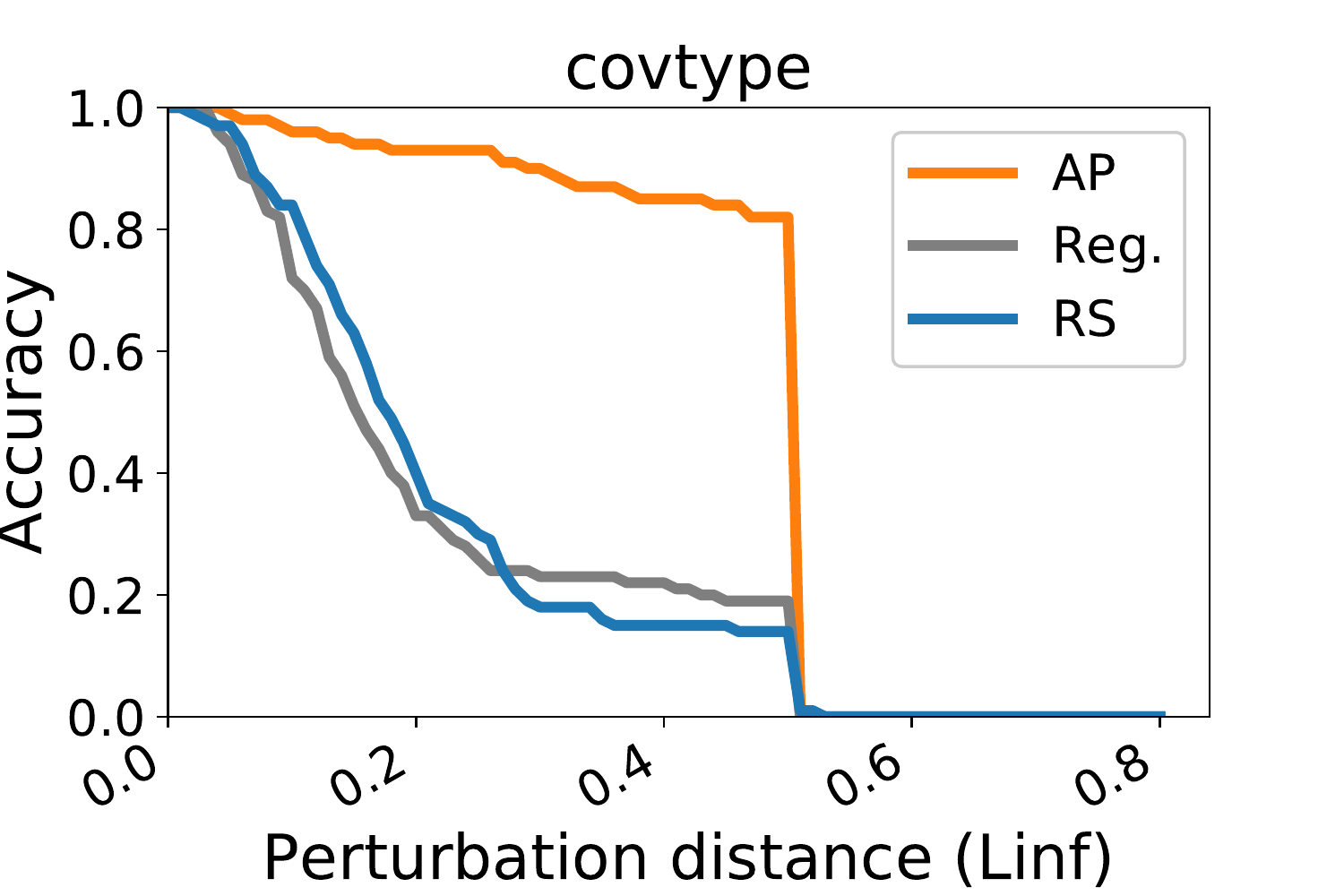}}

\subfloat[1-NN]{
    \includegraphics[width=0.24\textwidth]{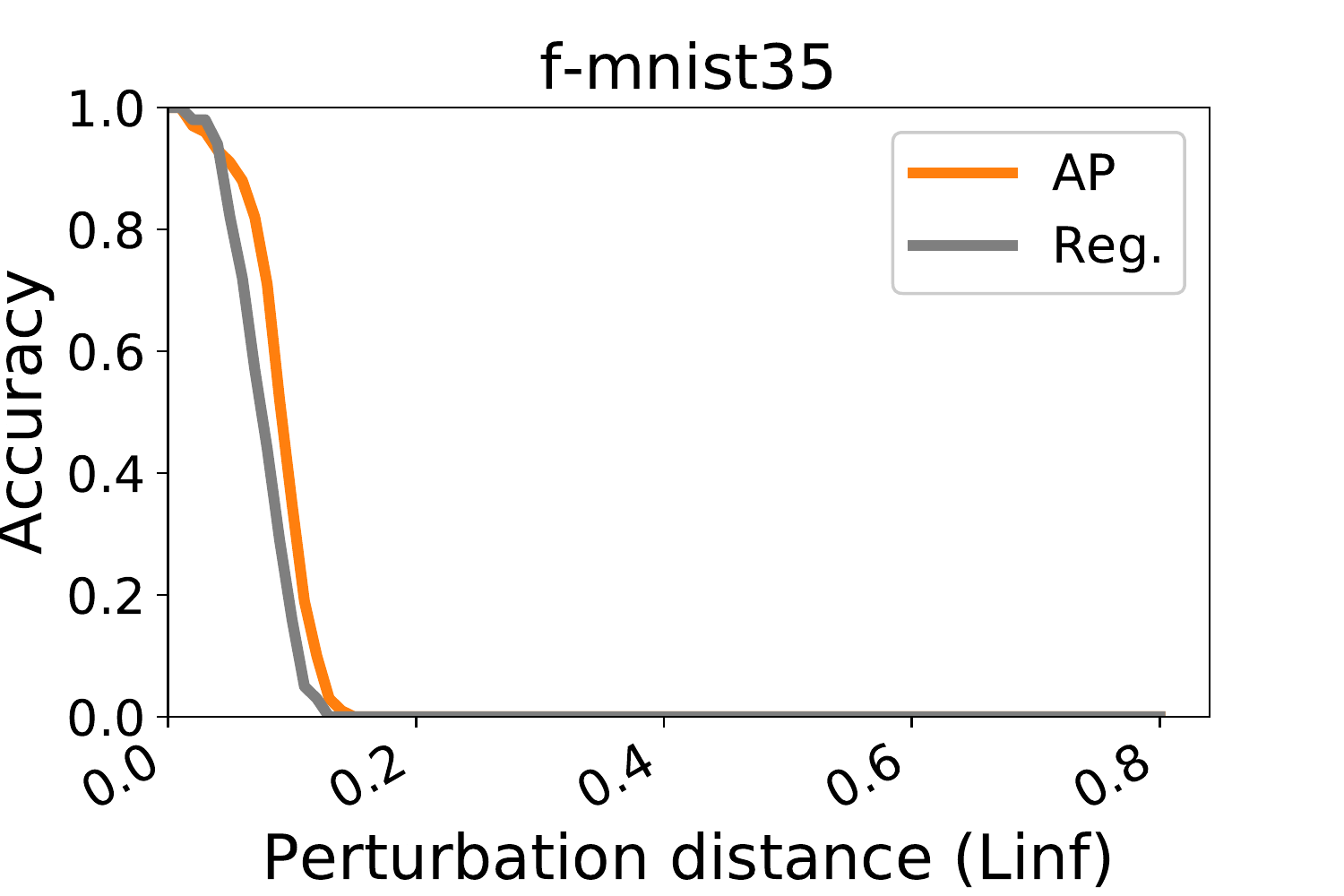}}
\subfloat[3-NN]{
    \includegraphics[width=0.24\textwidth]{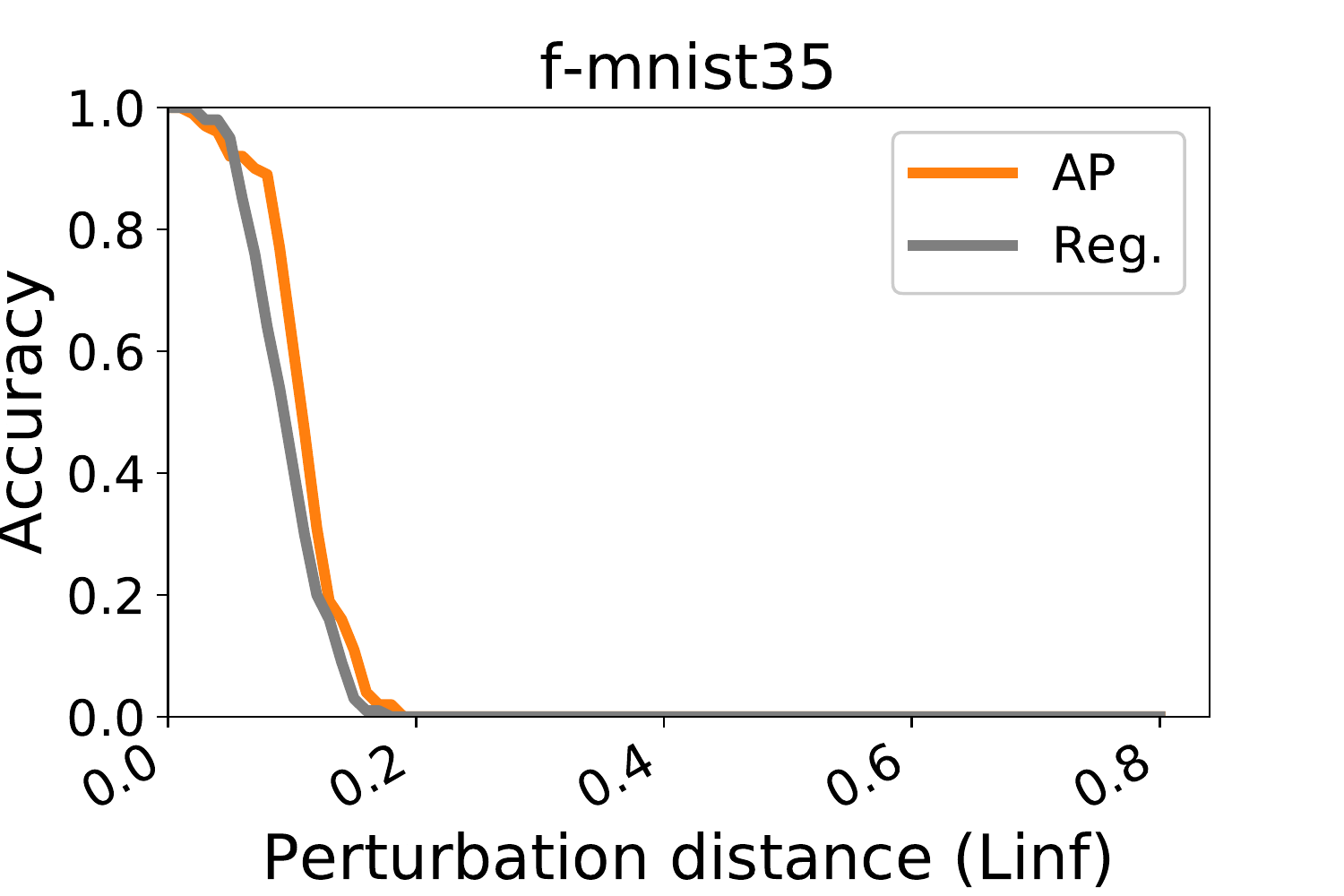}}
\subfloat[Decision tree]{
    \includegraphics[width=0.24\textwidth]{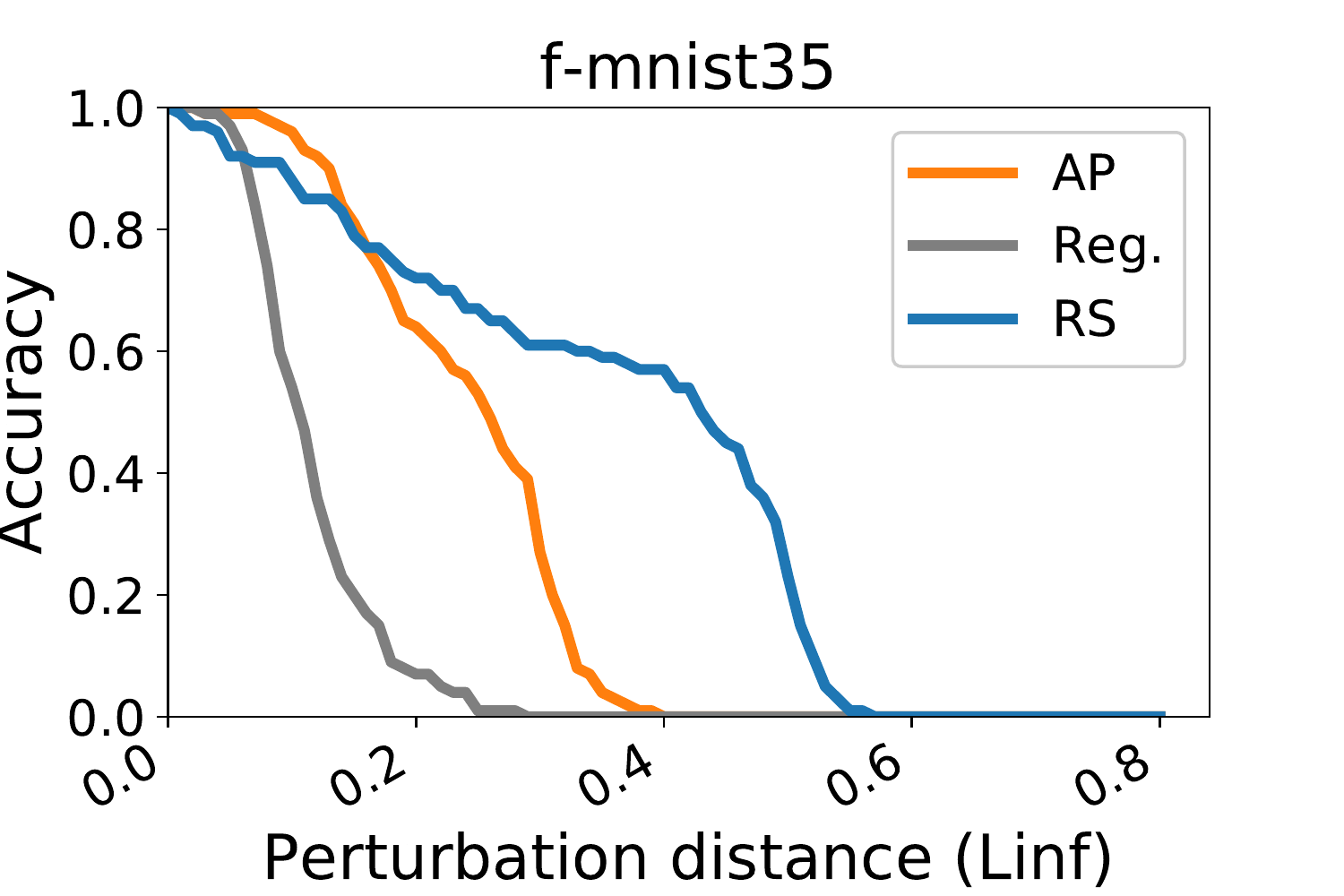}}
\subfloat[Random forest]{
    \includegraphics[width=0.24\textwidth]{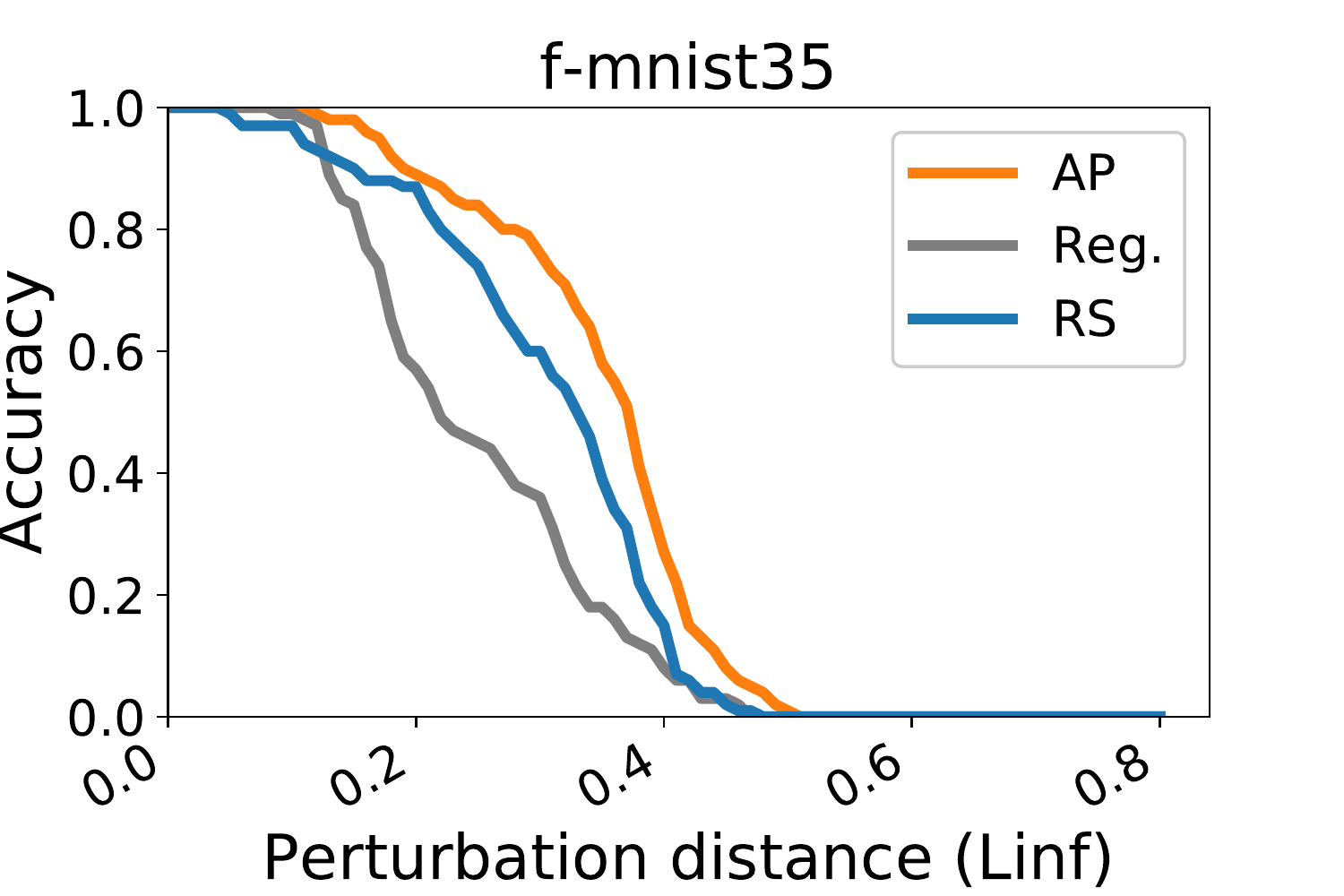}}

\subfloat[1-NN]{
    \includegraphics[width=0.24\textwidth]{./figs/nn_k1_robustness_figs_fashion_mnist06f_pca25_inf-eps-converted-to}}
\subfloat[3-NN]{
    \includegraphics[width=0.24\textwidth]{./figs/nn_k3_robustness_figs_fashion_mnist06f_pca25_inf-eps-converted-to}}
\subfloat[Decision tree]{
    \includegraphics[width=0.24\textwidth]{./figs/dt_robustness_figs_fashion_mnist06f_pca25_inf-eps-converted-to}}
\subfloat[Random forest]{
    \includegraphics[width=0.24\textwidth]{./figs/rf_robustness_figs_fashion_mnist06f_pca25_inf-eps-converted-to}}

\subfloat[1-NN]{
    \includegraphics[width=0.24\textwidth]{./figs/nn_k1_robustness_figs_mnist17f_pca25_inf-eps-converted-to}}
\subfloat[3-NN]{
    \includegraphics[width=0.24\textwidth]{./figs/nn_k3_robustness_figs_mnist17f_pca25_inf-eps-converted-to}}
\subfloat[Decision tree]{
    \includegraphics[width=0.24\textwidth]{./figs/dt_robustness_figs_mnist17f_pca25_inf-eps-converted-to}}
\subfloat[Random forest]{
    \includegraphics[width=0.24\textwidth]{./figs/rf_robustness_figs_mnist17f_pca25_inf-eps-converted-to}}

\caption{The maximum perturbation distance allowed versus the accuracy on
the 100 correctly predicted test examples (see Section \ref{app:data_detail} for details).}
\label{fig:defense-cmp2}
\end{figure*}

\subsection{Images Removed by AP}

Figure \ref{fig:removed} shows examples of images removed by AP
and their closest opposite labeled images. In the case of mnist17, it is interesting to note that the sevens and ones do resemble each other, and so it makes sense that they are close to the training boundary and should be pruned. For the other datasets, recall that we first applied PCA, and therefore, these images are similar in the resulting feature space.

\begin{figure*}[h]
\centering
\subfloat[mnist17]{
    \includegraphics[width=0.25\textwidth]{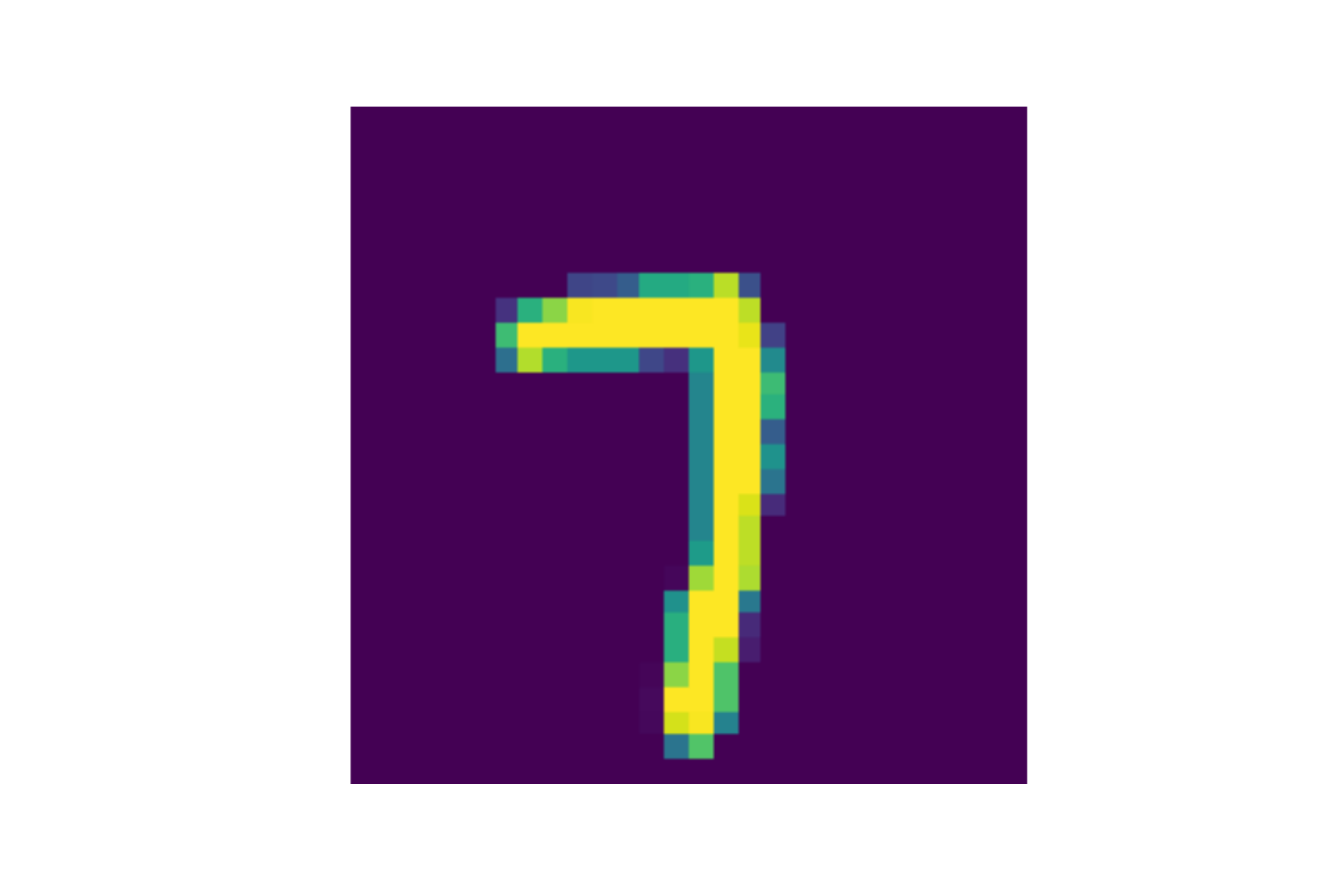}}
\subfloat[mnist17]{
    \includegraphics[width=0.25\textwidth]{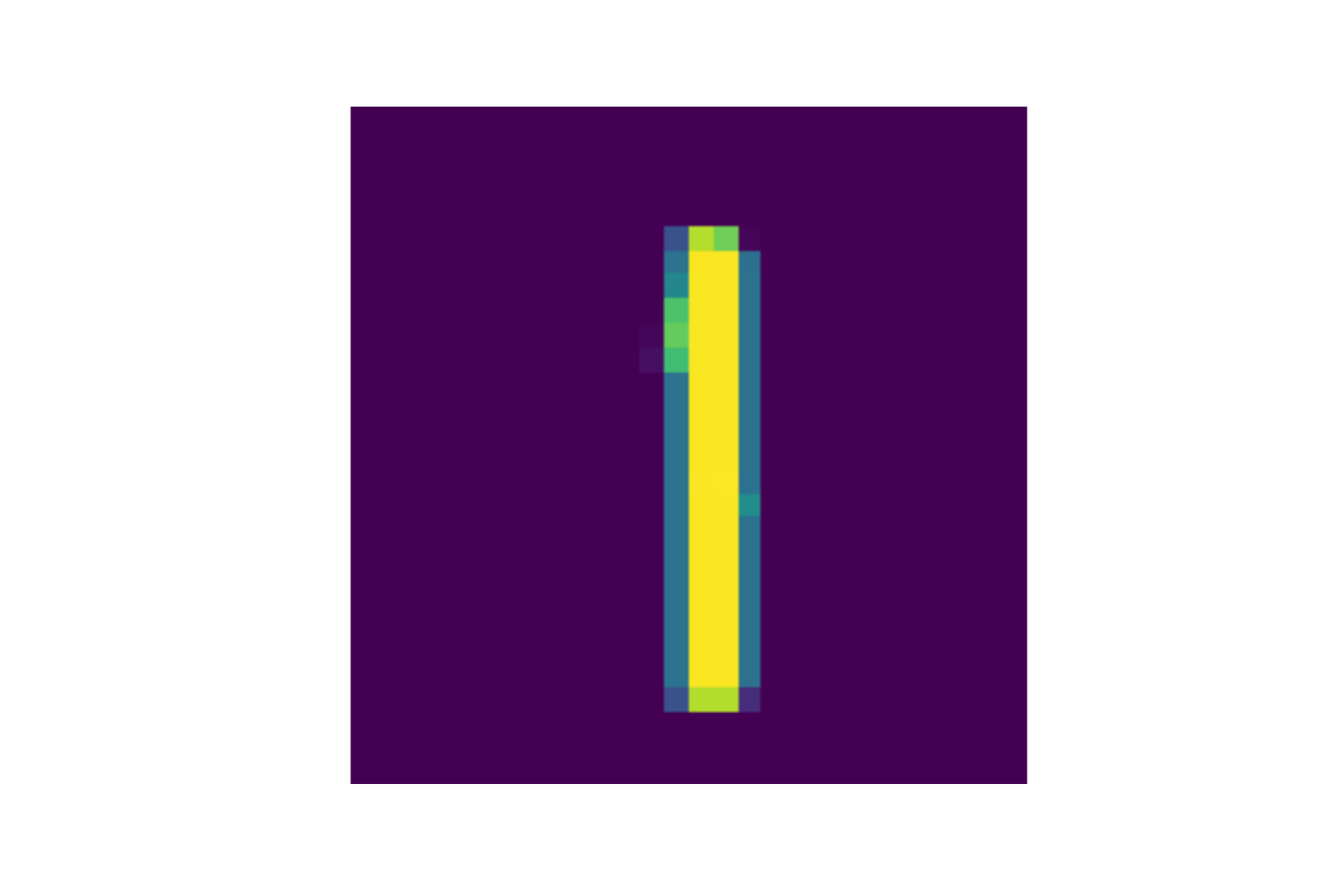}}
\subfloat[mnist17]{
    \includegraphics[width=0.25\textwidth]{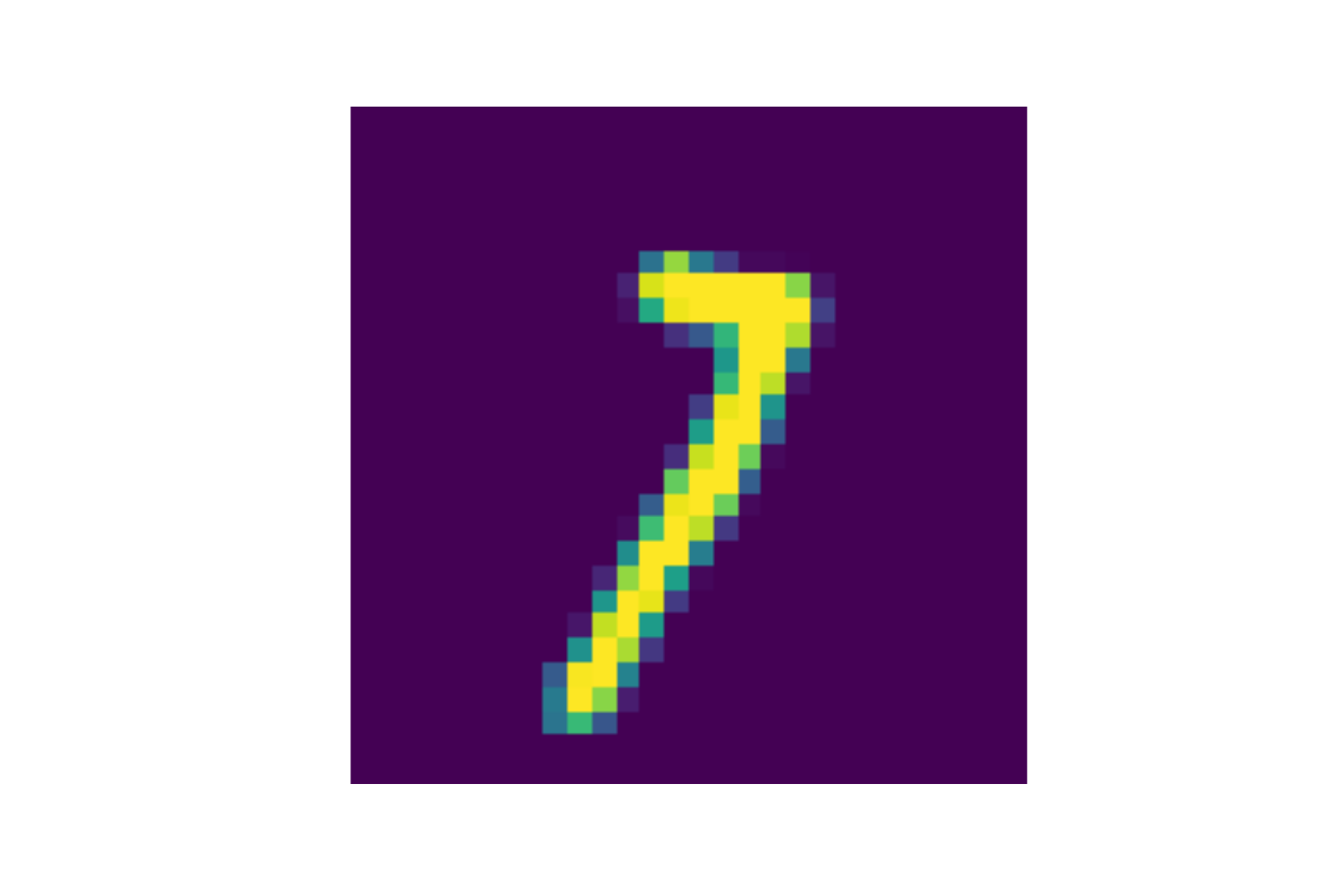}}
\subfloat[mnist17]{
    \includegraphics[width=0.25\textwidth]{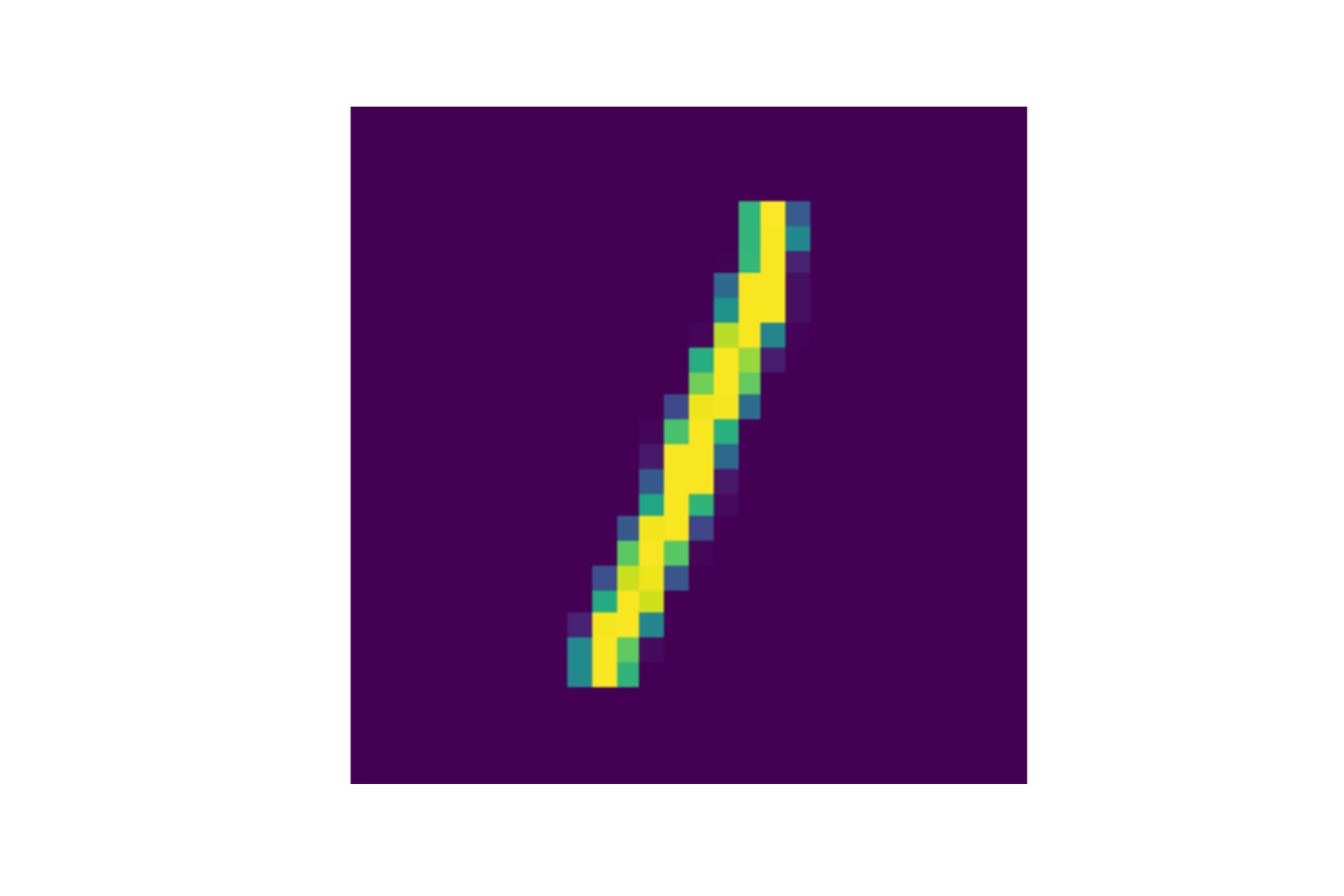}}

\subfloat[f-mnist35]{
    \includegraphics[width=0.25\textwidth]{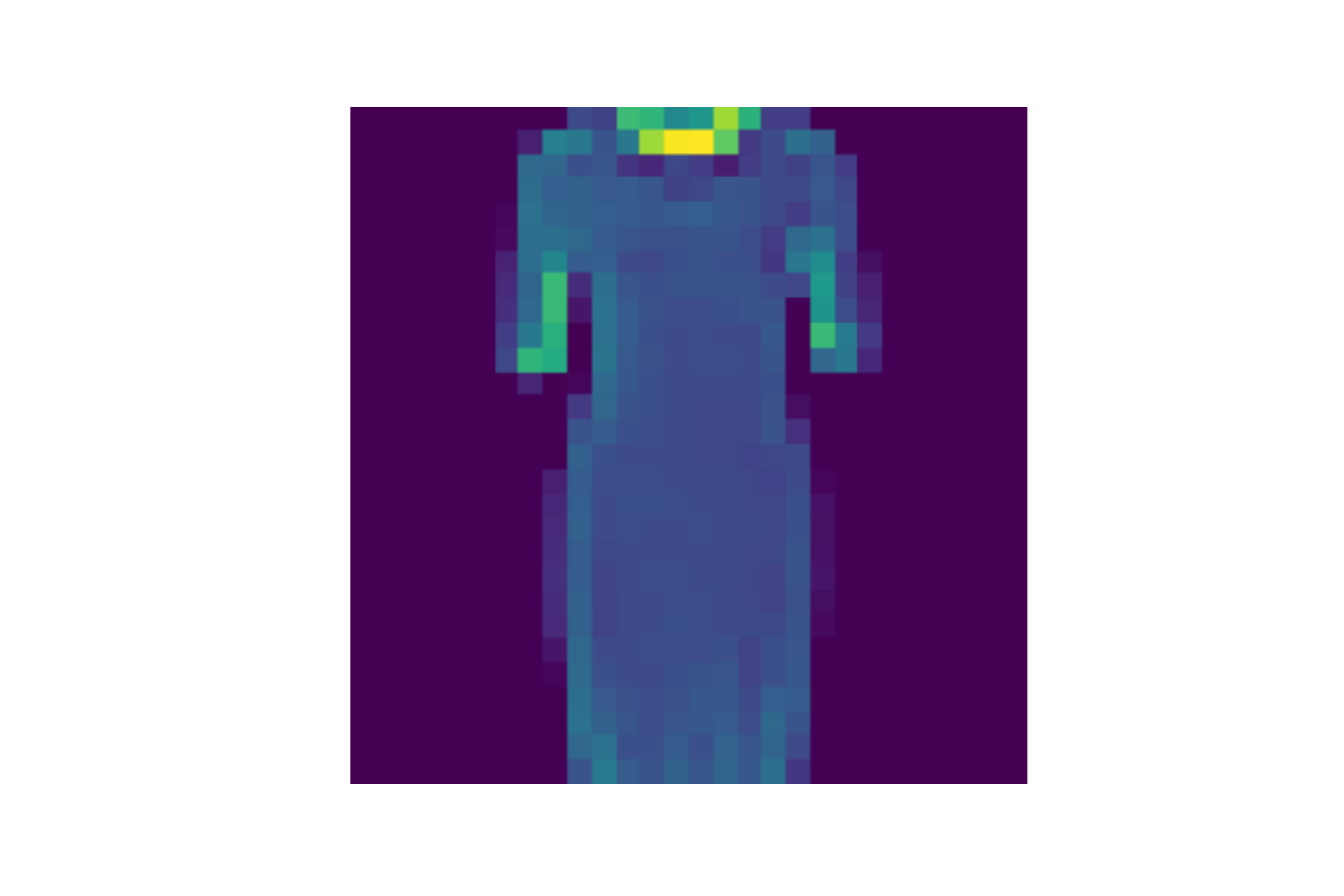}}
\subfloat[f-mnist35]{
    \includegraphics[width=0.25\textwidth]{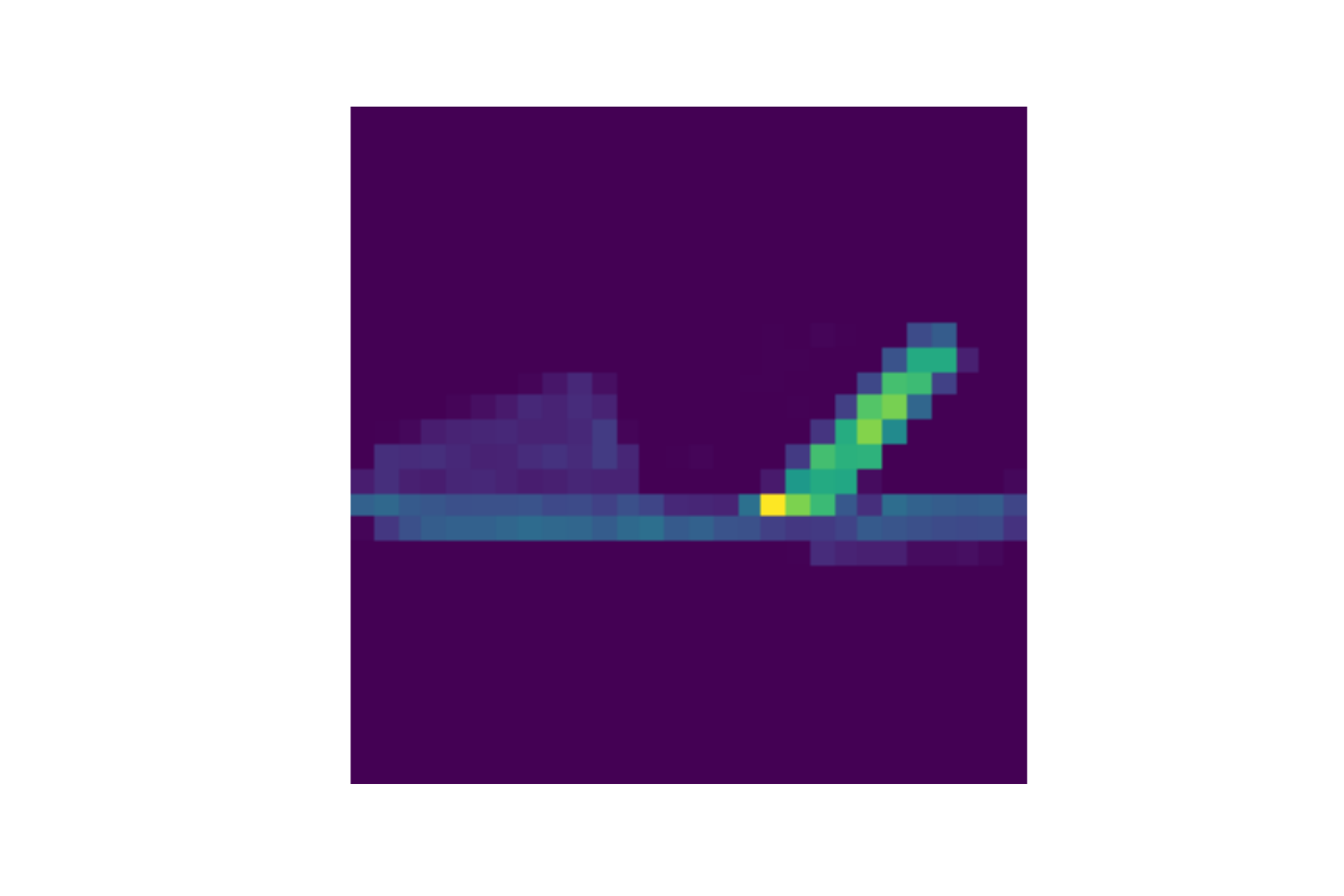}}
\subfloat[f-mnist35]{
    \includegraphics[width=0.25\textwidth]{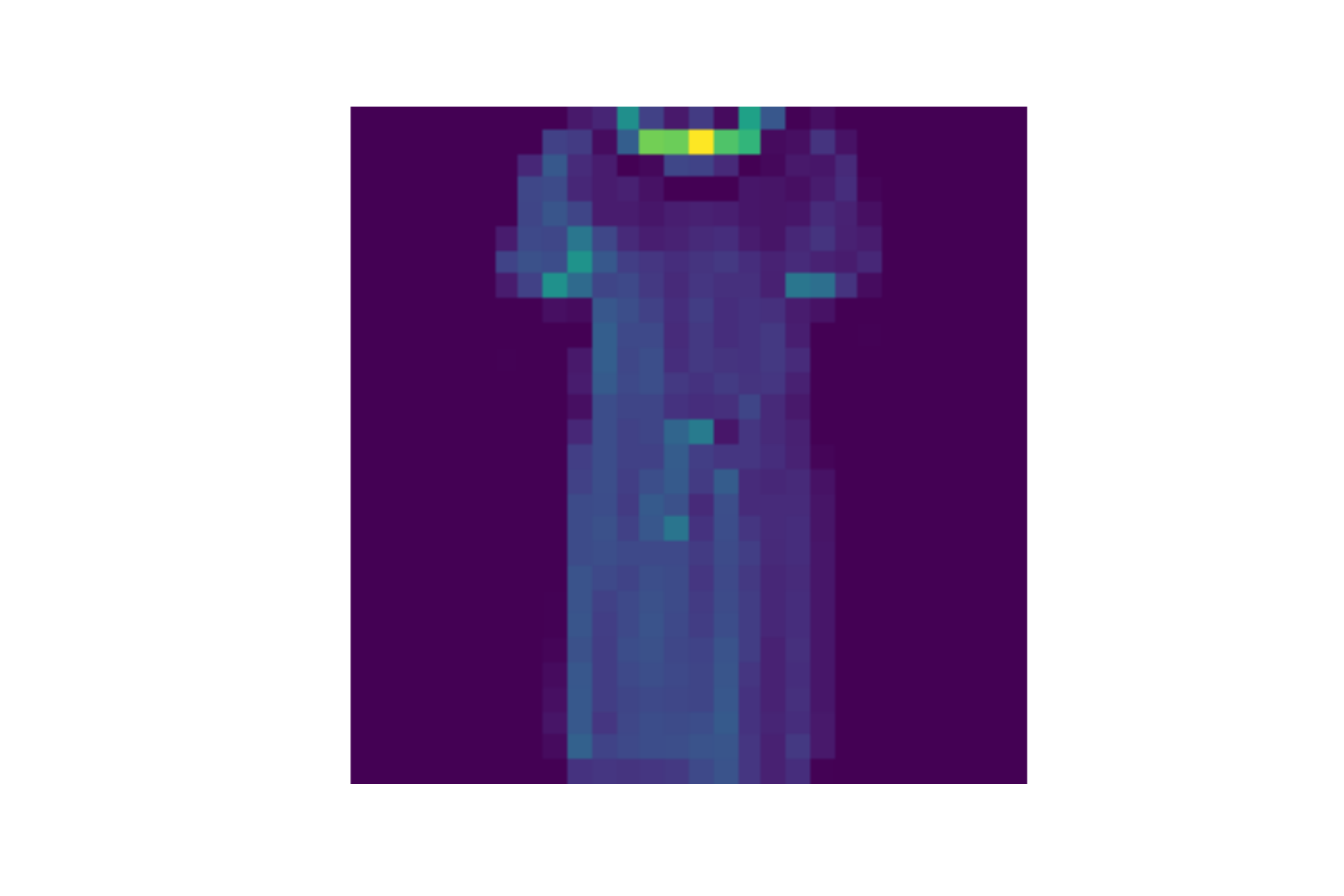}}
\subfloat[f-mnist35]{
    \includegraphics[width=0.25\textwidth]{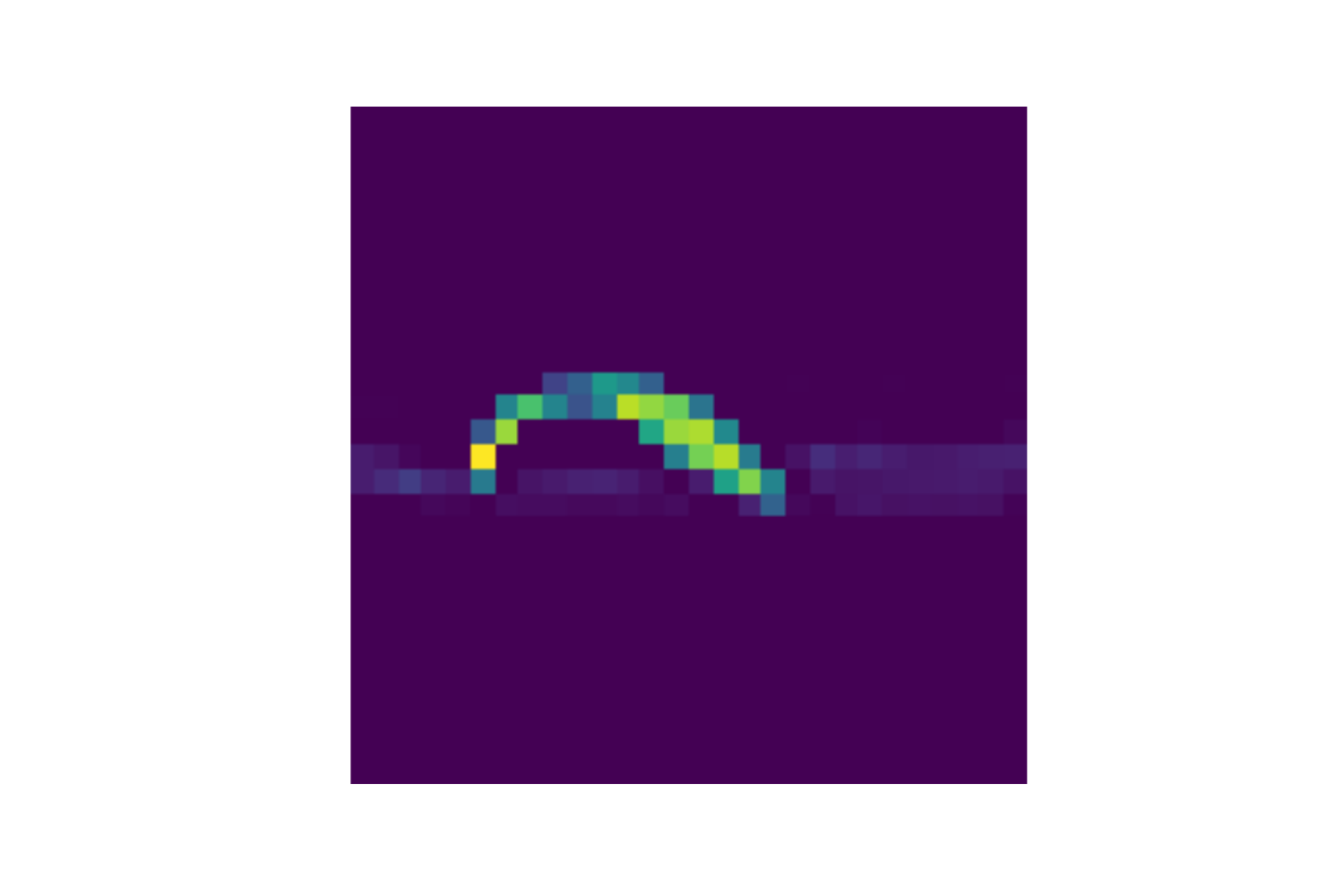}}

\subfloat[f-mnist06]{
    \includegraphics[width=0.25\textwidth]{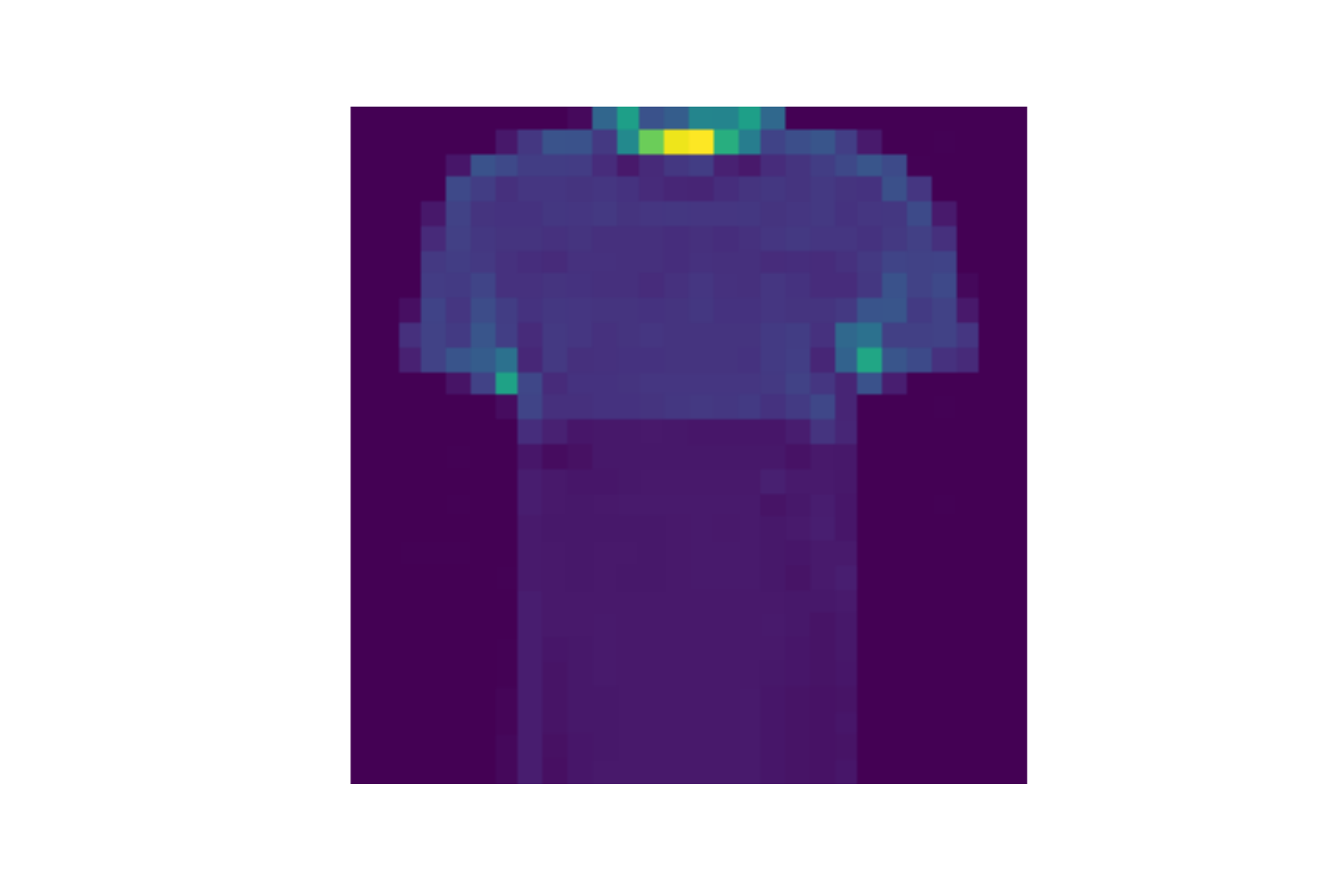}}
\subfloat[f-mnist06]{
    \includegraphics[width=0.25\textwidth]{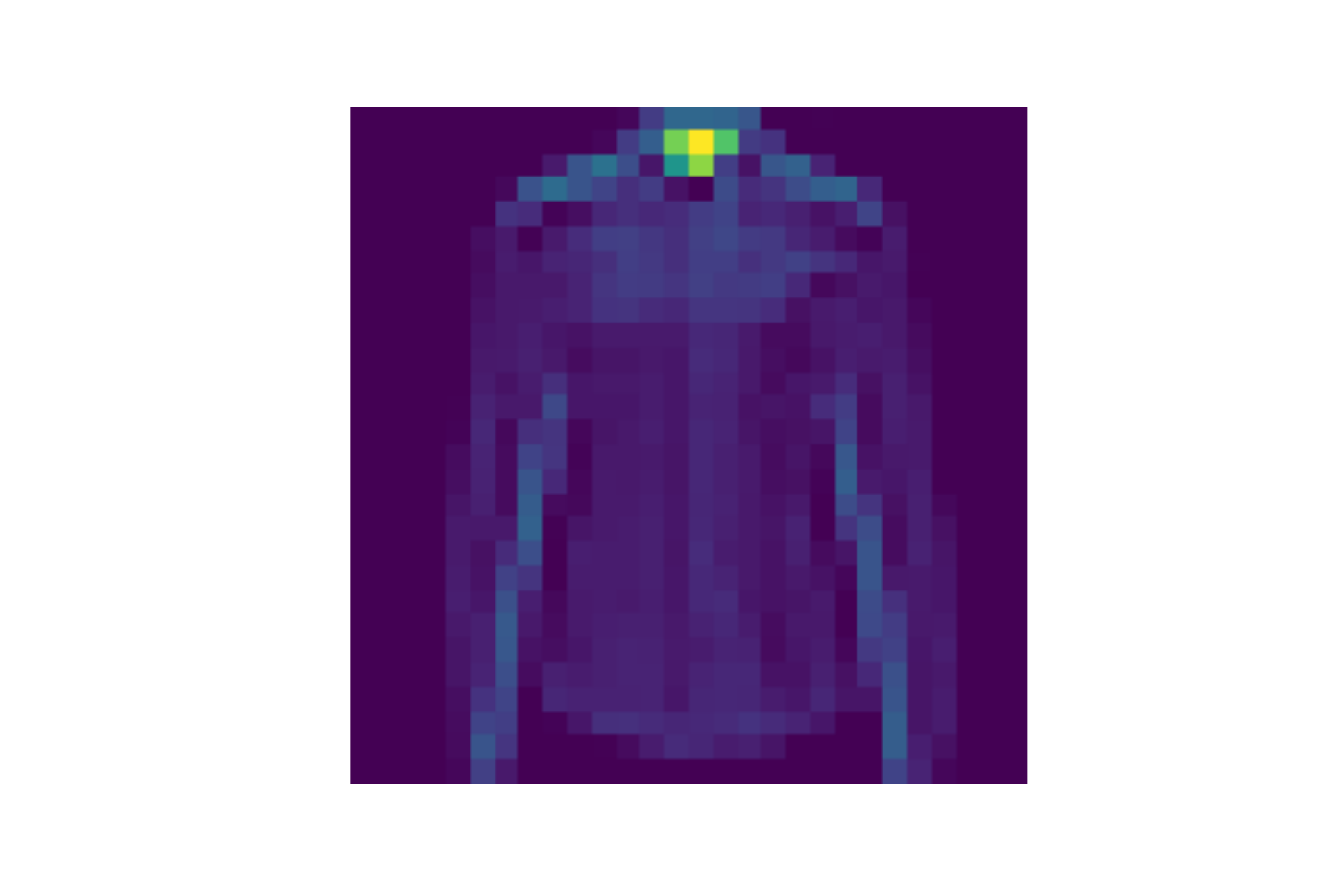}}
\subfloat[f-mnist06]{
    \includegraphics[width=0.25\textwidth]{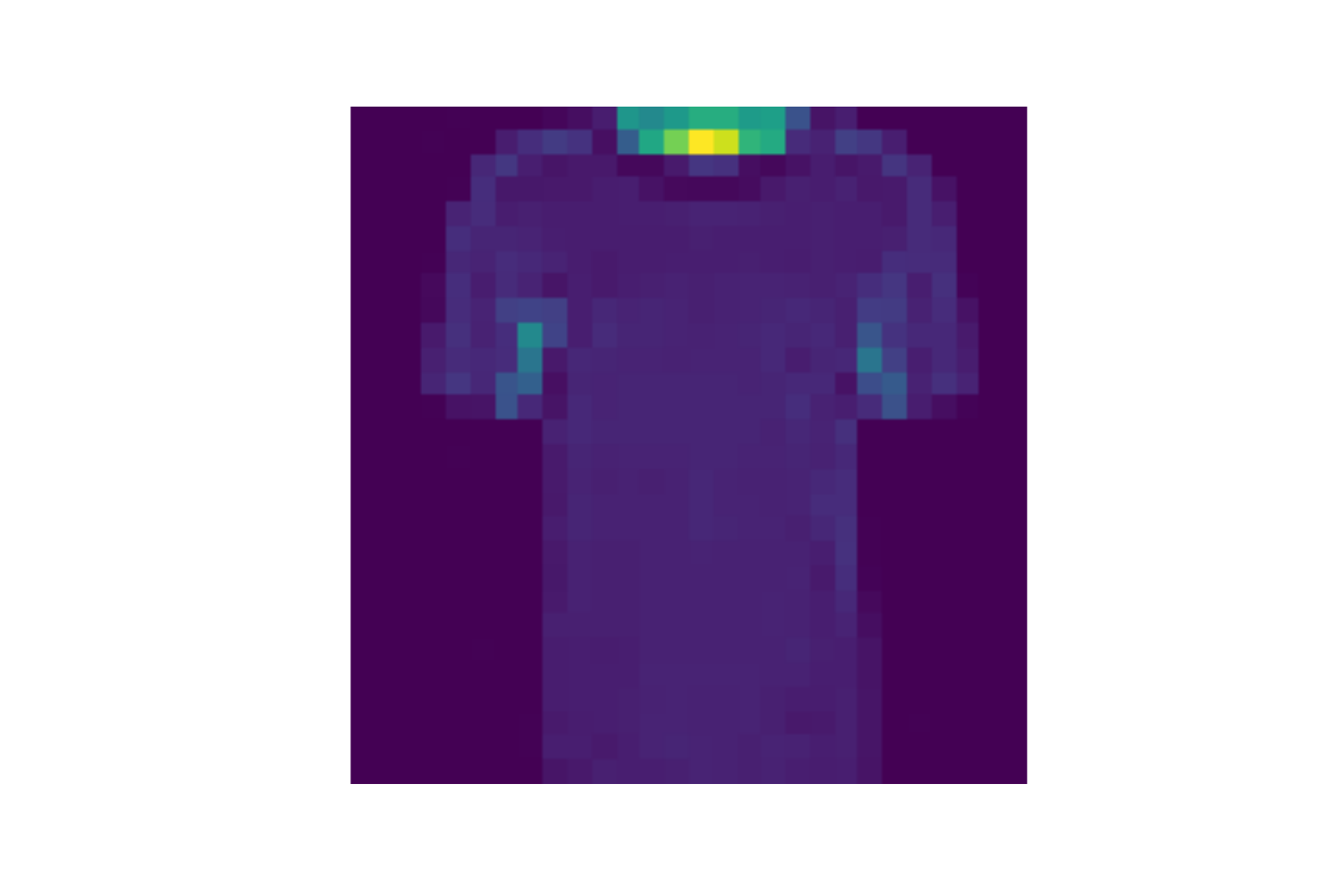}}
\subfloat[f-mnist06]{
    \includegraphics[width=0.25\textwidth]{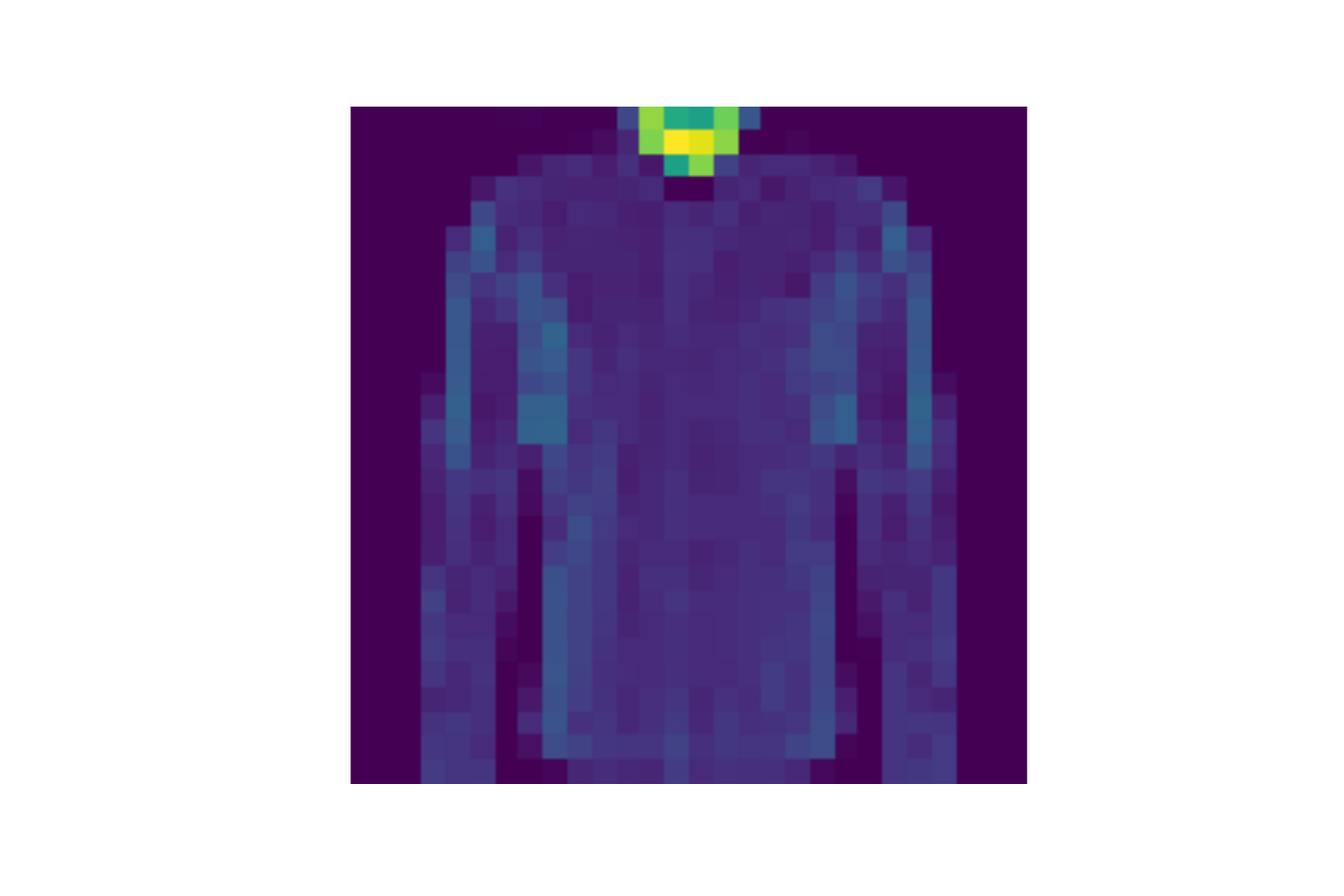}}

\caption{
    Examples of images removed by adversarial pruning (AP).
    The images removed are (a), (c), (e), (g), (i), (k)
    and the images to its right are the closest image with opposite labeled.
    To interpret the labels of these datasets, 
    mnist17 is the task of classifying one versus seven,
    f-mnist35 is Dress versus Sandal, and
    f-mnist06 is T-Shirt/top versus Shirt.
}
\label{fig:removed}
\end{figure*}  
\end{document}